\documentclass[10pt]{article} 
\usepackage[preprint]{tmlr}

\usepackage{amsmath,amsfonts,bm}

\def\eqref#1{equation~\ref{#1}}

\def\1{\bm{1}}

\def\va{{\bm{a}}}
\def\vb{{\bm{b}}}

\def\vg{{\bm{g}}}

\def\vm{{\bm{m}}}

\def\vu{{\bm{u}}}
\def\vv{{\bm{v}}}
\def\vw{{\bm{w}}}
\def\vx{{\bm{x}}}
\def\vy{{\bm{y}}}
\def\vz{{\bm{z}}}

\def\mG{{\bm{G}}}
\def\mH{{\bm{H}}}
\def\mI{{\bm{I}}}

\def\mU{{\bm{U}}}

\DeclareMathAlphabet{\mathsfit}{\encodingdefault}{\sfdefault}{m}{sl}
\SetMathAlphabet{\mathsfit}{bold}{\encodingdefault}{\sfdefault}{bx}{n}

\DeclareMathOperator*{\argmin}{arg\,min}

\usepackage{graphicx}
\usepackage{wrapfig}
\usepackage{amsthm}
\usepackage{algorithm}
\usepackage{algpseudocode} 
\usepackage{amssymb}
\usepackage{booktabs}
\usepackage{tabularx}
\usepackage{rotating}

\newcommand{\change}[1]{#1} 

\newcommand{\norm}[1]{\left\lVert#1\right\rVert}
\newcommand{\dist}[0]{\mathrm{dist}} 
\newcommand{\dif}[0]{\mathrm{d}} 
\newcommand\restr[2]{{
  \left.\kern-\nulldelimiterspace 
  #1 
  \vphantom{\big|} 
  \right|_{#2} 
  }}

\theoremstyle{plain}
\newtheorem{theorem}{Theorem}[section]
\newtheorem{proposition}[theorem]{Proposition}
\newtheorem{lemma}[theorem]{Lemma}
\newtheorem{corollary}[theorem]{Corollary}
\newtheorem{definition}[theorem]{Definition}
\newtheorem{assumption}[theorem]{Assumption}

\usepackage[hidelinks]{hyperref}
\usepackage{url}

\title{A Likely Geometry of Generative Models}

\author{\name Frederik Möbius Rygaard \email fmry@dtu.dk \\
      \addr DTU Compute\\
      Technical University of Denmark
      \AND
      \name Shen Zhu \email sz9jt@virginia.edu \\
      \addr Department of Computer Science\\
      University of Virginia
      \AND
      \name Yinzhu Jin \email yj3cz@virginia.edu \\
      \addr Department of Computer Science\\
      University of Virginia
      \AND
      \name Søren Hauberg \email sohau@dtu.dk \\
      \addr DTU Compute\\
      Technical University of Denmark
      \AND
      \name Tom Fletcher \email ptf8v@virginia.edu \\
      \addr Department of Electrical and Computer Engineering\\
      University of Virginia
      }

\def\month{10}  
\def\year{2026} 
\def\openreview{\url{https://openreview.net/forum?id=qKM72YmSSw}} 

\begin{document}

\maketitle

\begin{abstract}
The geometry of generative models serves as the basis for interpolation, model inspection, and more. Although certain generative models admit an implicit geometric structure, there is no broadly applicable framework that captures a principled notion of geometry across generative models without imposing restrictive assumptions on the model class or data dimensionality. In this paper, we show how to equip generative models with a general geometry compatible with different metrics and probability distributions to analyze generative models. Our method does not require additional training. We consider curves analogous to geodesics constrained to a suitable data distribution aimed at targeting high-density regions learned by generative models. We formulate this as a (pseudo-)metric and prove correspondence to a Newtonian system on a Riemannian manifold. We show that shortest paths can be characterized by a system of ordinary differential equations, which, along the optimal path, locally correspond to geodesics under a suitable Riemannian metric. Numerically, we derive a novel algorithm to efficiently compute interpolation and generalized Fr\'echet means. Quantitatively, we show that curves using our metric traverse regions of higher likelihood than baselines across a range of models and datasets. The code is available at \url{https://github.com/FrederikMR/likely_geometry}.
\end{abstract} 

\section{Introduction} \label{sec:introduction}
Generative models learn the data distribution but rarely specify how to compute statistics such as interpolation and means on the generated samples. Such actions require a \emph{model geometry}. 
However, geometries are commonly defined \emph{ad hoc}, e.g., through spherical and linear interpolation \citep{song2020improvedtechniquestrainingscorebased, song2022denoisingdiffusionimplicitmodels, zheng2024noisediffusioncorrectingnoiseimage}.
Although these interpolation schemes are simple, they assume that the prior distribution is an isotropic Gaussian, where spherical interpolation also assumes that the space is high-dimensional. This does not generalize well across generative models.

Other works equip generative models with a Riemannian metric to model the data manifold \citep{tosi2014metricsprobabilisticgeometries, arvanitidis2018latentspaceodditycurvature, shao2017riemanniangeometrydeepgenerative, karczewski2025spacetimediffusionmodelsinformation, bethune2025energy}, but this geometry is specified on a \emph{per-model} basis. In general, Riemannian manifolds provide an operational framework for the analysis of generative models \citep{hauberg2019bayeslearnmanifoldon} and naturally appear in latent space models \citep{shao2017riemanniangeometrydeepgenerative, arvanitidis2018latentspaceodditycurvature, hauberg2019bayeslearnmanifoldon}. For Riemannian manifolds, interpolation corresponds to \emph{geodesics}, which are locally length-minimizing curves. Geodesics can be locally characterized by an ordinary differential equation (\textsc{ode}) \citep[Lemma 2.3]{do1992riemannian}
\begin{equation} \label{eq:ode_riemann}
    \ddot{\bm{\bm{\gamma}}}^{k}(t) + \Gamma_{ij}^{k}\left(\bm{\gamma}(t)\right)\dot{\bm{\gamma}}^{i}(t)\dot{\bm{\gamma}}^{j}(t) = 0, \quad \change{\Gamma_{ij}^{k}=\sum_{kl}\frac{1}{2}g^{kl}\left(\frac{\partial g_{jl}}{\partial x^{i}} + \frac{\partial g_{ij}}{\partial x^{j}} - \frac{\partial g_{il}}{\partial x^{l}}\right)}.
\end{equation}
\change{where $\left\{\Gamma_{ij}^{k}\right\}_{ijk}$ denotes the Christoffel symbols derived from the Riemannian metric $g$ with inverse $\{g^{kl}\}_{kl}$}. However, geodesics for generative models are not always within high-likelihood regions of the data distribution, and for geodesics to have high likelihood, it requires imposing restrictions on the underlying generative model to incorporate the uncertainty of the data into the metric \citep{hauberg2019bayeslearnmanifoldon, tosi2014metricsprobabilisticgeometries}. 

\begin{wrapfigure}[23]{r}{0.4\textwidth}
    \vspace{0.0em}
    \centering
    \includegraphics[width=0.4\textwidth]{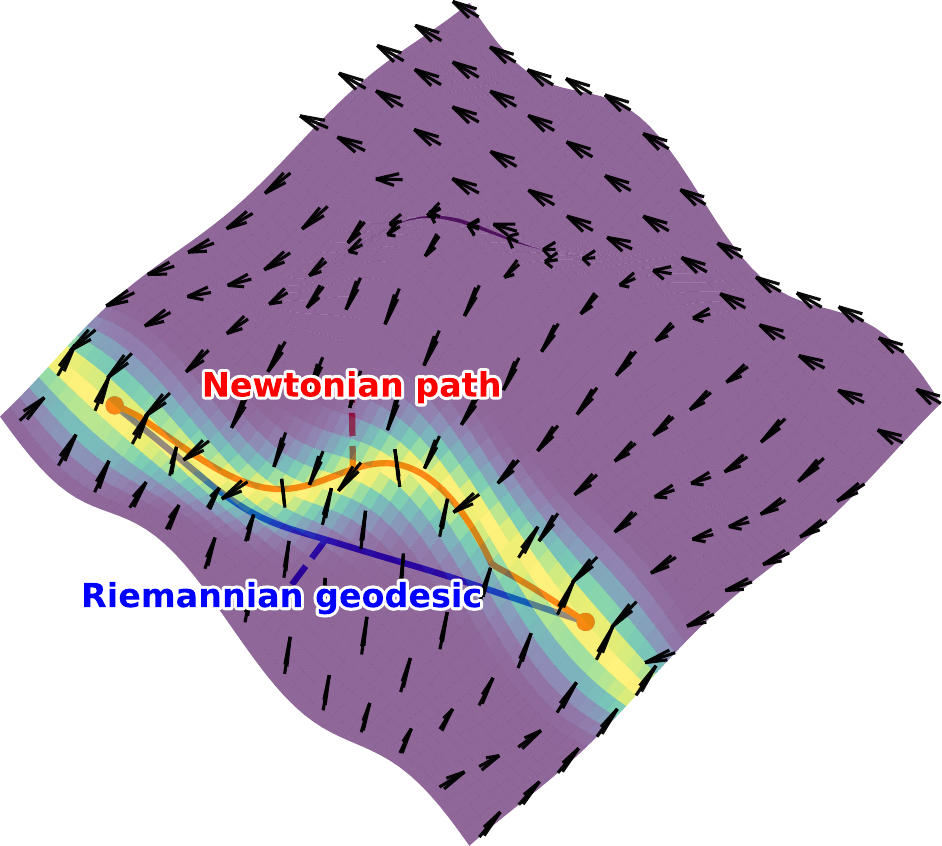}
    \vspace{-5mm}
    \caption{A conceptual illustration of our method, where we consider the geometry corresponding to a Newtonian system on a Riemannian manifold. The gradient vector field of the density ``pushes'' the Riemannian geodesic to areas of high likelihood.}
    \label{fig:path_illustration}
\end{wrapfigure}
Rather than modeling the geometry of the data manifold as a Riemannian manifold, we show in this paper how to construct a force field that points to regions of high likelihood as illustrated in Fig.~\ref{fig:path_illustration}. We show that enforcing geodesics to be within high-likelihood areas does not correspond to a Riemannian metric but instead corresponds to a Newtonian system on a Riemannian manifold. Then, shortest curves are given by \citep{newton_riemann_mechanics}
\begin{equation} \label{eq:ode_newton}
    \ddot{\bm{\gamma}}^{k}(t) + \Gamma_{ij}^{k}\left(\bm{\gamma}(t)\right)\dot{\bm{\gamma}}^{i}(t)\dot{\bm{\gamma}}^{j}(t) = \nabla_{x}^{k}S(x),
\end{equation}
where $\nabla_{x}^{k}$ is the $k$th element of the Riemannian gradient of a function $S: \mathcal{M} \rightarrow \mathbb{R}$. We start by stating the requirements for geodesics that they should be restricted to high-likelihood areas, and formulate this as a (pseudo-)metric to define the notion of interpolation and a mean value. We further show that along the shortest curve, our approach can locally be seen as a geodesic under a Riemannian metric. We derive a novel algorithm to compute interpolation and means using our geometric framework and prove that our algorithm converges to a local minimum and has local asymptotic quadratic convergence under sufficient regularity, enabling fast computation of statistics for generative models. Empirically, we demonstrate that our approach obtains curves with a higher likelihood compared to baselines on various datasets and generative models.

\section{Background and related work} \label{sec:background}
\textbf{Riemannian manifolds} provide an operational framework suitable for analyzing data and generative models \citep{hauberg2019bayeslearnmanifoldon} and appear naturally in latent space models like the variational-autoencoder (\textsc{vae}) \citep{shao2017riemanniangeometrydeepgenerative, arvanitidis2018latentspaceodditycurvature}. \change{A Riemannian manifold can be viewed as a smooth $n$-dimensional surface in a sufficiently high-dimensional Euclidean space \citep{nash_embedding, whitney_embedding}. In practice, rather than working in the ambient Euclidean space, computations are performed intrinsically in local coordinates of the manifolds using the Riemannian metric. For example, the sphere and torus admit embeddings in $\mathbb{R}^{3}$, yet all computations can naturally be carried out in two-dimensional coordinate charts of the manifold.} Formally, a Riemannian manifold is a differentiable manifold $\mathcal{M}$, equipped with a Riemannian metric, in the sense that it defines a smoothly varying inner product $g: T_{\vx}\mathcal{M} \times T_{\vx}\mathcal{M} \rightarrow \mathbb{R}$. The inner product is a quadratic form $g_{\vx}(\vv,\vw) = \vv^{\top}\mG(\vx)\vw$, where $\vv,\vw \in T_{\vx}\mathcal{M}$ denote elements in the tangent space of $\mathcal{M}$ in $\vx \in \mathcal{M}$ and $\mG$ is a symmetric and positive definite matrix known as the metric matrix function \citep[Page 38]{do1992riemannian}. The tangent space at $\vx \in \mathcal{M}$ is a vector space that consists of the tangent to all curves at $\vx$. The Riemannian metric gives rise to the notion of curves that locally minimize the Riemannian distance:
\begin{equation*}
    \dist(\va,\vb) := \min_{\bm{\gamma}} \int_{0}^{1} \sqrt{\dot{\bm{\gamma}}(t)^{\top}\mG(\bm{\gamma}(t))\dot{\bm{\gamma}}(t)} \, \dif t,
\end{equation*}
\change{where $\dot{\bm{\gamma}}$ denotes the time-derivative of a smooth curve $\bm{\gamma}$ on $\mathcal{M}$ with $\bm{\gamma}(0)=\va \in \mathcal{M}$ and $\bm{\gamma}(1)=\vb \in \mathcal{M}$.} Geodesics are locally length minimizing and can be computed by minimizing the energy functional
\begin{equation} \label{eq:energy}
    \mathcal{E}(\bm{\gamma}) = \frac{1}{2}\int_{0}^{1}\dot{\bm{\gamma}}(t)^{\top}\mG\left(\bm{\gamma}(t)\right)\dot{\bm{\gamma}}(t)\,\dif t,
\end{equation}
or equivalently solving the Euler-Lagrange equations in Eq.~\ref{eq:ode_riemann}. In general, we will assume that the manifolds studied in this paper are \emph{geodesically complete} in the sense that between any two points $\va,\vb \in \mathcal{M}$, there exists at least one length-minimizing geodesic connecting the boundary points. Simple examples of Riemannian manifolds are the sphere used in climate data \citep{karpatne2017machinelearninggeoscienceschallenges, cnf_flow} and the $n$-dimensional torus used in protein modeling \citep{protein_lovell, rna_murray}. More abstract examples of Riemannian manifolds are learned manifolds using latent space models as the \textsc{vae}, where the decoder under sufficient regularity learns a smooth $d$-dimensional immersion \citep{shao2017riemanniangeometrydeepgenerative}. For a further discussion on Riemannian geometry and generative models, we refer to \citep{dggm}. 
%
%

\textbf{Generative models} aim to approximate the underlying data distribution. Although this objective is mostly shared across generative models, different models solve this problem in very different ways. \textsc{vae}'s learn the data distribution by maximizing a lower bound of the likelihood \citep{kingma2022autoencodingvariationalbayes}, autoregressive models model the conditional distribution \citep{autoregressive_model}, normalizing flows use the change of variables theorem \citep{normalizing_flows}, while score-based diffusion models learn the data distribution using the probability flow \textsc{ode} by learning the score function \citep{song2021scorebasedgenerativemodelingstochastic}.

All of these generative models can generate samples that approximate the data distribution, but do not allow for elaborate statistics on the generated samples, such as interpolation, mean, and so on. Different works equip generative models with a geometry, but this is typically based on strong assumptions about the data dimension and the generative model. \citet{shao2017riemanniangeometrydeepgenerative, arvanitidis2018latentspaceodditycurvature, wang2021geometrydeepgenerativeimage} propose using the learned metric in the latent space for the \textsc{vae} \citep{kingma2022autoencodingvariationalbayes} and the generative-adversarial network (\textsc{gan}) \citep{goodfellow2014generativeadversarialnetworks} to compute geodesics for interpolation and extraction of model information. Although these approaches provide insight into latent space models, they assume a latent space equipped with a Riemannian metric, and it has been observed that for the geodesic to have a high likelihood with respect to the data distribution, it requires incorporating data uncertainty directly into the generative model \citep{hauberg2019bayeslearnmanifoldon, tosi2014metricsprobabilisticgeometries, rvae, pmlr-v151-arvanitidis22b}. 

Diffusion models do not immediately learn a latent representation of the data, and the above constructions do not apply. Rather than applying the ambient Euclidean geometry, diffusion models usually compute interpolation using spherical and linear interpolation in the noise distribution that arises as the limit of the forward dynamics \citep{song2021scorebasedgenerativemodelingstochastic, song2020improvedtechniquestrainingscorebased, du2020implicitgenerationgeneralizationenergybased}, where 
\citet{zheng2024noisediffusioncorrectingnoiseimage, aid, yang2024impus, guo2024smooth} exploit the specific structure of diffusion models to generate highly realistic transitions between samples. These interpolants are based on strong assumptions of diffusion models and do not generalize to other generative models, such as Riemannian diffusion models, due to their ad hoc nature \citep{debortoli2022riemannianscorebasedgenerativemodelling, huang2022riemanniandiffusionmodels, jo2024generativemodelingmanifoldsmixture}. Alternatively, \citet{wang2023interpolatingimagesdiffusionmodels, yang2025versatile} train a diffusion model to generate realistic transitions between data, but the approach requires additional training and does not directly relate to any underlying generative model. In general, non-Riemannian methods do not extend beyond interpolation to support statistical measures such as the means or principal components. This limits their use in post hoc data analysis, which is often the long-term goal of geometric constructions.

Other works equip diffusion models with a Riemannian structure \citep{Yu2025ProbabilityDG, karczewski2025spacetimediffusionmodelsinformation, saito2025tangentialmanifolddiscoveringriemannian, bethune2025energy}, where \citet{karczewski2025spacetimediffusionmodelsinformation} considers an information geometric approach inspired by the Fisher-Rao metric. \citet{lebanon_learning_metric, Hartmann2022LagrangianMM, Yu2025ProbabilityDG, bethune2025energy} construct a generic Riemannian metric for generative models based on a probability density function, where \citet{bethune2025energy} considers a conformal Riemannian metric designed for energy-based models. 
Rather than incorporating the probabilistic structure into the Riemannian metric, we opt to consider the metric structure of the data used in training and penalize with a scalar field. Not only do we find that this gives better results compared to baselines, but it also reduces computational costs.

\section{Our likely geometry} \label{sec:likely_geometry}
\begin{figure*}[t]
    \centering
    \includegraphics[
      width=1.0\textwidth,
      height=0.4\textheight, 
      keepaspectratio
    ]{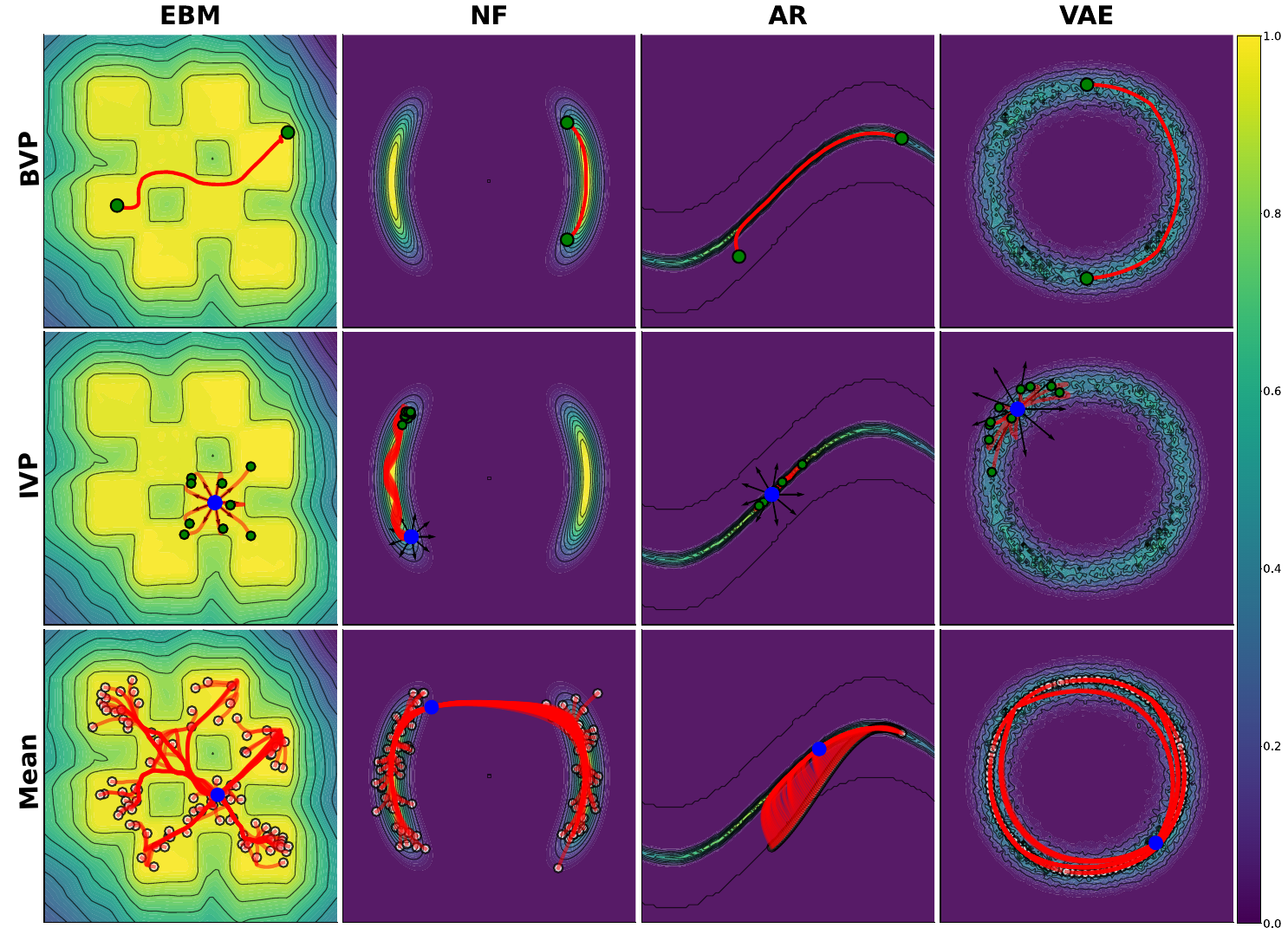}
    \caption{Examples of the three fundamental concepts derived from the (pseudo-)metric in Eq.~\ref{eq:metric_definition} with $\lambda=20.0$: boundary-value curves (row 1), initial value curves (row 2), and a generalized version of the Fr\'echet mean (row 3). The concepts are applied to an energy-based model (\textsc{ebm}) \citep{lecun2006tutorial} for a checkerboard, normalizing flow (\textsc{nf}) \citep{normalizing_flows} for the two-moon distributions, autoregressive model (\textsc{ar}) \citep{autoregressive_model} for a sinus curve, and a \textsc{vae} \citep{kingma2022autoencodingvariationalbayes} for circle data. Red denotes the connecting curves between the green boundary points. For the generalized Fr\'echet mean, blue denotes the mean, while the white points are $100$ sampled data points. The black arrows denote the initial directions of the curves. The regularization function for each model is described in Section~\ref{sec:experiments}.}
    \label{fig:multirow_plot_2d_gen}
    \vspace{-2.0em}
\end{figure*}

In this section, we describe how to equip generative models with our geometric framework. We do so by first stating the properties of interpolation and using this to define the notion of interpolation and means, as illustrated in Fig.~\ref{fig:multirow_plot_2d_gen}. In general, generative models often reside in a Riemannian manifold \citep{debortoli2022riemannianscorebasedgenerativemodelling, s-vae18, bjerregaard2025riemannian, rozo2025riemann}. For example, earthquake data used in Riemannian diffusion models \citep{debortoli2022riemannianscorebasedgenerativemodelling} are restricted to the sphere, while image data used to train diffusion models reside in Euclidean space \citep{song2021scorebasedgenerativemodelingstochastic, ho2020denoisingdiffusionprobabilisticmodels}. We aim to construct a generic geometric framework that favors high-density regions learned by any generative model, given the metric structure of the data. More formally, we will consider a Riemannian manifold $\mathcal{M}$, in which the generative process resides, and consider any function $S: \mathcal{M} \rightarrow \mathbb{R}$ that is bounded from below. 

\change{In this paper, we set the metric for all experiments to the metric of the ambient space on which the generative model relies and $S$ as the negative log-likelihood learned by the generative model, but this is not restrictive.} For a Riemannian manifold, geodesics can be found by minimizing the energy functional in Eq.~\ref{eq:energy}, but this disregards the learned data distribution. To ensure that the interpolating curves are attracted to the data distribution, we propose the following constrained minimization problem
\begin{equation} \label{eq:constrained_geodesic}
    \begin{split}
        \min_{\bm{\gamma}} \quad &\frac{1}{2}\int_{0}^{1}\dot{\bm{\gamma}}(t)^{\top}\mG\left(\bm{\gamma}(t)\right)\dot{\bm{\gamma}}(t)\,\dif t, \\
        \text{s.t.} \quad &\int_{0}^{1}S\left(\bm{\gamma}(t)\right)\,\dif t \leq \bar{S},
    \end{split}
\end{equation}
where $\bar{S} \in \mathbb{R}$ is a suitable bound and $\bm{\gamma}(0)=\va \in \mathcal{M}$ and $\bm{\gamma}(1)=\vb \in \mathcal{M}$. Minimizing this constrained optimization problem can be computationally expensive depending on $S$. We therefore consider the following minimization problem of the regularized energy functional, which corresponds to a Lagrange relaxation of the constraint in Eq.~\ref{eq:constrained_geodesic}.
\begin{equation} \label{eq:geodesic_lagrange_simplify}
    \bm{\gamma}^* = \min_{\bm{\gamma}} \mathcal{E}_{\lambda}(\va,\vb) \quad \text{   where} \quad \mathcal{E}_{\lambda} = \frac{1}{2}\int_{0}^{1}\!\!\dot{\bm{\gamma}}(t)^{\top}\mG\left(\bm{\gamma}(t)\right)\dot{\bm{\gamma}}(t)\,\dif t + \lambda \int_{0}^{1}\!\! S\left(\bm{\gamma}(t)\right) \dif t, 
\end{equation}
where $\lambda > 0$ is a dual variable. If $\vu^{\top}\mG(\vz)\vu$ and $S(\vz)$ are convex in $(\vz,\vu)$, then there exists a $\lambda$ such that the solution to this minimization problem Eq.~\ref{eq:geodesic_lagrange_simplify} is a solution to the minimization problem in Eq.~\ref{eq:constrained_geodesic}.
If this is not the case, there could be duality gaps between Eq.~\ref{eq:geodesic_lagrange_simplify} and \ref{eq:constrained_geodesic}, i.e., the solution to Eq.~\ref{eq:geodesic_lagrange_simplify} is an optimal solution, but there is not necessarily a $\lambda$ for all values of $\bar{S}$ in Eq.~\ref{eq:constrained_geodesic}. For the rest of the paper, we will consider the minimization problem in Eq.~\ref{eq:geodesic_lagrange_simplify} to reduce computational complexity and implicitly assume that a critical point is a minimum point.

\paragraph{Metric.}
The energy minimization in Eq.~\ref{eq:geodesic_lagrange_simplify} provides a direct trade-off between smooth curves with respect to the geometry and targeting high-likelihood areas defined by $S$, which is directly controlled by $\lambda$. Inspired by Eq.~\ref{eq:geodesic_lagrange_simplify}, we define the following (pseudo-)metric structure of the Riemannian manifold $\mathcal{M}$ and the tangent space $T_{\vz}\mathcal{M}$ for $\vz \in \mathcal{M}$.
\begin{definition}[Metric structure] \label{def:metric}
    Let $\left(\mathcal{M},g\right)$ be a Riemannian manifold and let $S: \mathcal{M} \rightarrow \mathbb{R}$ be bounded from below. We define the following (pseudo-)inner product in the tangent space $T_{\vz}\mathcal{M}$ for $\vz \in \mathcal{M}$
    \begin{equation} \label{eq:ip_definition}
        F_{\vz,\lambda}(\vv,\vw) = \vv^{\top}\mG(\vz)\vw+\lambda S(\vz),
    \end{equation}
    Using Eq.~\ref{eq:ip_definition}, we define the following (pseudo-)metric on $\mathcal{M}$
    \begin{equation} \label{eq:metric_definition}
        \mathrm{dist}^{2}_{\lambda}(\va,\vb) = \min_{\bm{\gamma}} \, \mathcal{E}_{\lambda}(\va,\vb)
    \end{equation}
    where $\va,\vb \in \mathcal{M}$ with $\bm{\gamma}(0)=\va$ and $\bm{\gamma}(1)=\vb$ and $\mathcal{E}_{\lambda}(\va,\vb)$ is the regularized energy Eq.~\ref{eq:geodesic_lagrange_simplify}.
\end{definition}
Since we have assumed that $\mathcal{M}$ is geodesically complete for the Riemannian background metric in which the data reside and that $S$ is bounded from below, there exists a solution to Eq.~\ref{eq:metric_definition}. Note that in Eq.~\ref{eq:metric_definition} we have removed $1/2$ to simplify the computations. \change{In the paper, we focus on computing high-likelihood statistics on generative models, where we apply the likelihood function as regularization function and the Euclidean background metric. However, the approach is not limited to this. Other constraints capturing the structure of generative models can easily be added such that the regularization function takes the form }
\begin{equation*}
    S(x) = \sum_{i=1}^{N_{\mathrm{constraints}}}\lambda_{i}S_{i}(x), \quad \lambda_{i}\geq 0, \, S_{i}: \mathcal{M} \rightarrow \mathbb{R}.
\end{equation*}
Since the scale between $S$ and $G$ might be very different, we show in Appendix~\ref{ap:exp_details} how to heuristically normalize $\lambda$, which we apply for all computations. Note that the phrasing, \emph{inner product}, in Definition~\ref{def:metric} is not necessarily correct. It is easily seen that $F$ generally does not satisfy positive-definiteness, i.e., $F_{\vz, \lambda}(v,v) > 0$ for $v \neq 0$. In this way, we can consider $S$ as a shift of the Riemannian inner product. Similarly, the (pseudo-)metric in Eq.~\ref{eq:metric_definition} is not a metric, as it can be negative depending on $S$. To simplify the wording, we will denote Eq.~\ref{eq:ip_definition} and Eq.~\ref{eq:metric_definition} as an inner product and metric, respectively. However, we show that we have a local representation in terms of a Riemannian metric along the optimal curve.
\begin{proposition}[Local Metric] \label{prop:local_metric}
    Let $\tilde{S}$ denote the lower bound of $S$. Assume that $\lambda$ is sufficiently large so that $S(\cdot)$ is close to $\tilde{S}$ in Eq.~\ref{eq:constrained_geodesic}. Let $\bm{\gamma}^{*}$ denote the optimal solution to Eq.~\ref{eq:metric_definition}. Then the regularized energy can locally along the optimal curve $\bm{\gamma}^{*}$ be estimated as
    \begin{equation*}
        \begin{split}
            \mathcal{E}(\bm{\gamma}) &\approx \tilde{S} + \int_{0}^{1}\left(\dot{\bm{\gamma}}^{*}\right)^{\top}(t)\mU\left(\bm{\gamma}^{*}(t)\right)\dot{\bm{\gamma}}^{*}(t)\,\dif t,
        \end{split}
    \end{equation*}
    where $\mU\left(\bm{\gamma}^{*}(t)\right) = \mG\left(\bm{\gamma}^{*}(t)\right)+\frac{\lambda}{2}\bm{H}(S)(\bm{\gamma}^{*}(t))$ and $\bm{H}(S)$ denotes the Euclidean Hessian of $S$. Thus, if $\mU\left(\bm{\gamma}^{*}(t)\right)$ is positive definite, we can interpret it as a local representation of a Riemannian metric along the optimal curve.
\end{proposition}
\begin{proof}
    See Appendix~\ref{ap:local_metric}.
\end{proof}
In the case $S = -\log p$, we see that in Proposition~\ref{prop:local_metric} we get the Hessian of the log-likelihood similar to the Fisher-Rao metric \citep{Nielsen_2020}. Thus, we can interpret our method as a trade-off between the naturally length-minimizing curves for the given metric structure of the space for $\lambda = 0$ and the curves strongly restricted to the data distribution as $\lambda$ increases. The exact choice of $\lambda$ is therefore a preference for the user that balances the smoothness and likelihood of the curves depending on the use-case, as illustrated in Fig.~\ref{fig:gmm_kde_examples} for simple synthetic data with a Gaussian mixture model (\textsc{gmm}) and a kernel density estimator (\textsc{kde}) with $G=I$ and $S=-\log p$.
\begin{wrapfigure}[19]{r}{0.52\textwidth}
    \vspace{-4.0mm}
    \centering
    \includegraphics[width=0.5\textwidth]{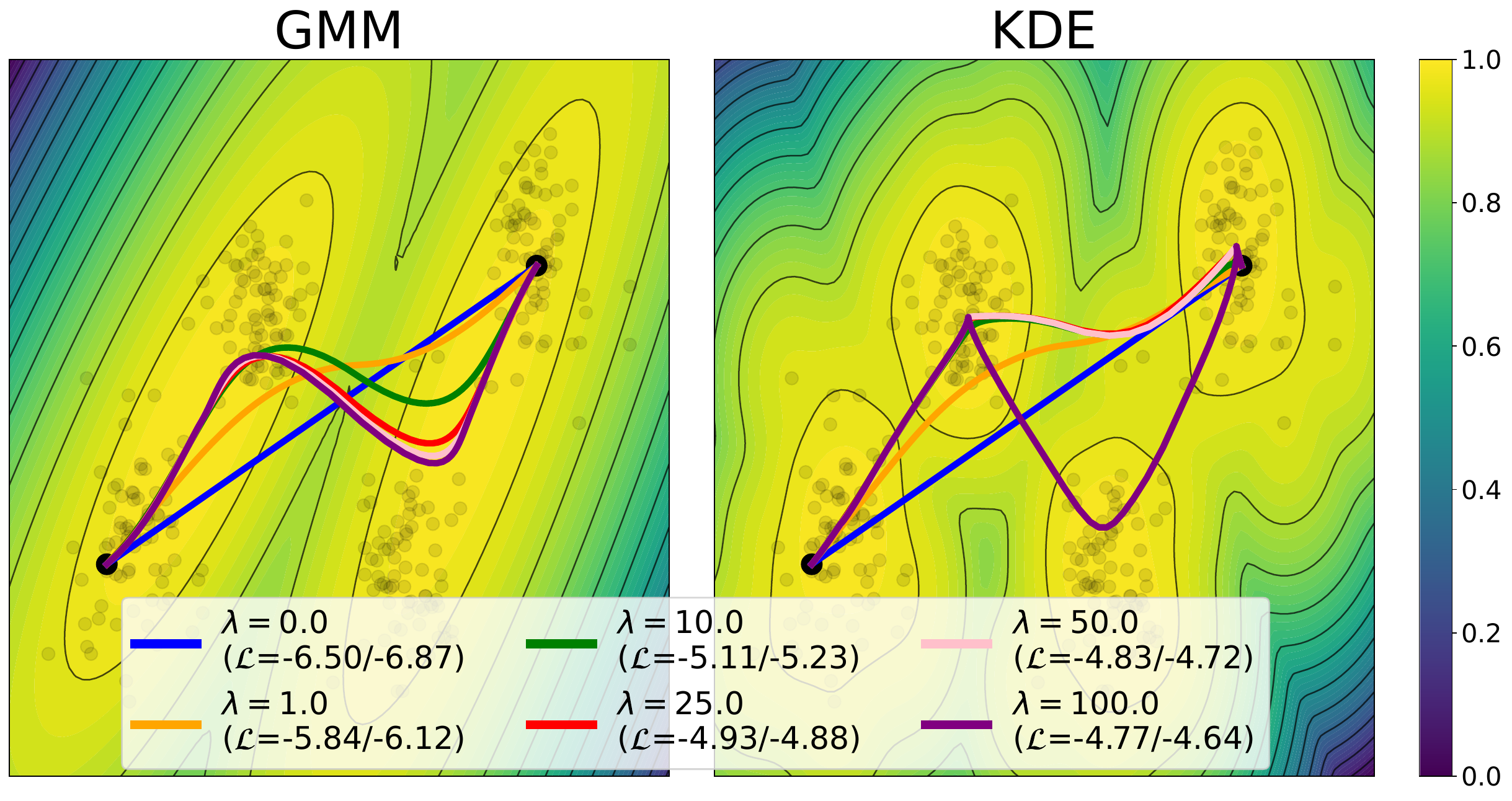}
    \caption{Interpolation curves computed for different values of $\lambda$ in Eq.~\ref{eq:geodesic_lagrange_simplify} with an Euclidean background metric and $S(\vz)=-\log p(\vz)$. The data density, $p$, is approximated by a \textsc{gmm} and \textsc{kde} shown in the background color for synthetic data (black). The mean log-likelihoods of the estimated curves are denoted $\mathcal{L}$, where the first number is for \textsc{gmm} and the second for \textsc{kde}.}
    \label{fig:gmm_kde_examples}
    \vspace{-0.5mm}
\end{wrapfigure}
\change{For a strictly Riemannian metric, the grid points along the curve should have the same distance. However, for the metric in Eq.~\ref{eq:geodesic_lagrange_simplify}, this does not have to be the case, since the regularization term only depends on the grid points. Therefore, our method can have a longer distance between grid points in low-density areas and a denser number of grid points in high-density regions. For lower values of $\lambda$, smoothness will play a bigger role, while a higher value of $\lambda$ will have higher-likelihood but more rapidly changing curves, as seen in Fig.~\ref{fig:gmm_kde_examples}. This effect is especially desired for multimodal distributions, where there does not exist a smooth curve with continuous high likelihood between two points. In this case, by increasing $\lambda$, the curve will increase the distance between grid points and thus reduce the ``time'' spent between the modes of the distribution. In practice, we find that increasing $\lambda$ too high will give very sudden jumps in the curves as seen in Fig.~\ref{fig:gmm_kde_examples}. Throughout the paper, we will consider $\lambda=20.0$ and use the normalization described in Appendix~\ref{ap:exp_additional}, which we find to nicely balance high-likelihood and smoothness.}

In the following, we will derive and define the properties of the metric in Eq.~\ref{eq:metric_definition}, as well as computational methods for fast and accurate computations for statistical analysis.

\paragraph{Minimizing curves.}
In the Riemannian case, geodesics serve as the foundation for elaborate statistical models and are deeply rooted in the exponential and logarithmic map that generalizes the notion of vector addition and subtraction to Riemannian manifolds \citep[Chapter 9]{dggm}. The exponential and logarithmic map, when it exists, can be found by solving the \textsc{ode} in Eq.~\ref{eq:ode_riemann} as an initial-value problem (\textsc{ivp}) or boundary-value problem (\textsc{bvp}), respectively. Given the metric in Eq.~\ref{eq:metric_definition}, we derive an \textsc{ode} that defines a generalized geodesic.
\begin{proposition}[First variation] \label{prop:ode_metric}
    Consider the regularized energy in Eq.~\ref{eq:metric_definition} for a Riemannian manifold $\left(\mathcal{M}, g\right)$. Let $g_{ij}$ and $g^{ij}$ denote the local coordinates of the metric matrix function and its inverse, respectively. The first variation gives the following \textsc{ode}
    \begin{equation} \label{eq:ode_solution}
        \ddot{\bm{\gamma}}^{k}(t) + \Gamma_{ij}^{k}\dot{\bm{\gamma}}^{s}(t)\dot{\bm{\gamma}}^{j}(t) = \frac{\lambda}{2} g^{kp}\partial_{p} S\left(\bm{\gamma}(t)\right),
    \end{equation}
    where $\Gamma_{ij}^{k} = \frac{1}{2}g^{kj}\left(\partial_{l} g_{pj} + \partial_{j} g_{pl} - \partial_{p}g_{ij}\right)$ denotes the Christoffel symbols of the Levi-Civita connection. Note that Eq.~\ref{eq:ode_solution} can equivalently be written as
    \begin{equation} \label{eq:ode_control}
        \dot{\bm{\gamma}}(t) = -\frac{1}{2}\mG(\bm{\gamma}(t))\mu(t), \quad \dot{\bm{\mu}}(t) = \frac{1}{4}\nabla_{x}\restr{\left(\bm{\mu}^{\top}(t)\mG(x)\bm{\mu}(t)\right)}{x=\bm{\gamma}(t)} - \lambda \nabla_{x}\restr{S(x)}{x=\bm{\gamma}(t)},
    \end{equation}
    where $\bm{\mu}(0)=-2\mG(\va)\dot{\bm{\gamma}}(0)$.
\end{proposition}
\begin{proof}
    See Appendix~\ref{ap:ode_metric} for the proof of Eq.~\ref{eq:ode_solution}. For the proof of Eq.~\ref{eq:ode_control}, we refer to \citet{Rygaard2026_thesis}.
\end{proof}
If $\lambda=0$, we see from Eq.~\ref{eq:ode_solution} that we immediately get the corresponding \textsc{ode} for Riemannian geodesics (\ref{eq:ode_riemann}). Note that Eq.~\ref{eq:ode_solution} is similar to the \textsc{ode} governing motion on a Riemannian manifold under external forces, where $\frac{\lambda}{2} g^{kp}\partial_{p} S\left(\bm{\gamma}(t)\right)$ can be interpreted as the external force \citep{newton_riemann_mechanics}. In this way, we can interpret $S$ as a force that pushes the curve into the data distribution, where the ``push'' is controlled by $\lambda$ corresponding to a Newtonian system on a Riemannian manifold. Note that the \textsc{ode} in Eq.~\ref{eq:ode_control} can be suitable in high dimensions as it avoids the cumbersome computation of the Christoffel symbols and only requires the gradient of an inner product. However, for all computations we will use the \textsc{ode} in Eq.~\ref{eq:ode_solution}. Note also that if Eq.~\ref{eq:metric_definition} is locally convex for the solution to Eq.~\ref{eq:ode_solution}, the solution to Eq.~\ref{eq:ode_solution} is a local minimum of Eq.~\ref{eq:metric_definition}.

Using Eq.~\ref{eq:ode_solution}, we can easily define the exponential map corresponding to the metric in Eq.~\ref{eq:metric_definition} as $\mathrm{Exp}_{x, \lambda}(v) = \bm{\gamma}(1)$,
where $\bm{\gamma}(1)$ is the solution to Eq.~\ref{eq:ode_solution} with $\bm{\gamma}(0)=\vx$ and $\dot{\bm{\gamma}}(0)=v$. Similarly, we can define, when it exists, the logarithmic map as $\mathrm{Log}_{x, \lambda}(\vy) = \dot{\bm{\gamma}}(0)$, the solution to Eq.~\ref{eq:ode_solution} with $\bm{\gamma}(0)=x$ and $\bm{\gamma}(1)=\vy$. The exponential map corresponds to an \textsc{ivp} that can be estimated numerically using standard \textsc{ode} solvers \citep{rk45, rk32, dop853, radau, bdf, odepack}. The logarithmic map corresponds to a \textsc{bvp} problem which, in principle, can be solved similarly using \textsc{bvp} solvers \citep{leapfrog_noakes, leapfrog_optimal_control, hennig2014probabilisticsolutionsdifferentialequations, arvanitidis2019fastrobustshortestpaths}, but below we will derive a fast and efficient algorithm to solve the \textsc{bvp} rather than using \textsc{ode} solvers.

\paragraph{Geometric statistics.} Using the modified version of the exponential and logarithmic map, we can compute geodesics targeting high-density regions. The metric structure in Eq.~\ref{eq:metric_definition} lets us go further and compute geometric statistics, such as means and principal components, on the manifold \citep{pennec2006statriemann, fletcher_pga}. This allows generative models to be incorporated into data analysis pipelines. Using Eq.~\ref{eq:metric_definition}, we define the discrete mean-value as
\begin{equation} \label{eq:frechet_energy}
    \mu, \{\gamma_{s}\}_{i=1}^{N_{\mathrm{data}}} = \argmin_{\substack{\vy \in \mathcal{M} \\ \{\gamma_{s}\}_{i=1}^{N_{\mathrm{data}}}}}\sum_{i=1}^{N_{\mathrm{data}}}w_{i}\mathcal{E}_{\lambda}\left(\va_{i}, \vy\right),
\end{equation}
where $\left\{\va_{i}\right\}_{i=1}^{N_{\mathrm{data}}} \subset \mathcal{M}$ are data points with weights $\left\{w_{i}\right\}_{i=1}^{N_{\mathrm{data}}} \subset \mathbb{R}_{+}$ and shortest curves $\{\gamma_{i}\}_{i=1}^{N_{\mathrm{data}}}$ connecting the data points and mean $\mu$. If $\lambda=0$, then this is equivalent to the Fr\'echet mean \citep{frechet1948}. Generally, uniqueness and existence of the Fr\'echet mean is not guaranteed even on Riemannian manifolds. In Appendix~\ref{ap:frechet_mean_prop}, we summarize the conditions for existence and uniqueness on Riemannian manifolds and state results for when this holds for Eq.~\ref{eq:frechet_energy}. From the definition of the mean in Eq.~\ref{eq:frechet_energy}, we can easily generalize well-known geometric statistics such as \emph{geodesic regression} \citep{fletcher_geo_reg} and \emph{principal geodesic analysis} (\textsc{pga}) \citep{fletcher_pga}. 

\paragraph{Algorithm.} Solving the \textsc{ivp} using the \textsc{ode} in Eq.~\ref{eq:ode_solution} is computationally cheap. However, computing geodesics as a \textsc{bvp} can even in the Riemannian case be computationally difficult and often scales poorly in high dimensions or exhibits slow convergence. \citet{georce} circumvent this through an iterative scheme to estimate geodesics with fast convergence using optimal control known as the \textit{GEORCE}-algorithm. However, their method does not support regularization of the geodesic. We generalize \textit{GEORCE} to compute boundary-value geodesics under regularization and prove that our extended method converges to a local minimum (global convergence) and exhibits local asymptotic quadratic convergence under certain additional conditions similar to \citet{georce}. We formulate the regularized geodesic problem in Eq.~\ref{eq:geodesic_lagrange_simplify} as a discrete control problem using a first-order approximation of the velocity along the curve in Eq.~\ref{eq:geodesic_lagrange_simplify}.
\begin{equation} \label{eq:control_problem}
    \begin{split}
        \min_{(\vz_{s},\vu_{s})} E(\vz) := \min_{(\vz_{s},\vu_{s})} \quad &\sum_{s=0}^{N_{\mathrm{grid}}-1}\left(\vu_{s}^{\top}\mG(\vz_{s})\vu_{s}+\lambda S(\vz_{s})\right) \\
        \text{s.t.} \quad &\vz_{s+1} = \vz_{s}+\vu_{s}, \quad s=0,\dots,N_{\mathrm{grid}}-1,\\
        &\vz_{0}=\va, \, \vz_{N_{\mathrm{grid}}}=\vb.
    \end{split}
\end{equation}
With this formulation, $\vz_{0:N_{\mathrm{grid}}}$ denotes the state variables, while $\vu_{0:(N_{\mathrm{grid}}-1)}$ denotes the control variable such that the control variables correspond to a discretization of the tangent vectors at the grid points of the curve. Since the regularizing function depends only on the state $\vz_{0:N_{\mathrm{grid}}}$, we can apply a similar approach as \citet{georce} to decompose the optimal control problem in Eq.~\ref{eq:control_problem} into convex subproblems in the control variables. Using this, we show that
\begin{proposition} \label{prop:riemann_cond}
    The necessary conditions for a minimum in Eq.~\ref{eq:control_problem} are
    \begin{equation} \label{eq:energy_opt_condtions}
        \begin{split}
            &2\mG(\vz_{s})\vu_{s}+\bm{\mu}_{s}=0, \quad s=0,\dots, N_{\mathrm{grid}}-1, \\
            &\vz_{s+1}=\vz_{s}+\vu_{s}, \quad s=0,\dots,N_{\mathrm{grid}}-1, \\
            &\restr{\nabla_{\vy}\left[\vu_{s}^{\top}\mG(\vy)\vu_{s} + \lambda S(\vy)\right]}{\vy=\vz_{s}}+\bm{\mu}_{s}=\bm{\mu}_{s-1}, \quad s=1,\dots,N_{\mathrm{grid}}-1. \\
            &\vz_{0}=\va, \vz_{N_{\mathrm{grid}}}=\vb,
        \end{split}
    \end{equation}
    where $\bm{\mu}_{s} \in \mathbb{R}^{d}$ for $s=0,\dots,N_{\mathrm{grid}}-1$.
\end{proposition}
\begin{proof}
    See Appendix~\ref{ap:pgeorce_cond}.
\end{proof}
The necessary conditions in Eq.~\ref{eq:energy_opt_condtions} can generally not be solved with respect to $\vu_{0:(N_{\mathrm{grid}}-1)}$ and $\bm{\mu}_{0:(N_{\mathrm{grid}}-1)}$. However, we can circumvent this iteratively. In iteration $k$ consider the state and control variables $\vz_{0:N_{\mathrm{grid}}}^{(k)}$ and $\vu_{0:(N_{\mathrm{grid}}-1)}^{(k)}$. We fix $\mG(\cdot)$ and the gradient term in iteration $k$ to define the following variables.
\begin{equation} \label{eq:fixed_variables}
    \begin{split}
        &\bm{\nu}_{s} := \restr{\nabla_{\vy}\left(\vu_{s}^{\top}\mG(\vy)\vu_{s} + \lambda S(\vy)\right)}{\vy=\vz_{s}^{(k)},\vu_{s}=\vu_{s}^{(k)}}, \quad s=1,\dots,N_{\mathrm{grid}}-1, \\
        &\mG_{s} := \mG\left(\vz_{s}^{(k)}\right), \quad s=0,\dots,N_{\mathrm{grid}}-1.
    \end{split}
\end{equation}
Inserting these variables into Eq.~\ref{eq:energy_opt_condtions}, we can solve the equations in closed-form, which naturally leads to the following update scheme.
\begin{proposition}[Update Scheme] \label{prop:update_scheme}
    The update scheme for $\vu_{s},\bm{\mu}_{s}$ and $\vz_{s}$ is
    \begin{equation} \label{eq:energy_update_scheme}
        \begin{split}
            &\bm{\mu}_{N_{\mathrm{grid}}-1} = \left(\sum_{s=0}^{N_{\mathrm{grid}}-1}\mG_{s}^{-1}\right)^{-1}\left(2(\va-\vb)-\sum_{s=0}^{N_{\mathrm{grid}}-1}\mG_{s}^{-1}\sum_{j>s}\bm{\nu}_{j}\right), \\
            &\vu_{s} = -\frac{1}{2}\mG_{s}^{-1}\left(\bm{\mu}_{N_{\mathrm{grid}}-1}+\sum_{j>s}\bm{\nu}_{j}\right), \quad s=0,\dots,N_{\mathrm{grid}}-1, \\
            &\vz_{s+1} = \vz_{s}+\vu_{s}, \quad s=1,\dots,N_{\mathrm{grid}}-1, \\
            &\vz_{0}=\va.
        \end{split}
    \end{equation}
\end{proposition}
\change{Some generative models, such as diffusion models, operate directly in the high-dimensional ambient Euclidean space, using the Euclidean metric throughout training and generation.} In this case $\mG=I$ and storing any matrix would be intractable. However, for a Euclidean metric, the update formulas reduce to
\begin{corollary}[Euclidean Update Scheme] \label{cor:update_scheme_euclidean}
    Assume $\mG=I$. The update scheme for $\vu_{s}$ and $\vz_{s}$ is
    \begin{equation} \label{eq:energy_update_scheme_euclidean}
        \begin{split}
            &\vu_{s} = \frac{\vb-\va}{N_{\mathrm{grid}}}+\frac{1}{2}\left(\frac{1}{N_{\mathrm{grid}}}\sum_{k=0}^{N_{\mathrm{grid}}-1}\sum_{j>k}\bm{\nu}_{j}-\sum_{j>s}\bm{\nu}_{j}\right)\, \quad s=0,\dots,N_{\mathrm{grid}}-1, \\
            &\vz_{s+1} = \vz_{s}+\vu_{s}, \quad s=0,\dots,N_{\mathrm{grid}}-1, \\
            &\vz_{0}=\va,
        \end{split}
    \end{equation}
    where $\bm{\nu}_{s} := \lambda \restr{\nabla_{\vy}S(\vy)}{\vy=\vz_{s}}$
\end{corollary}
Thus, for a Euclidean metric, the algorithm has the same complexity as a gradient descent method.
Given the update scheme in Proposition~\ref{prop:update_scheme}, the new iteration is found by applying line search with backtracking similar to \citep{georce}.
%
We denote the algorithm \textit{ProbGEORCE} (Probabilistic GEORCE) and display it in pseudo-code in Appendix~\ref{ap:algorithms}, where we apply line-search in the update scheme using backtracking with parameter $\rho=0.5$ similar to \citep{georce}. We note that compared to the original \textit{GEORCE}-algorithm for unregularized geodesic construction, the only difference between the algorithms is that $\bm{\nu}_{s}$ depends on $\lambda S (\vz_{s})$. Note that $\vu_{0}$ can be directly interpreted as the logarithmic map modulo scaling. If $S$ or $\mG$ are expensive to evaluate, then repeatedly evaluating the energy functional for line-search can be time-consuming and possibly intractable. In Appendix~\ref{ap:adaptive_update}, we show how to adaptively estimate $\alpha$ using an adaptive update scheme similar to \textit{ADAM} to avoid repeated evaluation of $\mG$ and $S$. If $S$ stems from a neural network, we will use the adaptive update scheme derived in Appendix~\ref{ap:adaptive_update} and otherwise we will use line-search to estimate the steps-size $\alpha$. In general, we set $N_{\mathrm{grid}}=100$, while for image interpolation we set $N_{\mathrm{grid}}=10$.

\begin{wraptable}{r}{0.4\textwidth}
\vspace{-2.3em}
\centering
\caption{Complexity of different optimization algorithms.}
\begin{tabularx}{0.4\textwidth}{Xcc}
\toprule
Method & Complexity \\
\midrule
Gradient descent & $\mathcal{O}\left(N_{\mathrm{grid}}d\right)$ \\
Quasi-Newton method & $\mathcal{O}\left(N_{\mathrm{grid}}^{2}d^{2}\right)$ \\
(Sparse) Newton method & $\mathcal{O}\left(N_{\mathrm{grid}}d^{3}\right)$ \\
\textit{ProbGEORCE} & $\mathcal{O}\left(N_{\mathrm{grid}}d^{3}\right)$  \\
\bottomrule
\end{tabularx}
\label{tab:complexity}
\vspace{-2.0em}
\end{wraptable}

\change{In Table~\ref{tab:complexity}, we show the complexity of our method compared to applying other methods for minimizing Eq.~\ref{eq:control_problem}. Note that since the Hessian of the first-order approximation of the energy with respect to all grid points is sparse, and that the Hessian of $S$ with respect to all grid points is a diagonal of block matrices, then the complexity of the Newton method is not $\mathcal{O}\left(N_{\mathrm{grid}}^{3}d^{3}\right)$. Also note that if $G=I$, then \textit{ProbGEORCE} has complexity $\mathcal{O}\left(N_{\mathrm{grid}}d\right)$, since the metric matrix inversion is avoided.}

In Appendix~\ref{ap:pgeorce_convergence} we show that \textit{ProbGEORCE} converges to a local minimum (global convergence) similar to the original \textit{GEORCE} algorithm \citep{georce}. Similar to \citep{georce}, we also show that \textit{ProbGEORCE} has locally asymptotic quadratic convergence if $S$ is an affine function. In practice, $S$ will not be an affine function, and therefore our method will be fast for sufficiently small values of $\lambda$, but as $\lambda$ increases, our method will not necessarily be faster than optimizers. This result seems intuitive, as when $\lambda$ increases the Riemannian energy is negligible, and therefore the problem is to minimize a general function $S$. For the generalized Fr\'echet mean in Eq.~\ref{eq:frechet_energy}, we can similarly consider the discretized version
\begin{equation}
    \begin{split}
        &\min_{\left(\vz_{s,i}, \vu_{s,i}\right)} \,\sum_{i=1}^{N_{\mathrm{data}}}\sum_{s=0}^{N_{\mathrm{grid}}}w_{i}\left(\vu_{s,i}^{\top}\mG(\vz_{s,i})\vu_{s,i} + \lambda S(\vz_{s,i})\right),  \\
    \end{split}
\end{equation}
where $\vz_{s+1,i} = \vu_{s,i} + \vz_{s,i}$ and $\vz_{0,i}=\va_{i} \in \mathcal{M}$ denote the data points for $i=1,\dots,N_{\mathrm{data}}$ and $\vz_{N_{\mathrm{grid}},i}=\vy \in \mathcal{M}$ denote the candidate mean.
Following a similar approach, we show in Appendix~\ref{ap:mean_value} how to estimate the mean in Eq.~\ref{eq:frechet_energy} by generalizing the approach from \citet{georce_frechet}.

\paragraph{Application to diffusion models.} Using our framework, we have defined a (pseudo-)metric that deviates from the background Riemannian metric with an external force field. A natural choice for this force is to set  $S = -\log p$, but this quantity is not available for diffusion models without having to solve the probability flow \textsc{ode}, which is expensive. However,  $\nabla_{\vz}S = -\nabla_{\vz} \log p$, is the only one needed for the \textsc{ode} in Eq.~\ref{eq:ode_solution}, solving the \textsc{bvp} in Eq.~ \ref{eq:energy_update_scheme} and computing the mean in Appendix~\ref{ap:frechet_mean_prop}. For diffusion models $\nabla_{\vz}S = -\nabla_{\vz} \log p$ is exactly the \emph{score} function and, therefore, our method is directly applicable to this case. However, for diffusion models, the score depends also on time. Given our framework, we can compute the regularized geodesics in the following two ways:
\begin{itemize}
    \item Curves in noise space: Given the limiting distribution of the diffusion model, we compute the interpolation in noise space and, using the time-reversal dynamics, transport the geodesic to data space. The regularization function can be set as the log-likelihood of the limiting distribution.
    \item Curves in data space: With our framework, we can also directly compute the geodesics in data space by applying the score function for time sufficiently close to 0 using the adaptive update scheme for \textit{ProbGEORCE} in Appendix~\ref{ap:adaptive_update}.
\end{itemize}

\begin{wrapfigure}[15]{r}{0.52\textwidth}
    \vspace{-1.2em}
    \centering
    \includegraphics[width=0.5\textwidth]{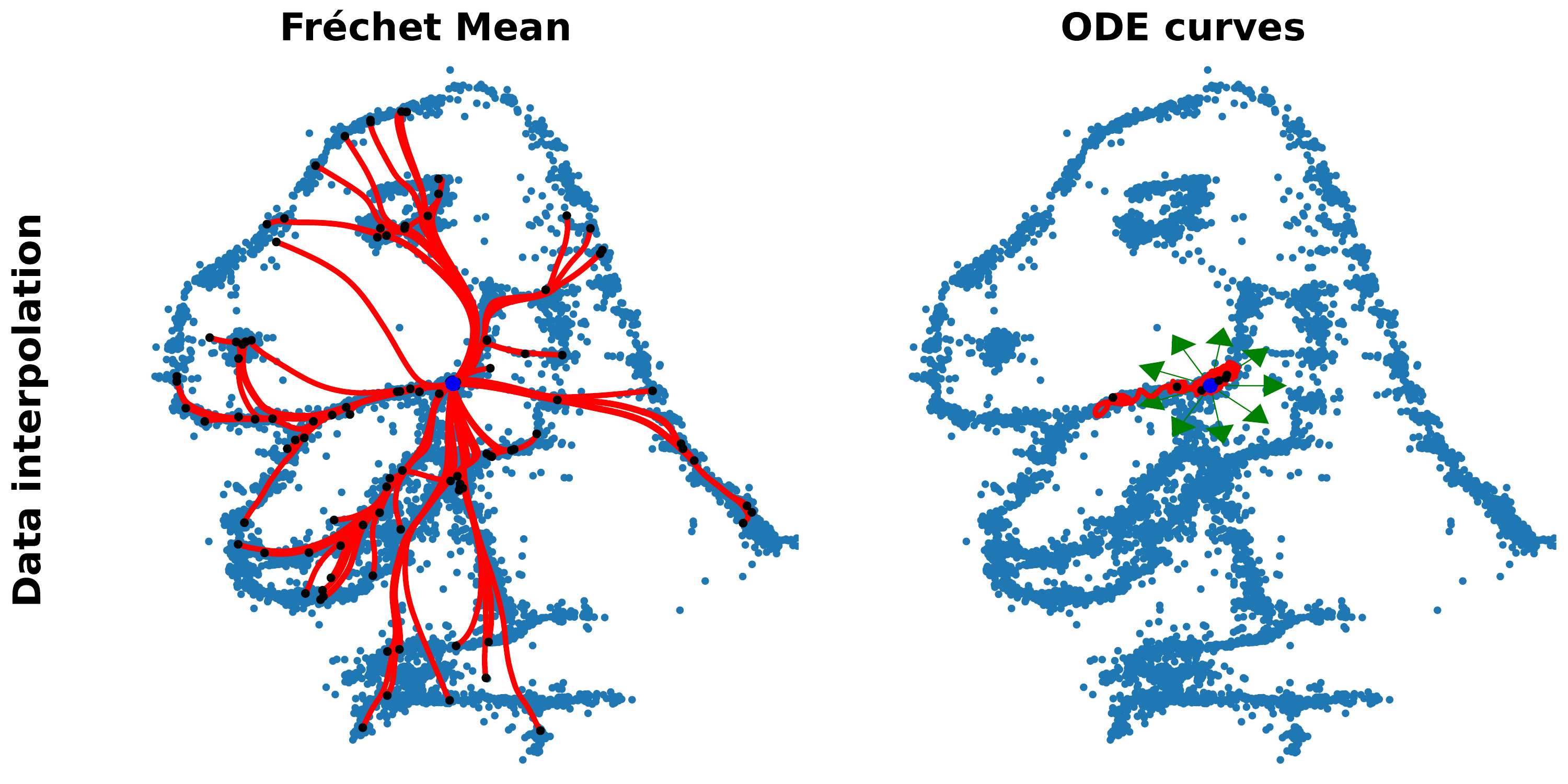}
    \vspace{-1.0em}
    \caption{The Fr\'echet mean (blue) for the dino-dataset using a \textsc{ddpm} computed in data space as well as 10 curves (red) in ten different directions (green arrows) for the corresponding \textsc{ivp} (\ref{eq:ode_solution}).}
    \label{fig:ddpm_2d_fmean_ode}
\end{wrapfigure}
Both approaches provide different pros and cons. Computing curves in noise space can be computationally cheaper, since the limiting distributions, and hence their density, are often known in closed form. However, these curves will only be smooth if the sampling scheme of the diffusion model is deterministic as for the probability flow \textsc{ode} \citep{chen2019neuralordinarydifferentialequations} or \textsc{ddim} sampling \citep{song2022denoisingdiffusionimplicitmodels}. Note also that high density in noise space does not necessarily imply high density in the data space by the change of variable formula. In contrast, computing the curves in the data space will result in smooth curves and use the learned score function directly. However, this is computationally more expensive, as it requires evaluating the diffusion model repeatedly in the optimization, and can be numerically unstable if the initial curve is not sufficiently close to the data to have a proper estimate of the score. In Fig.~\ref{fig:ddpm_2d_fmean_ode}, we show the corresponding Fr\'echet mean and \textsc{ivp} curves for the dino-\textsc{ddpm} computed in the data space. In Appendix~\ref{ap:app_diff_model}, we provide more details on diffusion models, the difference between interpolation in noise and data space, and also how our method can be incorporated into image editing methods such as NoiseDiffusion \citep{zheng2024noisediffusioncorrectingnoiseimage}.

\clearpage
\section{Experiments} \label{sec:experiments}
\begin{wrapfigure}[16]{r}{0.60\textwidth}
    \vspace{-1.5em}
    \centering
    \includegraphics[width=0.55\textwidth]{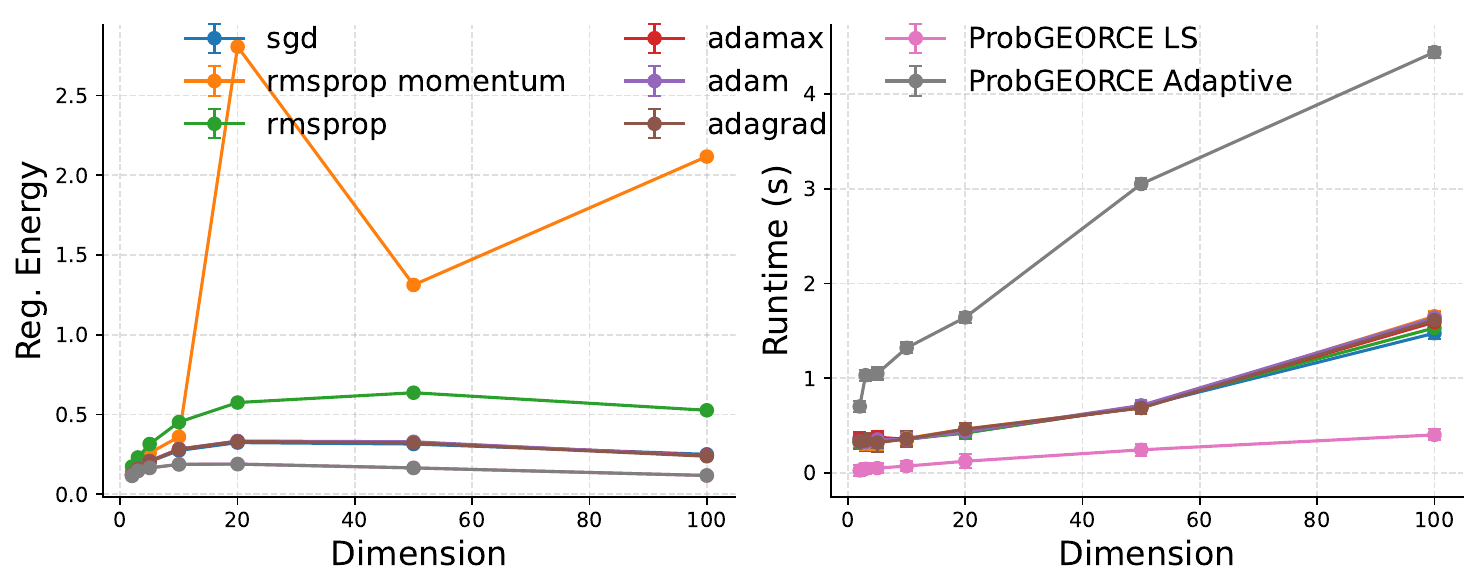}
    \caption{The regularized energy and runtime for different methods minimizing Eq.~\ref{eq:metric_definition} for $\lambda=1.0$ for the $n$-sphere and $S=-\log p$, where $p$ is the density of three Gaussians randomly weighted with random mean values and identity as the covariance matrix. Note that the regularized energy for \textit{ProbGEORCE} with line-search (LS) and adaptive update scheme (Adaptive) are very similar.}
    \label{fig:runtime_dim}
    \vspace{1.0em}
\end{wrapfigure}
\paragraph{Runtime.} Initially, to illustrate the efficiency of \textit{ProbGEORCE}, Fig.~\ref{fig:runtime_dim} considers the \textsc{bvp} between two points in a local chart of the $n$-sphere. We set $S=-\log p$, where $p$ denotes the data density of three randomly weighted isotropic Gaussian distributions with random means, and $\lambda=1.0$. We terminate the algorithms if the norm of the gradient of Eq.~\ref{eq:control_problem} is less than $10^{-4}$. We see that our algorithm obtains the lowest regularized energy across the different dimensions. We also see that \textit{ProbGEORCE} with line-search (LS) is significantly faster compared to using the adaptive scheme. This is expected as the evaluation of $S$ is computationally cheap. In Appendix~\ref{ap:add_runtime}, we show the runtime and estimates for other manifolds, and also for the models in Fig.~\ref{fig:multirow_plot_2d_gen} for different optimizers and values of $\lambda$, where $S$ is expensive to evaluate. In this case, the adaptive update scheme performs significantly better than line-search as expected. In Appendix~\ref{ap:add_runtime}, we also show the runtime and estimates for varying $\lambda$, where we see that if $\lambda$ is sufficiently high, then our method is similar to other optimizers. This is expected, as when $\lambda$ is sufficiently high, then the energy term in Eq.~\ref{eq:geodesic_lagrange_simplify} is negligible, and the problem is to minimize a general function $S$.
\begin{table*}[h]
\caption{Negative log-likelihood (NLL) and runtime results for the energy-based model, normalizing flow, autoregressive model and variational autoencoder used in Fig.~\ref{fig:multirow_plot_2d_gen}. For the \textsc{vae}, we report the \textsc{elbo}. We write $-$ if the Riemannian metric is not invertible or if the generative model returns NaNs for the likelihood. In Appendix~\ref{ap:exp_details}, we provide details on the benchmarks and any hyper-parameters.}
\centering
\resizebox{\textwidth}{!}{%
\begin{tabular}{lcccccccc}
\toprule
  & \multicolumn{2}{c}{\bfseries\large EBM} & \multicolumn{2}{c}{\bfseries\large NF} & \multicolumn{2}{c}{\bfseries\large AR} & \multicolumn{2}{c}{\bfseries\large VAE} \\
\cmidrule(lr){2-3} \cmidrule(lr){4-5} \cmidrule(lr){6-7} \cmidrule(lr){8-9} 
Method & NLL $\downarrow$ & Runtime $\downarrow$ & NLL $\downarrow$ & Runtime $\downarrow$ & NLL $\downarrow$ & Runtime $\downarrow$ & NLL $\downarrow$ & Runtime $\downarrow$ \\
\midrule
\multicolumn{9}{c}{\bfseries\large IVP} \\
\bottomrule
\addlinespace[0.3em]
Ours ($\lambda=20.0$) & \textbf{8.483} & 4.90 $\pm$ 0.09 & \pmb{$1.20 \times 10^{3}$} & 101.08 $\pm$ 2.17 & \pmb{$1.37 \times 10^{4}$} & 7.66 $\pm$ 0.08 & \textbf{-144.673} & 64.01 $\pm$ 0.49 \\
Euclidean & 30.010 & \textbf{0.00 $\pm$ 0.00} & $4.01 \times 10^{3}$ & \textbf{0.00 $\pm$ 0.00} & $1.06 \times 10^{5}$ & \textbf{0.00 $\pm$ 0.00} & -105.053 & \textbf{0.00 $\pm$ 0.00} \\
Spherical & 16.536 & 0.00 $\pm$ 0.00 & $4.10 \times 10^{3}$ & 0.00 $\pm$ 0.00 & $8.25 \times 10^{4}$ & 0.00 $\pm$ 0.00 & -85.552 & 0.00 $\pm$ 0.00 \\
Fisher-Rao & 29.682 & 12.61 $\pm$ 0.09 & - & - & - & - & - & - \\
Jacobian \citep{saito2025tangentialmanifolddiscoveringriemannian} & 29.682 & 18.50 $\pm$ 0.03 & - & - & $6.03 \times 10^{32}$ & 54.93 $\pm$ 0.58 & - & - \\
Jacobian (Reg) & 29.682 & 25.43 $\pm$ 0.07 & - & - & $5.30 \times 10^{29}$ & 61.82 $\pm$ 2.22 & - & - \\
Inverse Density \citep{Yu2025ProbabilityDG} & 28.673 & 9.14 $\pm$ 0.06 & $2.78 \times 10^{3}$ & 255.21 $\pm$ 1.63 & - & - & - & - \\
Generative \citep{Kim2024DeepGG} & 29.184 & 9.70 $\pm$ 0.08 & $3.87 \times 10^{3}$ & 225.10 $\pm$ 2.41 & $1.08 \times 10^{5}$ & 17.06 $\pm$ 0.12 & -120.631 & 212.38 $\pm$ 2.05 \\
Monge \citep{Hartmann2022LagrangianMM} & 29.682 & 12.64 $\pm$ 0.05 & $2.81 \times 10^{3}$ & 452.26 $\pm$ 2.78 & $9.51 \times 10^{4}$ & 31.83 $\pm$ 0.23 & -121.875 & 399.59 $\pm$ 3.87 \\
\addlinespace[0.5em]
\bottomrule
\addlinespace[0.3em]
\multicolumn{9}{c}{\bfseries\large BVP} \\
\bottomrule
\addlinespace[0.3em]
Ours ($\lambda=20.0$) & \textbf{-1.892} & 7.89 $\pm$ 0.02 & \textbf{94.492} & 217.49 $\pm$ 0.37 & \textbf{28.611} & 11.53 $\pm$ 0.20 & \textbf{-24.382} & 167.97 $\pm$ 3.81 \\
Euclidean & 2.435 & \textbf{0.00 $\pm$ 0.00} & 273.797 & \textbf{0.00 $\pm$ 0.00} & $6.05 \times 10^{3}$ & \textbf{0.00 $\pm$ 0.00} & 18.142 & \textbf{0.00 $\pm$ 0.00} \\
Spherical & 141.120 & 0.00 $\pm$ 0.00 & 103.706 & 0.00 $\pm$ 0.00 & $1.43 \times 10^{4}$ & 0.00 $\pm$ 0.00 & 948.467 & 0.00 $\pm$ 0.00 \\
Fisher-Rao & 0.863 & 21.99 $\pm$ 0.04 & - & 1.49 $\pm$ 0.00 & $1.73 \times 10^{14}$ & 37.51 $\pm$ 0.36 & - & - \\
Jacobian \citep{saito2025tangentialmanifolddiscoveringriemannian} & 2.435 & 2.52 $\pm$ 0.02 & - & - & $7.84 \times 10^{5}$ & $6.90 \times 10^{3}$ $\pm$ 23.44 & - & - \\
Jacobian (Reg) & 2.435 & 2.41 $\pm$ 0.01 & - & - & $7.15 \times 10^{5}$ & $6.98 \times 10^{3}$ $\pm$ 53.38 & - & - \\
Inverse Density \citep{Yu2025ProbabilityDG} & -1.451 & 16.41 $\pm$ 0.01 & 273.797 & 0.39 $\pm$ 0.00 & $6.05 \times 10^{3}$ & 0.04 $\pm$ 0.00 & - & - \\
Generative \citep{Kim2024DeepGG} & -1.256 & 16.63 $\pm$ 0.35 & 274.235 & 158.04 $\pm$ 0.70 & $6.05 \times 10^{3}$ & 0.04 $\pm$ 0.00 & 21.332 & 116.09 $\pm$ 0.41 \\
Monge \citep{Hartmann2022LagrangianMM} & 2.490 & 24.69 $\pm$ 0.23 & 174.743 & 553.14 $\pm$ 1.95 & $3.50 \times 10^{6}$ & 38.31 $\pm$ 0.15 & -6.788 & 314.86 $\pm$ 5.48 \\
\addlinespace[0.5em]
\bottomrule
\addlinespace[0.3em]
\multicolumn{9}{c}{\bfseries\large Mean} \\
\bottomrule
\addlinespace[0.3em]
Ours ($\lambda=20.0$) & \textbf{-79.383} & 9.78 $\pm$ 0.04 & \pmb{$2.14 \times 10^{4}$} & 188.16 $\pm$ 4.97 & \pmb{$2.54 \times 10^{4}$} & 8.80 $\pm$ 0.38 & \pmb{$-2.28 \times 10^{3}$} & 166.96 $\pm$ 0.10 \\
Euclidean & 153.896 & \textbf{0.00 $\pm$ 0.00} & $1.18 \times 10^{5}$ & \textbf{0.00 $\pm$ 0.00} & $5.62 \times 10^{5}$ & \textbf{0.00 $\pm$ 0.00} & $2.07 \times 10^{3}$ & \textbf{0.00 $\pm$ 0.00} \\
Spherical & $7.52 \times 10^{3}$ & 0.11 $\pm$ 0.00 & - & 0.11 $\pm$ 0.00 & $1.26 \times 10^{7}$ & 0.11 $\pm$ 0.00 & -558.952 & 0.14 $\pm$ 0.00 \\
Fisher-Rao & 28.670 & 30.32 $\pm$ 0.12 & - & 1.68 $\pm$ 0.02 & $7.20 \times 10^{18}$ & 0.12 $\pm$ 0.00 & - & - \\
Jacobian \citep{saito2025tangentialmanifolddiscoveringriemannian} & 153.896 & 247.24 $\pm$ 0.76 & - & - & - & - & - & - \\
Jacobian (Reg) & 153.896 & 283.84 $\pm$ 9.80 & - & - & - & - & - & - \\
Inverse Density \citep{Yu2025ProbabilityDG} & 153.896 & 0.03 $\pm$ 0.00 & $1.18 \times 10^{5}$ & 0.35 $\pm$ 0.00 & $5.62 \times 10^{5}$ & 0.04 $\pm$ 0.00 & $2.10 \times 10^{3}$ & 0.32 $\pm$ 0.00 \\
Generative \citep{Kim2024DeepGG} & 153.896 & 0.03 $\pm$ 0.00 & $1.18 \times 10^{5}$ & 0.45 $\pm$ 0.00 & $5.62 \times 10^{5}$ & 0.04 $\pm$ 0.00 & $2.10 \times 10^{3}$ & 0.29 $\pm$ 0.01 \\
Monge \citep{Hartmann2022LagrangianMM} & 152.098 & 9.90 $\pm$ 0.05 & $4.93 \times 10^{4}$ & 572.29 $\pm$ 18.44 & $2.21 \times 10^{6}$ & 47.44 $\pm$ 0.09 & $-1.19 \times 10^{3}$ & 356.90 $\pm$ 6.55 \\
\addlinespace[0.5em]
\bottomrule
\end{tabular}}
\vspace{-2.0em}
\label{tab:gen2d}
\end{table*}
\paragraph{Synthetic data.} We consider the four generative models and the corresponding data distributions in Fig.~\ref{fig:multirow_plot_2d_gen}. Energy-based models estimate the data distribution as $\exp\left(-E_{\theta}(\vx)\right)/Z_{\theta}$, where $E_{\theta}$ denotes a neural network and $Z_{\theta}$ is a normalization constant \citep{lecun2006tutorial}, while normalizing flow models estimate the log-likelihood of the data distribution applying the transformation of random variables \citep{normalizing_flows}. Neural autoregressive models learn the conditional distribution $\log p_{\theta}(\vx) = \sum_{i}\log p_{\theta}(\vx_{i} | \vx_{<i})$ \citep{autoregressive_model}. For all these models, we set $S=-\log p_{\theta}$, except for \textsc{vae}'s where we use the negative negative evidence lower bound (\textsc{elbo}). We show the result of applying our framework compared to other baselines in Table~\ref{tab:gen2d}. We see that our method, in general, obtains a higher likelihood compared to different baseline methods. \change{Since, by construction, our method (controlled by $\lambda$) targets high-density regions, we would also expect that this is the case compared to alternative methods. For a strictly Riemannian metric, the grid points along the curve should have the same distance. However, for the Newtonian system, this does not have to be the case, since the regularization term only depends on the grid points. Therefore, our method can have a longer distance between grid points in low-density areas and a denser number of grid points in high-density regions. For lower values of $\lambda$, smoothness will play a bigger role, while a higher value of $\lambda$ will have higher-likelihood but more rapidly changing curves. We illustrate in Fig.~\ref{fig:nf_lambda_changes} how the distance between grid points increases in low-density areas as $\lambda$ increases for the normalizing flow example in Fig.~\ref{fig:multirow_plot_2d_gen}. Note also that it is possible within our framework to combine a likelihood-based Riemannian metric with the likelihood regularization term.}
\begin{figure}[h!]
    \vspace{0.0em}
    \centering
    \includegraphics[width=1.0\textwidth]{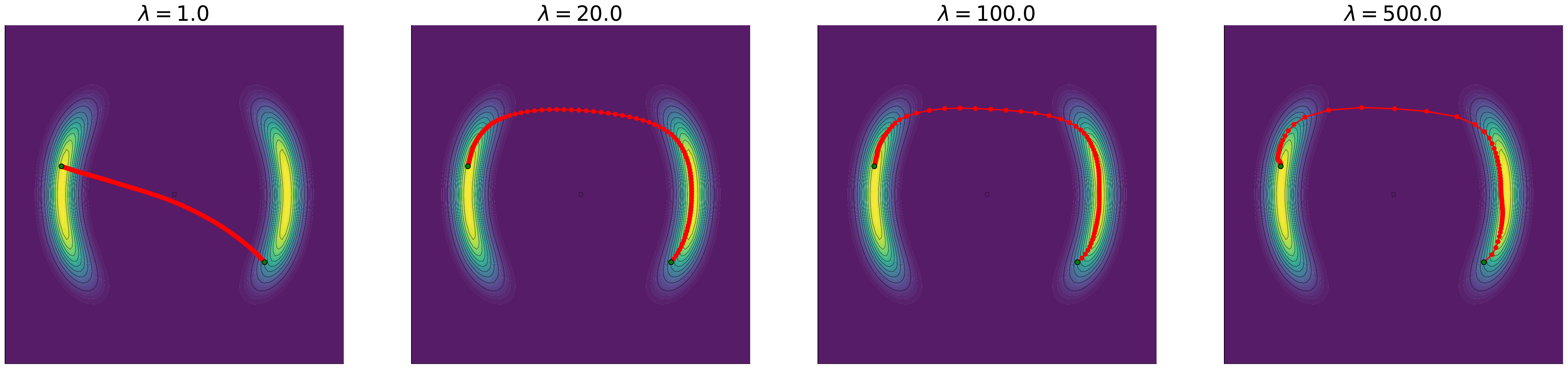}
    \caption{The distance between grid points for the normalizing flow model in Fig.~\ref{fig:multirow_plot_2d_gen} for varying values of $\lambda$.}
    \label{fig:nf_lambda_changes}
    \vspace{-1.0em}
\end{figure}

\begin{wraptable}{r}{0.5\textwidth}
\vspace{-2.15em}
\centering
\caption{Comparison of methods using FID and KID metrics (lower is better).}
\begin{tabularx}{0.5\textwidth}{Xcc}
\toprule
Method & FID $\downarrow$ & KID $\downarrow$ \\
\midrule
$\mG_{1/p_{\theta}}$ (conformal) & 187.25 & 0.0800 $\pm$ 0.0028 \\
$\mG_{\log p_{\theta}}$ (conformal) & 189.47 & 0.0819 $\pm$ 0.0033 \\
$\mG_{1/p_{\theta}}$ (diagonal) & 198.06 & 0.0828 $\pm$ 0.0032 \\
$\mG_{\log p_{\theta}}$ (diagonal) & 189.32 & 0.0745 $\pm$ 0.0031 \\
$\mG_{\mathrm{LAND}}$ & 186.87 & 0.0726 $\pm$ 0.0028 \\
$\mG_{1/p_{\theta}}$ (full) & 179.62 & 0.0676 $\pm$ 0.0024 \\
$\mG_{\log p_{\theta}}$ (full)& 197.63 & 0.0802 $\pm$ 0.0033 \\
Ours ($-1/p_{\theta}$) & \textbf{149.52} & 0.0375 $\pm$ 0.0018 \\
Ours ($-\log p_{\theta}$) & 150.50 & \textbf{0.0359 $\pm$ 0.0020} \\
\bottomrule
\end{tabularx}
\label{tab:energy_fid}
\vspace{-1.0em}
\end{wraptable}

\paragraph{Energy-based model.} Table~\ref{tab:gen2d} shows that our method traverses regions of higher-likelihood. To test whether our method also gives higher quality in images, we compare with the methods proposed in \citep{bethune2025energy}. Similarly to \citep{bethune2025energy}, we train an energy-based model \citep{lecun2006tutorial} on the Animal Faces High Quality dataset (AFHQ) \citep{choi2020starganv1} on the latent encodings by Stable Diffusion v1 VAE \citep{rombach2022highresolutionimagesynthesislatent}. We train the energy-based model for 24 hours corresponding to 3,000 epochs. We learn a neural network, $\phi_{\theta}$, to compute geodesics by minimizing the energy of the respective methods, which is trained for up to 24 hours for each method similar to \citep{bethune2025energy}. \citet{bethune2025energy} considers metrics on the following form
\begin{equation*}
    \mG(\vz) = \left(\alpha h(\vz) + \beta\right)\mI \quad \text{or} \quad \mG(\vz)= \left(\alpha h(\vz) + \beta\right)^{-1}\mI,
\end{equation*}
where $\alpha, \beta \in \mathbb{R}$ are constants and $h$ is either based on the log-likelihood $\log p_{\theta}$ or the inverse density $1/p_{\theta}$ learned by the energy based model. In Appendix~\ref{ap:exp_details}, we describe the metrics used in \citep{bethune2025energy}. 
\clearpage

\begin{wraptable}{r}{0.5\textwidth}
\vspace{-0.95em}
\centering
\caption{Comparison of methods using FID and KID metrics (lower is better) for the AFHQ dataset \citep{choi2020starganv2} for the \textsc{bvp} connecting two images.}
\begin{tabularx}{0.5\textwidth}{Xcc}
\toprule
Method & FID $\downarrow$ & KID $\downarrow$ \\
\midrule
Ours (Noise) & 182.73 & 0.12 $\pm$ 0.04 \\
Ours (Data) & \textbf{172.70} & 0.11 $\pm$ 0.04 \\
Linear & 200.62 & \textbf{0.08 $\pm$ 0.04} \\
Spherical & 184.21 & 0.12 $\pm$ 0.04 \\
NoiseDiffusion \citep{zheng2024noisediffusioncorrectingnoiseimage} & 174.46 & 0.11 $\pm$ 0.04 \\
\bottomrule
\end{tabularx}
\label{tab:fid_kid}
\vspace{-1.0em}
\end{wraptable}

We show the result in Table~\ref{tab:energy_fid}, where we consider our method with both $S(x)=-\log p_{\theta}$ and $S(x)=-1/p_{\theta}$, since \citet{bethune2025energy} finds that using the inverse density gives better results. We set $\lambda=20.0$ for our method and normalize $\lambda$ as described Appendix.~\ref{ap:exp_additional}. We see that our method gives better Fr\'echet inception distance (\textsc{fid}) \citep{fid} and kernel inception distance (\textsc{kid}) \citep{bińkowski2018demystifying} on the AFHQ dataset \citep{choi2020starganv1} compared to alternatives. Note that using the inverse density provides a marginally better \textsc{fid}, which is most likely due to the inverse density penalizing low-likelihood areas more. In Appendix~\ref{ap:exp_additional}, we show the image transitions for our method and the benchmarks.
\begin{wrapfigure}{r}{0.5\textwidth}
    \vspace{-1.2em}
    \centering
    \includegraphics[width=0.5\textwidth]{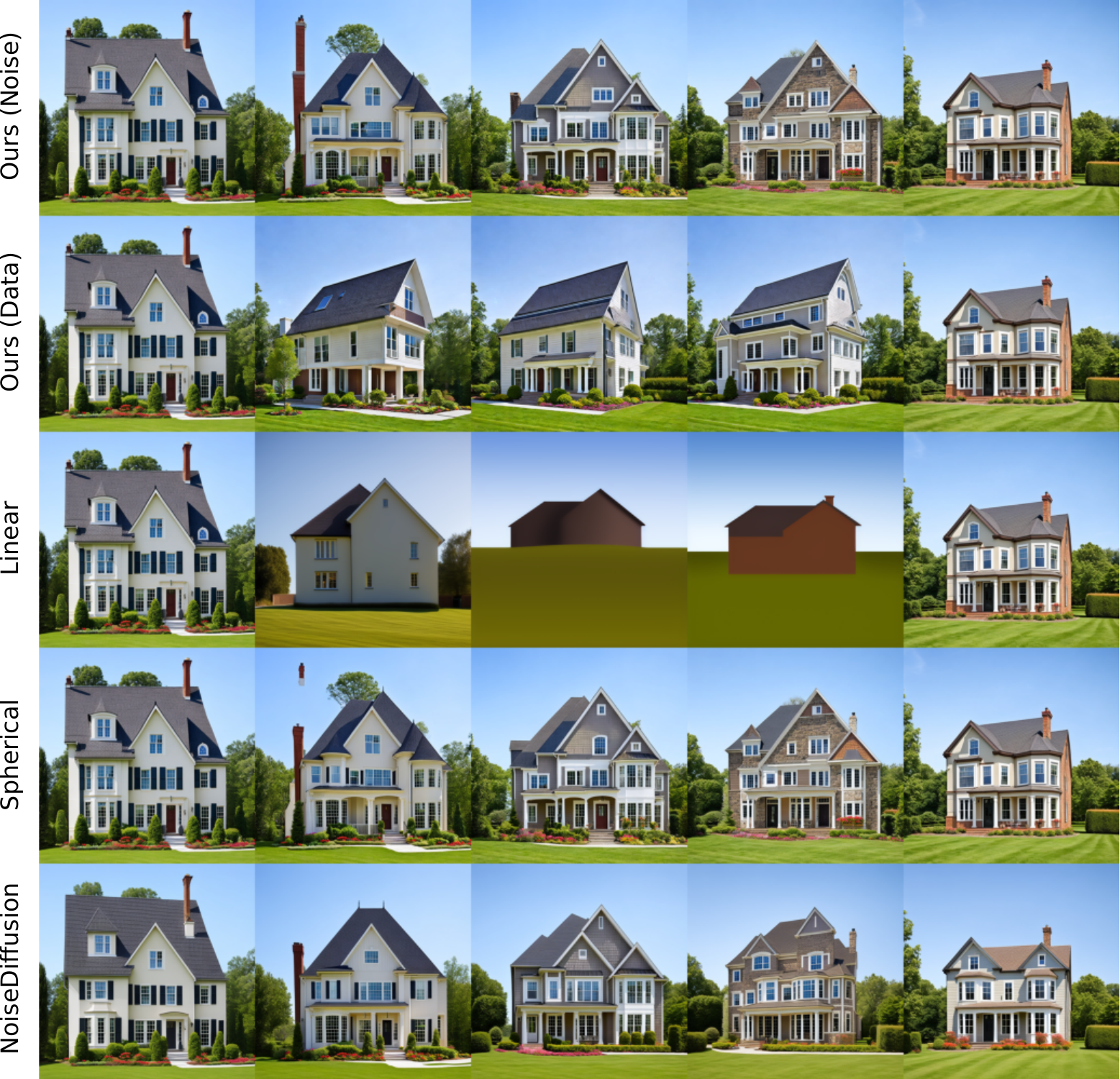}
    \caption{Interpolated images for different methods.}
    \label{fig:image_grid}
    \vspace{-2.0em}
\end{wrapfigure}
\paragraph{Diffusion models.}
To apply our framework to high-resolution diffusion models, we consider ControlNet \citep{zhang2023addingconditionalcontroltexttoimage}. 
We apply our method both directly in data space using the score proposed by \citet{katzir2024noisefree} for time 0, and in noise space using the density of a $\chi^{2}$-distribution on the squared norm. The latter is due to the fact that the size of the latent space is $4 \times 96 \times 96$, where the norm of an isotropic Gaussian prior is $\chi^{2}$-distributed and converges to a uniform distribution on a sphere. When computing curves in data space, we find that the images can get blurry. Therefore, when computing curves directly in the data space, we encode the computed curves in noise space and decode them back to data space, as we find that this gives smooth and realistic images. 

In Table~\ref{tab:fid_kid}, we compute \textsc{fid} \citep{fid} and \textsc{kid} \citep{bińkowski2018demystifying} on the AFHQ dataset \citep{choi2020starganv2} for our method and compare it with alternative methods. In general, for interpolation, we do not find a major difference between the different methods in terms of realism, as also seen in Fig.~\ref{fig:image_grid} for interpolation between houses. 

\begin{figure}[t!]
    \centering
    \includegraphics[
      width=1.0\columnwidth,
      height=0.3\textheight, 
      keepaspectratio
    ]{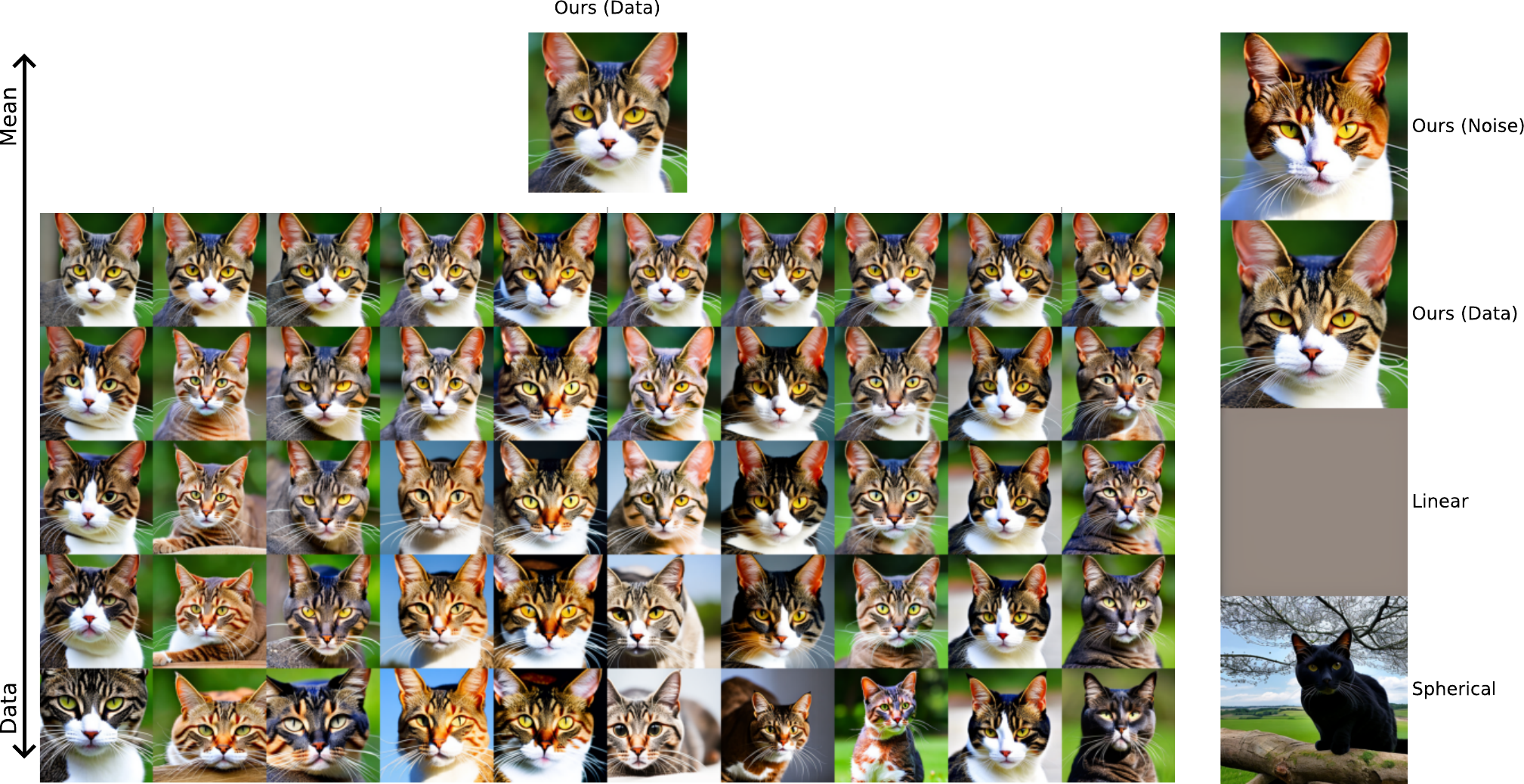}
    \caption{Left: Computed mean value using our method in data space, where the image shows the transition from the mean to the data points in the last row. Right: the estimated mean using our method in noise and data space, respectively, compared to the Fr\'echet mean using a linear and spherical geometry in noise space.}
    \label{fig:mean_grid_left}
    \vspace{-1.0em}
\end{figure}

\change{However, the difference can be seen more clearly when considering higher-order statistics such as the mean. In Fig.~\ref{fig:mean_grid_left}, we compute the mean value using our method in data space, where the last row shows the images over which the mean is calculated, as well as the mean for alternative methods by applying the Euclidean and spherical geometry in noise space. We see that spherical geometry yields a mean that decodes to a black cat, whereas our method yields a tabby cat. The primary reason is that the two methods solve different optimization problems and, therefore, estimate different means. The spherical approach computes the intrinsic Fr\'echet mean under the spherical Riemannian metric, whereas our method computes a Fr\'echet mean under a Euclidean metric regularized by the log-density of the $\chi^{2}$-distribution on the norms of the curve points. In high dimensions, this regularization encourages norms consistent with the typical Gaussian shell while still allowing radial variation, whereas the spherical formulation assumes a fixed norm throughout the optimization. Consequently, the two approaches need not produce the same representative point in the diffusion noise space. Since the mapping from diffusion noise to images is highly nonlinear, even relatively small differences between the estimated means can lead to noticeable semantic differences in the generated images. A secondary factor is that the two methods imply different optimization procedures. For the spherical geometry, we can efficiently estimate the Fr\'echet mean in the embedded space iteratively using the gradient
\begin{equation*}
    \sum_{i=1}^{N}\mathrm{Log}_{\mu}(z_{i})
\end{equation*}
For our method, we use \textit{ProbGEOCRCE}. This also means that it is possible to obtain a difference due to the different optimization algorithms. In Appendix~\ref{ap:app_diff_model}, we show the transition from the mean to the data points for all methods, where we also see that the spherical geometry has very rapid transitions to the data from the mean.}

\begin{figure*}[h!] 
    \centering
    \includegraphics[
      width=0.9\linewidth
    ]{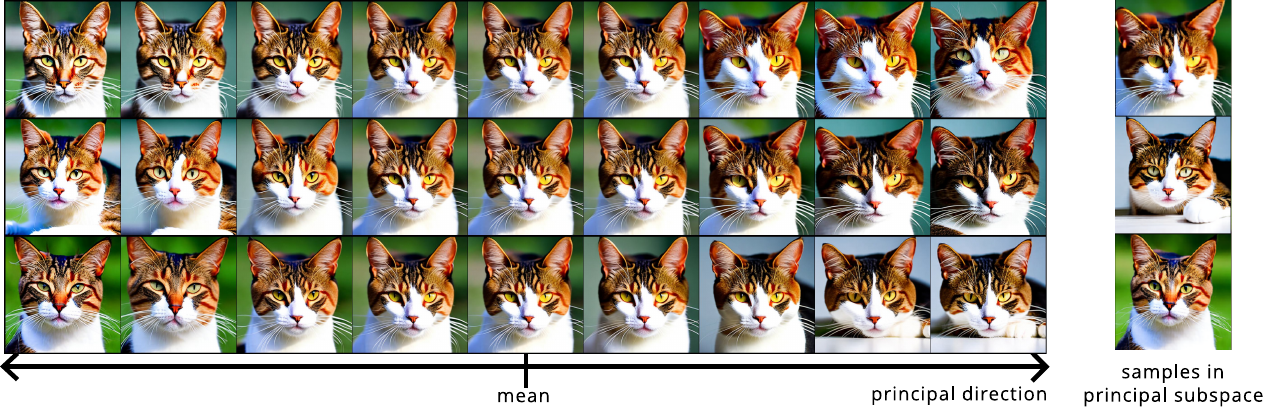}
    \caption{Three principal geodesics shown on the left. On the right, samples drawn within the principal subspace are shown.}
    \label{fig:pga}
    \vspace{-1.0em}
\end{figure*}

To further illustrate the potential of our method, we consider computing \textsc{pga} \citep{fletcher_pga} using our method, which generalizes principal component analysis to Riemannian manifolds. We consider the 10 images in Fig.~\ref{fig:mean_grid_left} as data and display the three principal geodesics from the mean using our method in noise space and three samples using the 3 principal directions. In this way, we can sample in the principal directions of samples generated by the generative model.

In Appendix~\ref{ap:exp_additional}, we provide additional qualitative results and also results for other diffusion models and score-based generative models \citep{song2021scorebasedgenerativemodelingstochastic}, where we also show results for different regularization functions $S$.

\section{Conclusion} \label{sec:conclusion}
We have proposed a general geometric framework to compute geometric statistics of generative models compatible with different metrics and probability distributions. We have shown that constructing geodesics that prefer high-likelihood areas does not naturally lead to a Riemannian metric, but instead to a Newtonian system on a Riemannian manifold. We have shown that shortest curves are characterized by a system of \textsc{ode}, and that this locally along the optimal curves corresponds to geodesics under a Riemannian metric. We have derived an algorithm with fast convergence to estimate interpolation and the mean value of generative models. Empirically, we have shown our method's applicability to different generative models and demonstrated that our interpolation curves are closer to regions with high likelihood. A key benefit of our approach is that it allows for higher-order statistical calculations, where we have demonstrated means and principal components. This shows that through a rigorous geometric construction, we can incorporate contemporary generative models into more traditional data analysis pipelines.

\paragraph{Limitations.}
Our method shows promising results for computing interpolation for different generative models but requires solving an optimization problem, which is more cumbersome than alternative methods such as linear and spherical interpolation. \change{Our method assumes that we have an approximation of the likelihood of the data distribution, and that we have access to a Riemannian metric that describes distances in the data space}. Furthermore, the optimal weight of the regularization $\lambda$ and the choice of the regularization function depend on the preferences of the user for smoothness versus high likelihood.

\subsubsection*{Broader Impact Statement} \label{sec:impact}
This work provides a geometrical framework for further statistical analysis of generative models enabling, e.g., interpolation. We have shown examples on images and synthetic data. Although this framework itself does not directly have any negative broader impact, it is possible to apply this framework to misuse a generative model. Since the framework is strongly dependent on the underlying generative model, any misuse can be avoided by imposing suitable restrictions on the generative model.

\subsection*{Acknowledgments} \label{sec:acknowledgments}
We are deeply grateful for the comments, suggestions, and discussions by Professor Steen Markvorsen in the process of writing this paper. FMR was supported by grants from \emph{Otto Mønsted Fond}, \emph{Knud Højgaard Fond}, \emph{Augustinus Fonden}, \emph{Reinholdt W. Jorck og Hustrus Fond} and \emph{Christian og Ottilia Brorsons Rejselegat} as part of a research stay at the University of Virginia. The project was further supported by research grant 42062 from VILLUM FONDEN; by the Novo Nordisk Foundation through the Center for Basic Machine Learning Research in Life Science (NNF20OC0062606); and by the European Research Council (ERC) under the European Union’s Horizon programme (grant agreement 101125993).
\bibliography{tmlr}

@article{whitney_embedding,
 ISSN = {0003486X, 19398980},
 URL = {http://www.jstor.org/stable/1968482},
 author = {Hassler Whitney},
 journal = {Annals of Mathematics},
 number = {3},
 pages = {645--680},
 publisher = {[Annals of Mathematics, Trustees of Princeton University on Behalf of the Annals of Mathematics, Mathematics Department, Princeton University]},
 title={{Differentiable Manifolds}},
 urldate = {2025-03-03},
 volume = {37},
 year = {1936}
}

@article{nash_embedding,
 ISSN = {0003486X, 19398980},
 URL = {http://www.jstor.org/stable/1969840},
 author = {John Nash},
 journal = {Annals of Mathematics},
 number = {3},
 pages = {383--396},
 publisher = {[Annals of Mathematics, Trustees of Princeton University on Behalf of the Annals of Mathematics, Mathematics Department, Princeton University]},
 title={{C1 Isometric Imbeddings}},
 urldate = {2025-03-03},
 volume = {60},
 year = {1954}
}

@phdthesis{Rygaard2026_thesis,
  author       = {Frederik M{\"o}bius Rygaard},
  title        = {{Algorithms for Geometric Statistics with Applications to Machine Learning}},
  year         = {2026},
  school       = {Technical University of Denmark},
  type         = {PhD thesis}
}

@inproceedings{bethune2025energy,
  title     = {Follow the Energy, Find the Path: Riemannian Metrics from Energy-Based Models},
  author    = {B{\'e}thune, Louis and Vigouroux, David and Du, Yilun and VanRullen, Rufin and Serre, Thomas and Boutin, Victor},
  booktitle = {Advances in Neural Information Processing Systems},
  year      = {2025},
  note      = {NeurIPS 2025},
  url       = {https://neurips.cc/virtual/2025/loc/san-diego/poster/119377}
}

@inproceedings{rozo2025riemann,
      title={Riemann$^2$: Learning Riemannian Submanifolds from Riemannian Data}, 
      author={Leonel Rozo and Miguel González-Duque and Noémie Jaquier and Søren Hauberg},
      year={2025},
      booktitle = {Proceedings of the 19th international Conference on Artificial Intelligence and Statistics (AISTATS)},
      journal = {Journal of Machine Learning Research, W\&CP},
}

@article{s-vae18,
  title={Hyperspherical Variational Auto-Encoders},
  author={Davidson, Tim R. and
          Falorsi, Luca and 
          De Cao, Nicola and
          Kipf, Thomas and
          Tomczak, Jakub M.},
  journal={34th Conference on Uncertainty in Artificial Intelligence (UAI-18)},
  year={2018}
}

@inproceedings{
vahdat2021scorebasedgenerativemodelinglatent,
title={Score-based Generative Modeling in Latent Space},
author={Arash Vahdat and Karsten Kreis and Jan Kautz},
booktitle={Advances in Neural Information Processing Systems},
editor={A. Beygelzimer and Y. Dauphin and P. Liang and J. Wortman Vaughan},
year={2021},
url={https://openreview.net/forum?id=P9TYG0j-wtG}
}

@inproceedings{pinaya2022brainimaginggenerationlatent,
author = {Pinaya, Walter H. L. and Tudosiu, Petru-Daniel and Dafflon, Jessica and Da Costa, Pedro F. and Fernandez, Virginia and Nachev, Parashkev and Ourselin, Sebastien and Cardoso, M. Jorge},
title = {Brain Imaging Generation with Latent Diffusion Models},
year = {2022},
isbn = {978-3-031-18575-5},
publisher = {Springer-Verlag},
address = {Berlin, Heidelberg},
url = {https://doi.org/10.1007/978-3-031-18576-2_12},
doi = {10.1007/978-3-031-18576-2_12},
booktitle = {Deep Generative Models: Second MICCAI Workshop, DGM4MICCAI 2022, Held in Conjunction with MICCAI 2022, Singapore, September 22, 2022, Proceedings},
pages = {117–126},
numpages = {10},
location = {Singapore, Singapore}
}

@INPROCEEDINGS {rombach2022highresolutionimagesynthesislatent,
author = { Rombach, Robin and Blattmann, Andreas and Lorenz, Dominik and Esser, Patrick and Ommer, Bjorn },
booktitle = { 2022 {IEEE}/{CVF} Conference on Computer Vision and Pattern Recognition (CVPR) },
title = {{ High-Resolution Image Synthesis with Latent Diffusion Models }},
year = {2022},
volume = {},
ISSN = {},
pages = {10674-10685},
doi = {10.1109/CVPR52688.2022.01042},
url = {https://doi.ieeecomputersociety.org/10.1109/CVPR52688.2022.01042},
publisher = {IEEE Computer Society},
address = {Los Alamitos, CA, USA},
month =Jun}

@article {LaMontagne2019.12.13.19014902,
	author = {LaMontagne, Pamela J. and Benzinger, Tammie LS. and Morris, John C. and Keefe, Sarah and Hornbeck, Russ and Xiong, Chengjie and Grant, Elizabeth and Hassenstab, Jason and Moulder, Krista and Vlassenko, Andrei G. and Raichle, Marcus E. and Cruchaga, Carlos and Marcus, Daniel},
	title = {OASIS-3: Longitudinal Neuroimaging, Clinical, and Cognitive Dataset for Normal Aging and Alzheimer Disease},
	elocation-id = {2019.12.13.19014902},
	year = {2019},
	doi = {10.1101/2019.12.13.19014902},
	publisher = {Cold Spring Harbor Laboratory Press},
	URL = {https://www.medrxiv.org/content/early/2019/12/15/2019.12.13.19014902},
	eprint = {https://www.medrxiv.org/content/early/2019/12/15/2019.12.13.19014902.full.pdf},
	journal = {medRxiv}
}

@inproceedings{fid,
author = {Heusel, Martin and Ramsauer, Hubert and Unterthiner, Thomas and Nessler, Bernhard and Hochreiter, Sepp},
title = {GANs trained by a two time-scale update rule converge to a local nash equilibrium},
year = {2017},
isbn = {9781510860964},
publisher = {Curran Associates Inc.},
address = {Red Hook, NY, USA},
booktitle = {Proceedings of the 31st International Conference on Neural Information Processing Systems},
pages = {6629–6640},
numpages = {12},
location = {Long Beach, California, USA},
series = {NIPS'17}
}

@inproceedings{
bińkowski2018demystifying,
title={Demystifying {MMD} {GAN}s},
author={Mikołaj Bińkowski and Dougal J. Sutherland and Michael Arbel and Arthur Gretton},
booktitle={International Conference on Learning Representations},
year={2018},
url={https://openreview.net/forum?id=r1lUOzWCW},
}

@inproceedings{choi2020starganv2,
  title={StarGAN v2: Diverse Image Synthesis for Multiple Domains},
  author={Yunjey Choi and Youngjung Uh and Jaejun Yoo and Jung-Woo Ha},
  booktitle={Proceedings of the IEEE Conference on Computer Vision and Pattern Recognition},
  year={2020}
}

@inproceedings{choi2020starganv1,
  title={StarGAN v1: Diverse Image Synthesis for Multiple Domains},
  author={Yunjey Choi and Youngjung Uh and Jaejun Yoo and Jung-Woo Ha},
  booktitle={Proceedings of the IEEE Conference on Computer Vision and Pattern Recognition},
  year={2018}
}

@INPROCEEDINGS{zhang2023addingconditionalcontroltexttoimage,
  author={Zhang, Lvmin and Rao, Anyi and Agrawala, Maneesh},
  booktitle={2023 IEEE/CVF International Conference on Computer Vision (ICCV)}, 
  title={Adding Conditional Control to Text-to-Image Diffusion Models}, 
  year={2023},
  volume={},
  number={},
  pages={3813-3824},
  doi={10.1109/ICCV51070.2023.00355}}

@inproceedings{
katzir2024noisefree,
title={Noise-free Score Distillation},
author={Oren Katzir and Or Patashnik and Daniel Cohen-Or and Dani Lischinski},
booktitle={The Twelfth International Conference on Learning Representations},
year={2024},
url={https://openreview.net/forum?id=dlIMcmlAdk}
}

@article{normalizing_flows,
author = {Papamakarios, George and Nalisnick, Eric and Rezende, Danilo Jimenez and Mohamed, Shakir and Lakshminarayanan, Balaji},
title = {{Normalizing flows for probabilistic modeling and inference}},
year = {2021},
issue_date = {January 2021},
publisher = {JMLR.org},
volume = {22},
number = {1},
issn = {1532-4435},
journal = {J. Mach. Learn. Res.},
month = jan,
articleno = {57},
numpages = {64}
}

@incollection{lecun2006tutorial,
  title={{A Tutorial on Energy-Based Learning}},
  author={LeCun, Yann and Chopra, Sumit and Hadsell, Raia and Ranzato, Marc’Aurelio and Huang, Fu Jie},
  booktitle={{Predicting Structured Data}},
  editor={Bakir, Ghulent and Hofmann, Thomas and Sch{\"o}lkopf, Bernhard and Smola, Alex J. and Taskar, Ben},
  publisher={MIT Press},
  year={2006},
  note={To appear},
  url={http://yann.lecun.com/exdb/publis/pdf/lecun-06.pdf}
}

@article{autoregressive_model,
author = {Uria, Benigno and C\^{o}t\'{e}, Marc-Alexandre and Gregor, Karol and Murray, Iain and Larochelle, Hugo},
title = {{Neural autoregressive distribution estimation}},
year = {2016},
issue_date = {January 2016},
publisher = {JMLR.org},
volume = {17},
number = {1},
issn = {1532-4435},
journal = {J. Mach. Learn. Res.},
month = jan,
pages = {7184–7220},
numpages = {37}
}

@article{Stimper2023, 
  author = {Vincent Stimper and David Liu and Andrew Campbell and Vincent Berenz and Lukas Ryll and Bernhard Schölkopf and José Miguel Hernández-Lobato}, 
  title = {{normflows: A PyTorch Package for Normalizing Flows}}, 
  journal = {Journal of Open Source Software}, 
  volume = {8},
  number = {86}, 
  pages = {5361}, 
  publisher = {The Open Journal}, 
  doi = {10.21105/joss.05361}, 
  url = {https://doi.org/10.21105/joss.05361}, 
  year = {2023}
}

@misc{torchebm_library_2025,
  author       = {Ghaderi, Soran and Contributors},
  title        = {{TorchEBM: A PyTorch Library for Training Energy-Based Models}},
  year         = {2025},
  url          = {https://github.com/soran-ghaderi/torchebm},
}

@article{Hartmann2022LagrangianMM,
  title={{Lagrangian Manifold Monte Carlo on Monge Patches}},
  author={Marcelo Hartmann and Mark A. Girolami and Arto Klami},
  journal={ArXiv},
  year={2022},
  volume={abs/2202.00755},
  url={https://api.semanticscholar.org/CorpusID:246473264}
}

@article{Kim2024DeepGG,
  title={{(Deep) Generative Geodesics}},
  author={Beomsu Kim and Michael Puthawala and Jong Chul Ye and Emanuele Sansone},
  journal={ArXiv},
  year={2024},
  volume={abs/2407.11244},
  url={https://api.semanticscholar.org/CorpusID:271218401}
}

@inproceedings{lebanon_learning_metric,
author = {Lebanon, Guy},
title = {{Learning riemannian metrics}},
year = {2002},
isbn = {0127056645},
publisher = {Morgan Kaufmann Publishers Inc.},
address = {San Francisco, CA, USA},
booktitle = {Proceedings of the Nineteenth Conference on Uncertainty in Artificial Intelligence},
pages = {362–369},
numpages = {8},
location = {Acapulco, Mexico},
series = {UAI'03}
}

@article{Kendall1990ProbabilityCA,
  title={{Probability, Convexity, and Harmonic Maps with Small Image I: Uniqueness and Fine Existence}},
  author={Wilfrid S. Kendall},
  journal={Proceedings of The London Mathematical Society},
  year={1990},
  volume={61},
  pages={371-406},
  url={https://api.semanticscholar.org/CorpusID:121554869}
}

@article{bpc2003,
author = {Rabi Bhattacharya and Vic Patrangenaru},
title = {{Large sample theory of intrinsic and extrinsic sample means on manifolds}},
volume = {31},
journal = {The Annals of Statistics},
number = {1},
publisher = {Institute of Mathematical Statistics},
pages = {1 -- 29},
year = {2003},
doi = {10.1214/aos/1046294456},
URL = {https://doi.org/10.1214/aos/1046294456}
}

@Inbook{Ziezold1977,
author="Ziezold, Herbert",
editor="Ko{\v{z}}e{\v{s}}nik, J.",
title={{"On Expected Figures and a Strong Law of Large Numbers for Random Elements in Quasi-Metric Spaces"}},
bookTitle="Transactions of the Seventh Prague Conference on Information Theory, Statistical Decision Functions, Random Processes and of the 1974 European Meeting of Statisticians: held at Prague, from August 18 to 23, 1974",
year="1977",
publisher="Springer Netherlands",
address="Dordrecht",
pages="591--602",
isbn="978-94-010-9910-3",
doi="10.1007/978-94-010-9910-3_63",
url="https://doi.org/10.1007/978-94-010-9910-3_63"
}

@inproceedings{
bjerregaard2025riemannian,
title={Riemannian generative decoder},
author={Andreas Bjerregaard and S{\o}ren Hauberg and Anders Krogh},
booktitle={ICML 2025 Generative AI and Biology (GenBio) Workshop},
year={2025},
url={https://openreview.net/forum?id=5i4ABK5QQp}
}

@article{pennec2006statriemann,
author = {Pennec, Xavier},
year = {2006},
month = {07},
pages = {127-154},
title = {{Intrinsic Statistics on Riemannian Manifolds: Basic Tools for Geometric Measurements}},
volume = {25},
journal = {Journal of Mathematical Imaging and Vision},
doi = {10.1007/s10851-006-6228-4}
}

@ARTICLE{fletcher_geo_reg,
  author = {Fletcher, Thomas},
  journal={Proceedings of the Third International Workshop on Mathematical Foundations of Computational Anatomy - Geometrical and Statistical Methods for Modelling Biological Shape Variability}, 
  title = {{Geodesic Regression on Riemannian Manifolds}}, 
  year={2011},
  pages={75-86},
}

@ARTICLE{fletcher_pga,
  author={Fletcher, P.T. and Conglin Lu and Pizer, S.M. and Sarang Joshi},
  journal={IEEE Transactions on Medical Imaging}, 
  title={{Principal geodesic analysis for the study of nonlinear statistics of shape}}, 
  year={2004},
  volume={23},
  number={8},
  pages={995-1005},
  doi={10.1109/TMI.2004.831793}
}

@article{karcher_1977,
author = {Karcher, H.},
title = {{Riemannian center of mass and mollifier smoothing}},
journal = {Communications on Pure and Applied Mathematics},
volume = {30},
number = {5},
pages = {509-541},
doi = {https://doi.org/10.1002/cpa.3160300502},
url = {https://onlinelibrary.wiley.com/doi/abs/10.1002/cpa.3160300502},
eprint = {https://onlinelibrary.wiley.com/doi/pdf/10.1002/cpa.3160300502},
year = {1977}
}

@article{afsari,
 ISSN = {00029939, 10886826},
 URL = {http://www.jstor.org/stable/41059320},
 author = {Bijan Afsari},
 journal = {Proceedings of the American Mathematical Society},
 number = {2},
 pages = {655--673},
 publisher = {American Mathematical Society},
 title = {{Riemannian Lp Center of Mass: Existence, Uniqueness, and Convexity}},
 urldate = {2026-01-14},
 volume = {139},
 year = {2011}
}

@article{Nielsen_2020,
   title={{An Elementary Introduction to Information Geometry}},
   volume={22},
   ISSN={1099-4300},
   url={http://dx.doi.org/10.3390/e22101100},
   DOI={10.3390/e22101100},
   number={10},
   journal={Entropy},
   publisher={MDPI AG},
   author={Nielsen, Frank},
   year={2020},
   month=sep, pages={1100} }

@InProceedings{pmlr-v151-arvanitidis22b,
  title = 	 {{ Pulling back information geometry }},
  author =       {Arvanitidis, Georgios and Gonz\'alez-Duque, Miguel and Pouplin, Alison and Kalatzis, Dimitrios and Hauberg, Soren},
  booktitle = 	 {Proceedings of The 25th International Conference on Artificial Intelligence and Statistics},
  pages = 	 {4872--4894},
  year = 	 {2022},
  editor = 	 {Camps-Valls, Gustau and Ruiz, Francisco J. R. and Valera, Isabel},
  volume = 	 {151},
  series = 	 {Proceedings of Machine Learning Research},
  month = 	 {28--30 Mar},
  publisher =    {PMLR},
  url = 	 {https://proceedings.mlr.press/v151/arvanitidis22b.html}
}

@inproceedings{yang2025versatile,
        title={{Versatile Transition Generation with Image-to-Video Diffusion}},
        author={Yang, Zuhao and Zhang, Jiahui and Yu, Yingchen and Lu, Shijian and Bai, Song},
        booktitle={Proceedings of the IEEE/CVF International Conference on Computer Vision},
        pages={16981--16990},
        year={2025}
      }

@InProceedings{guo2024smooth,
  title={{Smooth Diffusion: Crafting Smooth Latent Spaces in Diffusion Models}},
  author={Jiayi Guo and Xingqian Xu and Yifan Pu and Zanlin Ni and Chaofei Wang and Manushree Vasu and Shiji Song and Gao Huang and Humphrey Shi},
  booktitle={Proceedings of the IEEE/CVF Conference on Computer Vision and Pattern Recognition (CVPR)},
  year={2024}
}

@inproceedings{
yang2024impus,
title={{{IMPUS}: Image Morphing with Perceptually-Uniform Sampling Using Diffusion Models}},
author={Zhaoyuan Yang and Zhengyang Yu and Zhiwei Xu and Jaskirat Singh and Jing Zhang and Dylan Campbell and Peter Tu and Richard Hartley},
booktitle={The Twelfth International Conference on Learning Representations},
year={2024},
url={https://openreview.net/forum?id=gG38EBe2S8}
}

@inproceedings{aid,
author = {He, Qiyuan and Wang, Jinghao and Liu, Ziwei and Yao, Angela},
title = {{Attention interpolation for text-to-image diffusion}},
year = {2024},
isbn = {9798331314385},
publisher = {Curran Associates Inc.},
address = {Red Hook, NY, USA},
booktitle = {Proceedings of the 38th International Conference on Neural Information Processing Systems},
articleno = {3101},
numpages = {34},
location = {Vancouver, BC, Canada},
series = {NIPS '24}
}

@article{Yu2025ProbabilityDG,
  title={{Probability Density Geodesics in Image Diffusion Latent Space}},
  author={Qingtao Yu and Jaskirat Singh and Zhaoyuan Yang and Peter H. Tu and Jing Zhang and Hongdong Li and Richard Hartley and Dylan Campbell},
  journal={2025 IEEE/CVF Conference on Computer Vision and Pattern Recognition (CVPR)},
  year={2025},
  pages={27989-27998},
  url={https://api.semanticscholar.org/CorpusID:277634289}
}

@misc{karczewski2025spacetimediffusionmodelsinformation,
      title={{The Spacetime of Diffusion Models: An Information Geometry Perspective}}, 
      author={Rafał Karczewski and Markus Heinonen and Alison Pouplin and Søren Hauberg and Vikas Garg},
      year={2025},
      eprint={2505.17517},
      archivePrefix={arXiv},
      primaryClass={cs.LG},
      url={https://arxiv.org/abs/2505.17517}, 
}

@inproceedings{rvae,
author = {Kalatzis, Dimitris and Eklund, David and Arvanitidis, Georgios and Hauberg, S\o{}ren},
title = {{Variational autoencoders with Riemannian Brownian motion priors}},
year = {2020},
publisher = {JMLR.org},
booktitle = {Proceedings of the 37th International Conference on Machine Learning},
articleno = {469},
numpages = {14},
series = {ICML'20}
}

@misc{saito2025tangentialmanifolddiscoveringriemannian,
      title={{Be Tangential to Manifold: Discovering Riemannian Metric for Diffusion Models}}, 
      author={Shinnosuke Saito and Takashi Matsubara},
      year={2025},
      eprint={2510.05509},
      archivePrefix={arXiv},
      primaryClass={cs.CV},
      url={https://arxiv.org/abs/2510.05509}, 
}

@InProceedings{hennig2014probabilisticsolutionsdifferentialequations,
  title = 	 {{Probabilistic Solutions to Differential Equations and their Application to Riemannian Statistics}},
  author = 	 {Hennig, Philipp and Hauberg, Søren},
  booktitle = 	 {Proceedings of the Seventeenth International Conference on Artificial Intelligence and Statistics},
  pages = 	 {347--355},
  year = 	 {2014},
  editor = 	 {Kaski, Samuel and Corander, Jukka},
  volume = 	 {33},
  series = 	 {Proceedings of Machine Learning Research},
  address = 	 {Reykjavik, Iceland},
  month = 	 {22--25 Apr},
  publisher =    {PMLR},
  url = 	 {https://proceedings.mlr.press/v33/hennig14.html}
}

@InProceedings{arvanitidis2019fastrobustshortestpaths,
  title = 	 {{Fast and Robust Shortest Paths on Manifolds Learned from Data}},
  author =       {Arvanitidis, Georgios and Hauberg, Soren and Hennig, Philipp and Schober, Michael},
  booktitle = 	 {Proceedings of the Twenty-Second International Conference on Artificial Intelligence and Statistics},
  pages = 	 {1506--1515},
  year = 	 {2019},
  editor = 	 {Chaudhuri, Kamalika and Sugiyama, Masashi},
  volume = 	 {89},
  series = 	 {Proceedings of Machine Learning Research},
  month = 	 {16--18 Apr},
  publisher =    {PMLR},
  url = 	 {https://proceedings.mlr.press/v89/arvanitidis19a.html}
}

@article{leapfrog_optimal_control,
author = {Kaya, C. Yal\c{c}in and Noakes, J. Lyle},
title = {{Leapfrog for Optimal Control}},
journal = {SIAM Journal on Numerical Analysis},
volume = {46},
number = {6},
pages = {2795-2817},
year = {2008},
doi = {10.1137/060675034},
URL = { 
        https://doi.org/10.1137/060675034
},
eprint = { 
        https://doi.org/10.1137/060675034
}
}

@article{leapfrog_noakes, title={{A Global algorithm for geodesics}}, volume={65}, DOI={10.1017/S1446788700039380}, number={1}, journal={Journal of the Australian Mathematical Society. Series A. Pure Mathematics and Statistics}, author={Noakes, Lyle}, year={1998}, pages={37–50}}

@article{newton_riemann_mechanics,
AUTHOR = {Cariñena, José F. and Muñoz-Lecanda, Miguel-C.},
TITLE = {{Geodesic and Newtonian Vector Fields and Symmetries of Mechanical Systems}},
JOURNAL = {Symmetry},
VOLUME = {15},
YEAR = {2023},
NUMBER = {1},
ARTICLE-NUMBER = {181},
URL = {https://www.mdpi.com/2073-8994/15/1/181},
ISSN = {2073-8994},
DOI = {10.3390/sym15010181}
}

@book{do1992riemannian,
  title={{Riemannian Geometry}},
  author={do Carmo, M.P.},
  isbn={9783764334901},
  lccn={91037377},
  series={Mathematics (Boston, Mass.)},
  url={https://books.google.dk/books?id=uXJQQgAACAAJ},
  year={1992},
  publisher={Birkh{\"a}user}
}

@misc{kingma2017adammethodstochasticoptimization,
  author = {Kingma, Diederik P. and Ba, Jimmy},
  description = {Adam: A Method for Stochastic Optimization},
  note = {cite arxiv:1412.6980Comment: Published as a conference paper at the 3rd International Conference  for Learning Representations, San Diego, 2015},
  title = {{Adam: A Method for Stochastic Optimization}},
  url = {http://arxiv.org/abs/1412.6980},
  year = 2014
}

@inproceedings{odepack,
  address = {Amsterdam},
  author = {Hindmarsh, A. C.},
  booktitle = {Scientific Computing},
  editor = {Stepleman, R. S.},
  pages = {55--64},
  publisher = {North-Holland},
  title = {{{ODEPACK}, a Systematized Collection of {ODE} Solvers}},
  year = 1983
}

@article{bdf,
author = {Shampine, Lawrence F. and Reichelt, Mark W.},
title = {{The MATLAB ODE Suite}},
journal = {SIAM Journal on Scientific Computing},
volume = {18},
number = {1},
pages = {1-22},
year = {1997},
doi = {10.1137/S1064827594276424},

URL = { 
    
        https://doi.org/10.1137/S1064827594276424
    
    

},
eprint = { 
    
        https://doi.org/10.1137/S1064827594276424
    
    

}
}

@book{radau,
author = {Hairer, Ernst and Wanner, Gerhard},
year = {1996},
month = {01},
pages = {},
title = {{Solving Ordinary Differential Equations II. Stiff and Differential-Algebraic Problems}},
volume = {14},
journal = {Springer Verlag Series in Comput. Math.},
doi = {10.1007/978-3-662-09947-6}
}

@book{dop853,
author = {Hairer, Ernst and Norsett, Syvert and Wanner, Gerhard},
year = {1993},
month = {01},
pages = {},
title = {{Solving Ordinary Differential Equations I: Nonstiff Problems}},
volume = {8},
isbn = {978-3-540-56670-0},
doi = {10.1007/978-3-540-78862-1}
}

@article{rk32,
title = {{A 3(2) pair of Runge - Kutta formulas}},
journal = {Applied Mathematics Letters},
volume = {2},
number = {4},
pages = {321-325},
year = {1989},
issn = {0893-9659},
doi = {https://doi.org/10.1016/0893-9659(89)90079-7},
url = {https://www.sciencedirect.com/science/article/pii/0893965989900797},
author = {P. Bogacki and L.F. Shampine}
}

@article{rk45,
title = {{A family of embedded Runge-Kutta formulae}},
journal = {Journal of Computational and Applied Mathematics},
volume = {6},
number = {1},
pages = {19-26},
year = {1980},
issn = {0377-0427},
doi = {https://doi.org/10.1016/0771-050X(80)90013-3},
url = {https://www.sciencedirect.com/science/article/pii/0771050X80900133},
author = {J.R. Dormand and P.J. Prince}
}

@book{dggm,
    author = {Søren Hauberg},
    title = {{Differential geometry for generative modeling}},
    publisher = {},
    year = {2025},
}

@misc{georce,
      title={{GEORCE: A Fast New Control Algorithm for Computing Geodesics}}, 
      author={Frederik Möbius Rygaard and Søren Hauberg},
      year={2025},
      eprint={2505.05961},
      archivePrefix={arXiv},
      primaryClass={math.DG},
      url={https://arxiv.org/abs/2505.05961}, 
}

@article{frechet1948,
     author = {Fr\'echet, Maurice},
     title = {{Les \'el\'ements al\'eatoires de nature quelconque dans un espace distanci\'e}},
     journal = {Annales de l'institut Henri Poincar\'e},
     pages = {215--310},
     publisher = {INSTITUT HENRI POINCAR\'E ET GAUTHIER-VILLARS},
     volume = {10},
     number = {4},
     year = {1948},
     zbl = {0035.20802},
     language = {fr},
     url = {http://www.numdam.org/item/AIHP_1948__10_4_215_0/}
}

@misc{georce_frechet,
      title={{Simultaneous Optimization of Geodesics and Fr\'echet Means}}, 
      author={Frederik Möbius Rygaard and Steen Markvorsen and Søren Hauberg},
      year={2025},
      eprint={},
      archivePrefix={arXiv},
      primaryClass={},
      url={}, 
}

@inproceedings{ho2020denoisingdiffusionprobabilisticmodels,
 author = {Ho, Jonathan and Jain, Ajay and Abbeel, Pieter},
 booktitle = {Advances in Neural Information Processing Systems},
 editor = {H. Larochelle and M. Ranzato and R. Hadsell and M.F. Balcan and H. Lin},
 pages = {6840--6851},
 publisher = {Curran Associates, Inc.},
 title = {{Denoising Diffusion Probabilistic Models}},
 url = {https://proceedings.neurips.cc/paper_files/paper/2020/file/4c5bcfec8584af0d967f1ab10179ca4b-Paper.pdf},
 volume = {33},
 year = {2020}
}

@inproceedings{du2020implicitgenerationgeneralizationenergybased,
 author = {Du, Yilun and Mordatch, Igor},
 booktitle = {Advances in Neural Information Processing Systems},
 editor = {H. Wallach and H. Larochelle and A. Beygelzimer and F. d\textquotesingle Alch\'{e}-Buc and E. Fox and R. Garnett},
 pages = {},
 publisher = {Curran Associates, Inc.},
 title = {{Implicit Generation and Modeling with Energy Based Models}},
 url = {https://proceedings.neurips.cc/paper_files/paper/2019/file/378a063b8fdb1db941e34f4bde584c7d-Paper.pdf},
 volume = {32},
 year = {2019}
}

@inproceedings{goodfellow2014generativeadversarialnetworks,
 author = {Goodfellow, Ian J. and Pouget-Abadie, Jean and Mirza, Mehdi and Xu, Bing and Warde-Farley, David and Ozair, Sherjil and Courville, Aaron and Bengio, Yoshua},
 booktitle = {Advances in Neural Information Processing Systems},
 editor = {Z. Ghahramani and M. Welling and C. Cortes and N. Lawrence and K.Q. Weinberger},
 pages = {},
 publisher = {Curran Associates, Inc.},
 title = {{Generative Adversarial Nets}},
 url = {https://proceedings.neurips.cc/paper_files/paper/2014/file/f033ed80deb0234979a61f95710dbe25-Paper.pdf},
 volume = {27},
 year = {2014}
}

@inproceedings{kingma2022autoencodingvariationalbayes,
  author = {Kingma, Diederik P. and Welling, Max},
  booktitle = {2nd International Conference on Learning Representations, {ICLR} 2014, Banff, AB, Canada, April 14-16, 2014, Conference Track Proceedings},
  eprint = {http://arxiv.org/abs/1312.6114v10},
  eprintclass = {stat.ML},
  eprinttype = {arXiv},
  title = {{Auto-Encoding Variational Bayes}},
  year = 2014
}

@misc{wang2023interpolatingimagesdiffusionmodels,
      title={{Interpolating between Images with Diffusion Models}}, 
      author={Clinton J. Wang and Polina Golland},
      year={2023},
      eprint={2307.12560},
      archivePrefix={arXiv},
      primaryClass={cs.CV},
      url={https://arxiv.org/abs/2307.12560}, 
}

@inproceedings{cnf_flow, author = {Mathieu, Emile and Nickel, Maximilian}, title = {{Riemannian continuous normalizing flows}}, year = {2020}, isbn = {9781713829546}, publisher = {Curran Associates Inc.}, address = {Red Hook, NY, USA}, booktitle = {Proceedings of the 34th International Conference on Neural Information Processing Systems}, articleno = {211}, numpages = {13}, location = {Vancouver, BC, Canada}, series = {NIPS '20} }

@inproceedings{
wang2021geometrydeepgenerativeimage,
title={{A Geometric Analysis of Deep Generative Image Models and Its Applications}},
author={Binxu Wang and Carlos R Ponce},
booktitle={International Conference on Learning Representations},
year={2021},
url={https://openreview.net/forum?id=GH7QRzUDdXG}
}

@ARTICLE{protein_lovell,
  author={Lovell, Simon C and Davis, Ian W and Arendall, W Bryan 3rd and de Bakker, Paul I W and Word, J Michael and Prisant, Michael G and Richardson, Jane S and Richardson, David C},
  journal={Proteins}, 
  title={{Structure validation by Calpha geometry: phi,psi and Cbeta deviation.}}, 
  year={2003},
  volume={50},
  number={3},
  pages={437-450},
  doi={doi:10.1002/prot.10286}}

@ARTICLE{rna_murray,
  author={Murray, Laura J W and Arendall, W Bryan 3rd and Richardson, David C and Richardson, Jane S},
  journal={Proc Natl Acad Sci U S A.}, 
  title={{RNA backbone is rotameric.}}, 
  year={2003},
  volume={100},
  number={24},
  pages={13904-13909},
  doi={doi:10.1073/pnas.1835769100}}

@ARTICLE{karpatne2017machinelearninggeoscienceschallenges,
  author={Karpatne, Anuj and Ebert-Uphoff, Imme and Ravela, Sai and Babaie, Hassan Ali and Kumar, Vipin},
  journal={IEEE Transactions on Knowledge and Data Engineering}, 
  title={{Machine Learning for the Geosciences: Challenges and Opportunities}}, 
  year={2019},
  volume={31},
  number={8},
  pages={1544-1554},
  doi={10.1109/TKDE.2018.2861006}}

@inproceedings{
debortoli2022riemannianscorebasedgenerativemodelling,
title={{Riemannian Score-Based Generative Modelling}},
author={Valentin {de Bortoli} and Emile Mathieu and Michael John Hutchinson and James Thornton and Yee Whye Teh and Arnaud Doucet},
booktitle={Advances in Neural Information Processing Systems},
editor={Alice H. Oh and Alekh Agarwal and Danielle Belgrave and Kyunghyun Cho},
year={2022},
url={https://openreview.net/forum?id=oDRQGo8I7P}
}

@inproceedings{jo2024generativemodelingmanifoldsmixture,
author = {Jo, Jaehyeong and Hwang, Sung Ju},
title = {{Generative modeling on manifolds through mixture of Riemannian diffusion processes}},
year = {2024},
publisher = {JMLR.org},
booktitle = {Proceedings of the 41st International Conference on Machine Learning},
articleno = {899},
numpages = {23},
location = {Vienna, Austria},
series = {ICML'24}
}

@inproceedings{
huang2022riemanniandiffusionmodels,
title={{Riemannian Diffusion Models}},
author={Chin-Wei Huang and Milad Aghajohari and Joey Bose and Prakash Panangaden and Aaron Courville},
booktitle={Advances in Neural Information Processing Systems},
editor={Alice H. Oh and Alekh Agarwal and Danielle Belgrave and Kyunghyun Cho},
year={2022},
url={https://openreview.net/forum?id=ecevn9kPm4}
}

@InProceedings{song2019slicedscorematchingscalable,
  title = 	 {{Sliced Score Matching: A Scalable Approach to Density and Score Estimation}},
  author =       {Song, Yang and Garg, Sahaj and Shi, Jiaxin and Ermon, Stefano},
  booktitle = 	 {Proceedings of The 35th Uncertainty in Artificial Intelligence Conference},
  pages = 	 {574--584},
  year = 	 {2020},
  editor = 	 {Adams, Ryan P. and Gogate, Vibhav},
  volume = 	 {115},
  series = 	 {Proceedings of Machine Learning Research},
  month = 	 {22--25 Jul},
  publisher =    {PMLR},
  url = 	 {https://proceedings.mlr.press/v115/song20a.html}
}

@article{score_matching,
  author  = {Aapo Hyv{{\"a}}rinen},
  title   = {{Estimation of Non-Normalized Statistical Models by Score Matching}},
  journal = {Journal of Machine Learning Research},
  year    = {2005},
  volume  = {6},
  number  = {24},
  pages   = {695--709},
  url     = {http://jmlr.org/papers/v6/hyvarinen05a.html}
}

@ARTICLE{denoising_score_matching,
  author={Vincent, Pascal},
  journal={Neural Computation}, 
  title={{A Connection Between Score Matching and Denoising Autoencoders}}, 
  year={2011},
  volume={23},
  number={7},
  pages={1661-1674},
  doi={10.1162/NECO_a_00142}}

@inproceedings{chen2019neuralordinarydifferentialequations,
 author = {Chen, Ricky T. Q. and Rubanova, Yulia and Bettencourt, Jesse and Duvenaud, David K},
 booktitle = {Advances in Neural Information Processing Systems},
 editor = {S. Bengio and H. Wallach and H. Larochelle and K. Grauman and N. Cesa-Bianchi and R. Garnett},
 pages = {},
 publisher = {Curran Associates, Inc.},
 title = {{Neural Ordinary Differential Equations}},
 url = {https://proceedings.neurips.cc/paper_files/paper/2018/file/69386f6bb1dfed68692a24c8686939b9-Paper.pdf},
 volume = {31},
 year = {2018}
}

@article{ANDERSON1982313,
title = {{Reverse-time diffusion equation models}},
journal = {Stochastic Processes and their Applications},
volume = {12},
number = {3},
pages = {313-326},
year = {1982},
issn = {0304-4149},
doi = {https://doi.org/10.1016/0304-4149(82)90051-5},
url = {https://www.sciencedirect.com/science/article/pii/0304414982900515},
author = {Brian D.O. Anderson}
}

@inproceedings{
song2021scorebasedgenerativemodelingstochastic,
title={{Score-Based Generative Modeling through Stochastic Differential Equations}},
author={Yang Song and Jascha Sohl-Dickstein and Diederik P Kingma and Abhishek Kumar and Stefano Ermon and Ben Poole},
booktitle={International Conference on Learning Representations},
year={2021},
url={https://openreview.net/forum?id=PxTIG12RRHS}
}

@inproceedings{tosi2014metricsprobabilisticgeometries,
author = {Tosi, Alessandra and Hauberg, S\"{o}ren and Vellido, Alfredo and Lawrence, Neil D.},
title = {{Metrics for probabilistic geometries}},
year = {2014},
isbn = {9780974903910},
publisher = {AUAI Press},
address = {Arlington, Virginia, USA},
booktitle = {Proceedings of the Thirtieth Conference on Uncertainty in Artificial Intelligence},
pages = {800–808},
numpages = {9},
location = {Quebec City, Quebec, Canada},
series = {UAI'14}
}

@misc{hauberg2019bayeslearnmanifoldon,
      title={{Only Bayes should learn a manifold (on the estimation of differential geometric structure from data)}}, 
      author={Søren Hauberg},
      year={2019},
      eprint={1806.04994},
      archivePrefix={arXiv},
      primaryClass={stat.ML},
      url={https://arxiv.org/abs/1806.04994}, 
}

@inproceedings{
arvanitidis2018latentspaceodditycurvature,
title={{Latent Space Oddity: on the Curvature of Deep Generative Models}},
author={Georgios Arvanitidis and Lars Kai Hansen and Søren Hauberg},
booktitle={International Conference on Learning Representations},
year={2018},
url={https://openreview.net/forum?id=SJzRZ-WCZ},
}

@InProceedings{shao2017riemanniangeometrydeepgenerative,
author = {Shao, Hang and Kumar, Abhishek and Thomas Fletcher, P.},
title = {{The Riemannian Geometry of Deep Generative Models}},
booktitle = {Proceedings of the IEEE Conference on Computer Vision and Pattern Recognition (CVPR) Workshops},
month = {June},
year = {2018}
}

@inproceedings{song2020improvedtechniquestrainingscorebased,
 author = {Song, Yang and Ermon, Stefano},
 booktitle = {Advances in Neural Information Processing Systems},
 editor = {H. Larochelle and M. Ranzato and R. Hadsell and M.F. Balcan and H. Lin},
 pages = {12438--12448},
 publisher = {Curran Associates, Inc.},
 title = {{Improved Techniques for Training Score-Based Generative Models}},
 url = {https://proceedings.neurips.cc/paper_files/paper/2020/file/92c3b916311a5517d9290576e3ea37ad-Paper.pdf},
 volume = {33},
 year = {2020}
}

@inproceedings{
song2022denoisingdiffusionimplicitmodels,
title={{Denoising Diffusion Implicit Models}},
author={Jiaming Song and Chenlin Meng and Stefano Ermon},
booktitle={International Conference on Learning Representations},
year={2021},
url={https://openreview.net/forum?id=St1giarCHLP}
}

@inproceedings{
zheng2024noisediffusioncorrectingnoiseimage,
title={{NoiseDiffusion: Correcting Noise for Image  Interpolation  with Diffusion Models beyond Spherical Linear Interpolation}},
author={PengFei Zheng and Yonggang Zhang and Zhen Fang and Tongliang Liu and Defu Lian and Bo Han},
booktitle={The Twelfth International Conference on Learning Representations},
year={2024},
url={https://openreview.net/forum?id=6O3Q6AFUTu}
}
\bibliographystyle{tmlr}

\clearpage
\appendix
\section*{Appendix}
\section{Proofs and Derivations} \label{ap:proofs}
\subsection{Local representation of the metric} \label{ap:local_metric}
\begin{proposition}[Local Metric]
    Let $\tilde{S}$ denote the lower bound of $S$. Assume that $\lambda$ is sufficiently large such that $S(\cdot)$ is close to $\tilde{S}$ in Eq.~\ref{eq:constrained_geodesic}. Let $\bm{\gamma}^{*}$ denote the optimal solution to Eq.~\ref{eq:metric_definition}. Then the regularized energy can locally along the optimal curve $\bm{\gamma}^{*}$ be estimated as
    \begin{equation*}
        \mathcal{E}(\bm{\gamma}) \approx \tilde{S} + \int_{0}^{1}\left(\dot{\bm{\gamma}}^{*}\right)^{\top}(t)\left(\mG\left(\bm{\gamma}^{*}(t)\right)+\frac{\lambda}{2}\mH(S)\left(\bm{\gamma}^{*}(t)\right)\right)\dot{\bm{\gamma}}^{*}(t) \, \dif t,
    \end{equation*}
    where $\mH(S)$ denotes the Hessian of $S$. Thus, if
    \begin{equation*}
        \mG\left(\bm{\gamma}^{*}(t)\right)+\frac{\lambda}{2}\mH(S)\left(\bm{\gamma}^{*}(t)\right)
    \end{equation*}
    is positive definite, we can interpret it as a local representation of a Riemannian metric along the optimal curve.
\end{proposition}
\begin{proof}
    Let $\bm{\gamma}^{*}$ denote the optimal solution to Eq.~\ref{eq:metric_definition}.  A Taylor expansion of $S$ around a point $\bm{\gamma}^{*}(t)$ gives the approximation
    \begin{equation*}
        S\left(\bm{\gamma}^{*}(t+\Delta t)\right) = S\left(\bm{\gamma}^{*}(t)\right) + \left\langle \partial_{\vz}S\left(\bm{\gamma}^{*}(t)\right), \Delta \vz \right \rangle + \frac{1}{2}\Delta \vz_{i}^{\top} \mH(S)\left(\bm{\gamma}^{*}(t)\right) \Delta \vz + \mathcal{O}\left(\Delta \vz\right)\norm{\Delta \vz}_{2}^{2},
    \end{equation*}
    where $\Delta \vz_{i} = \bm{\gamma}^{*}(t+\Delta t)-\bm{\gamma}^{*}(t)$ and $\partial_{\vz}S\left(\bm{\gamma}^{*}(t)\right)$ and $\mH(S)\left(\bm{\gamma}^{*}(t)\right)$ denote the gradient and Hessian of $S(\cdot)$, respectively. By assumption, $S\left(\bm{\gamma}^{*}(t)\right)$ is close to the lower bound $\tilde{S}$, which implies that $\partial_{\vz}S\left(\bm{\gamma}^{*}(t)\right) \approx \mathbf{0}$. Similarly, we set $S\left(\bm{\gamma}^{*}(t)\right) \approx \tilde{S}$. The regularized energy in Eq.~\ref{eq:metric_definition} can therefore locally be written as
    \begin{equation*}
        \begin{split}
            \mathcal{E}(\bm{\gamma}) &\approx \tilde{S} + \int_{0}^{1}\left(\left(\dot{\bm{\gamma}}^{*}\right)^{\top}(t)\mG\left(\bm{\gamma}^{*}(t)\right)\dot{\bm{\gamma}}^{*}(t) + \left(\dot{\bm{\gamma}}^{*}\right)^{\top}(t)\frac{\lambda}{2}\mH(S)(\bm{\gamma}(t))\dot{\bm{\gamma}}^{*}(t)\right) \, \dif t \\
            &= \tilde{S} + \int_{0}^{1}\left(\dot{\bm{\gamma}}^{*}\right)^{\top}(t)\left(\mG\left(\bm{\gamma}^{*}(t)\right)+\frac{\lambda}{2}\mH(S)\left(\bm{\gamma}^{*}(t)\right)\right)\dot{\bm{\gamma}}^{*}(t) \, \dif t,
        \end{split}
    \end{equation*}
    which completes the proof.
\end{proof}
\subsection{ODE for minimizing curves} \label{ap:ode_metric}
\begin{proposition}[First variation]
    Consider the regularized energy in Eq.~\ref{eq:metric_definition} for a Riemannian manifold $\left(\mathcal{M}, g\right)$. Let $g_{ij}$ and $g^{ij}$ denote the local coordinates of metric matrix function and its inverse, respectively. The first variation gives the following \textsc{ode}
    \begin{equation} 
        \ddot{\bm{\gamma}}^{k}(t) + \bm{\gamma}_{ij}^{k}\dot{\bm{\gamma}}^{s}(t)\dot{\bm{\gamma}}^{j}(t) = \frac{\lambda}{2} g^{kp}\partial_{p} S\left(\bm{\gamma}(t)\right),
    \end{equation}
    where $\bm{\gamma}_{ij}^{k} = \frac{1}{2}g^{kj}\left(\partial_{l} g_{pj} + \partial_{j} g_{pl} - \partial_{p}g_{ij}\right)$ denotes the Christoffel symbols derived from the Levi-Civita connection.
\end{proposition}
\begin{proof}
    Consider the energy written in Einstein notation as
    \begin{equation*}
        \mathcal{E}\left(\bm{\gamma}\right) = \frac{1}{2}\int_{a}^{b} g_{ij}\left(\bm{\gamma}(t)\right)\dot{\bm{\gamma}}^{i}(t)\dot{\bm{\gamma}}^{j}(t) \, \dif t + \lambda \int_{a}^{b}S\left(\bm{\gamma}(t)\right) \, \dif t,
    \end{equation*}
    where $g_{ij}$ denotes the $(i,j)$th element of $\mG$. We implicitly assume that $g_{ij}$ depends on $\bm{\gamma}(t)$. We aim to derive the corresponding system of ordinary differential equations (ODE) governing the shortest curves using calculus of variation. The corresponding Lagrangian is given by
    \begin{equation*}
        L\left(\vx,\dot{\vx}\right) = g_{ij}\dot{\vx}^{i}\dot{\vx}^{j} + \lambda S,
    \end{equation*}
    where we to shorten notation set $\vx:=\bm{\gamma}(t)$ and $S:= S\left(\bm{\gamma}(t)\right)$ and implicitly assume the dependence on $t$. From calculus of variation we get the Euler-Lagrange equation
    \begin{equation*}
        \frac{\dif}{\dif t}\left(\frac{\partial L}{\partial \dot{\vx}^{p}}\right) - \frac{\partial L}{\partial \vx^{p}} = 0.
    \end{equation*}
    Since $\mG$ is symmetric, we see that
    \begin{equation*}
        \begin{split}
            \frac{\dif}{\dif t}\left(\frac{\partial L}{\partial \dot{\vx}^{p}}\right) &= \frac{\dif}{\dif t}\left(\frac{\partial}{\partial \dot{\vx}^{p}}\left(\frac{1}{2}g_{ij}\dot{\vx}^{i}\dot{\vx}^{j} + \lambda S\right)\right) \\
            &= \frac{\dif}{\dif t}\left(2g_{pj}\left(\vx\right)\dot{\vx}^{j}\right) \\
            &= 2\partial_{l} g_{pj}\left(\vx\right)\dot{\vx}^{l}\dot{\vx}^{j} + 2g_{pj}\ddot{\vx}^{j}.
        \end{split}
    \end{equation*}
    \begin{equation*}
        \begin{split}
            \frac{\partial L}{\partial \vx^{p}} &= \frac{\partial}{\partial \vx^{p}} \left(g_{ij}\dot{\vx}^{i}\dot{\vx}^{j} + \lambda S\right) \\
            &= \partial_{p}g_{ij}\dot{\vx}^{i}\dot{\vx}^{j} + \lambda \partial_{p} S.
        \end{split}
    \end{equation*}
    Combining this, we get that
    \begin{equation*}
        2\partial_{l} g_{pj}\dot{\vx}^{l}\dot{\vx}^{j} + 2g_{pj}\ddot{\vx}^{j} - \partial_{p}g_{ij}\dot{\vx}^{i}\dot{\vx}^{j} - \lambda \partial_{p} S = 0
    \end{equation*}
    Let $g^{pk}$ denote the elements of the inverse to $\mG$. Multiplying $g^{pk}$ on both sides, we get
    \begin{equation*}
        2g^{pk}\partial_{l} g_{pj}\dot{\vx}^{l}\dot{\vx}^{j} + 2\ddot{\vx}^{k} - g^{pk}\partial_{p}g_{ij}\left(\vx\right)\dot{\vx}^{i}\dot{\vx}^{j} = \lambda g^{pk}\partial_{p} S,
    \end{equation*}
    where identity the Christoffel symbols as $\bm{\gamma}_{ij}^{k} = \frac{1}{2}g^{kj}\left(\partial_{l} g_{pj} + \partial_{j} g_{pl} - \partial_{p}g_{ij}\right)$ such that
    \begin{equation*}
        2\ddot{\vx}^{k} + 2\bm{\gamma}_{ij}^{k}\dot{\vx}^{i}\dot{\vx}^{j} = \lambda g^{kp}\partial_{p} S,
    \end{equation*}
\end{proof}
\subsection{Necessary conditions for ProbGEORCE} \label{ap:pgeorce_cond}
The optimal control problem in Eq.~\ref{eq:control_problem} gives rise to the following necessary conditions for a minimum point.
\begin{proposition}
    The necessary conditions for a minimum in Eq.~\ref{eq:control_problem} are
    \begin{equation} 
        \begin{split}
            &2\mG(\vz_{s})\vu_{s}+\bm{\mu}_{s}=0, \quad s=0,\dots, N_{\mathrm{grid}}-1, \\
            &\vz_{s+1}=\vz_{s}+\vu_{s}, \quad s=0,\dots,N_{\mathrm{grid}}-1, \\
            &\restr{\nabla_{y}\left[\vu_{s}^{\top}\mG(y)\vu_{s} + \lambda S(y)\right]}{y=\vz_{s}}+\bm{\mu}_{s}=\bm{\mu}_{s-1}, \quad s=1,\dots,N_{\mathrm{grid}}-1. \\
            &\vz_{0}=\va, \vz_{N_{\mathrm{grid}}}=\vb,
        \end{split}
    \end{equation}
    where $\bm{\mu}_{s} \in \mathbb{R}^{d}$ for $s=0,\dots,N_{\mathrm{grid}}-1$.
\end{proposition}
\begin{proof}
    We prove the necessary conditions using the same approach as in \citep{georce} for the regularized energy function by exploiting Pontryagin's maximum principle. Define the Hamiltonian of the control problem in eq.~\ref{eq:control_problem} as
    \begin{equation*}
        H_{s}(\vz_{s},\vu_{s},\bm{\mu}_{s}) = \vu_{s}^{\top}\mG(\vz_{s})\vu_{s}+\lambda S(\vz_{s})+\bm{\mu}_{s}^{\top}(\vz_{s}+\vu_{s}),
    \end{equation*}
    which by the time discrete version of Pontryagin's maximum principle gives the following optimization problem
    \begin{equation} 
        \begin{split}
            \min_{\vu_{s}} \quad &\sum_{s=0}^{N_{\mathrm{grid}}-1} H_{s}(\vz_{s}, \vu_{s}, \bm{\mu}_{s}) \\
            \text{s.t.} \quad &\vz_{s+1} = \vz_{s} + \vu_{s}, \quad s=0,\dots,N_{\mathrm{grid}}-1 \\
            &\nabla_{\vz_{s}}H_{s}(\vz_{s},\vu_{s},\bm{\mu}_{s}) = \bm{\mu}_{s-1}, \quad s=0,\dots,N_{\mathrm{grid}}-1 \\
            &\vz_{0}=\va, \, \vz_{T}=\vb.
        \end{split}
    \end{equation}
    Since $\mG(\vz_{s})$ is positive definite, $H_{s}(\vz_{s}, \vu_{s}, \bm{\mu}_{s})$ is convex in $\vu_{s}$, and therefore the stationary point $\vu_{s}$ is also a global minimum point for $s=0,\dots,N_{\mathrm{grid}}-1$. This gives the following equations for the control problem
    \begin{equation*}
        \begin{split}
            &2\mG(\vz_{s})\vu_{s}+\bm{\mu}_{s}=0, \quad s=0,\dots,N_{\mathrm{grid}}-1, \\
            &\vz_{s+1} = \vz_{s}+\vu_{s}, \quad s=0,\dots,N_{\mathrm{grid}}-1, \\
            &\nabla_{\vz_{s}}\left(\vu_{s}^{\top}\mG_{s}\left(\vz_{s}\right)\vu_{s}+\lambda S(\vz_{s})\right)+\bm{\mu}_{s} = \bm{\mu}_{s-1}, \quad s=1,\dots,N_{\mathrm{grid}}-1, \\
            &\vz_{0}=\va, \, \vz_{N_{\mathrm{grid}}}=\vb.
        \end{split}
    \end{equation*}
\end{proof}
The necessary conditions in Eq.~\ref{eq:energy_opt_condtions} can generally not be solved with respect to $\vu_{0:(N_{\mathrm{grid}}-1)}$ and $\bm{\mu}_{0:(N_{\mathrm{grid}}-1)}$. However, we can circumvent this iteratively. In iteration $k$ consider the state and control variables $\vz_{0:N_{\mathrm{grid}}}^{(k)}$ and $\vu_{0:(N_{\mathrm{grid}}-1)}^{(k)}$. We fix $\mG(\cdot)$ and the gradient term in iteration $k$ to define the following variables.
\begin{equation} 
    \begin{split}
        &\bm{\nu}_{s} := \restr{\nabla_{y}\left(\vu_{s}^{\top}\mG(y)\vu_{s} + \lambda S(y)\right)}{y=\vz_{s}^{(k)},\vu_{s}=\vu_{s}^{(k)}}, \quad s=1,\dots,N_{\mathrm{grid}}-1,\\
        &\mG_{s} := \mG\left(\vz_{s}^{(k)}\right), \quad s=0,\dots,N_{\mathrm{grid}}-1.
    \end{split}
\end{equation}
With these variables fixed, the system of equations in Eq.~\ref{eq:energy_opt_condtions} reduces to
\begin{equation}
    \begin{split}
        &2\mG_{s}\vu_{s}+\bm{\mu}_{s} = 0, \quad s=0,\dots,N_{\mathrm{grid}}-1, \\
        &\bm{\nu}_{s}+\bm{\mu}_{s} = \bm{\mu}_{s-1}, \quad s=1,\dots,N_{\mathrm{grid}}-1, \\
        &\sum_{s=0}^{N_{\mathrm{grid}}-1}\vu_{s}=\vb-\va, \\
    \end{split}
\end{equation}
We see that this system is identical to the one in \citet{georce}, from which we immediately have the update scheme in Eq.~\ref{eq:energy_update_scheme} with the modified version of $\bm{\nu}$.
\subsection{Convergence for ProbGEORCE} \label{ap:pgeorce_convergence}
\change{We will assume the same regularity of the modified energy in eq.~\ref{eq:geodesic_lagrange_simplify} as in \citep{georce} with extra conditions on the regularizing function $S$. We state these below following the same outline as in \citep{georce}.}

\begin{assumption}[Local convergence] \label{assum:quad_conv_assumptions}
    \change{We assume the following regarding the discretized energy in Eq.~\ref{eq:control_problem} for the proof of local quadratic convergence.
    \begin{itemize}
        \item We assume that the discretized energy functional $E(\vz)$ is locally strictly convex in the (local) minimum point $\vx^{*}=\left(\vz^{*},\vu^{*}\right)$ in the sense that
        \begin{equation*}
            \exists \epsilon>0: \, \forall \vx \in B_{\epsilon}\left(\vx^{*}\right), \vx \neq \vx^{*}: \, \forall \alpha ]0,1[: \, E\left(\left(1-\alpha\right)\vx+ \alpha \vx^{*}\right) < \left(1-\alpha\right)E(\vx)+\alpha E\left(\vx^{*}\right),
        \end{equation*}
        where $B_{\epsilon}=\left\{\vx \, |\, \norm{\vx-\vx^{*}} < \epsilon\right\}$.
        \item Assume that the discretized energy functional $E(\vx)$ is locally Lipschitz, and consider the first order Taylor approximation of the discretized energy functional
        \begin{equation*}
            \Delta E = \langle \nabla E(\vx_{0}), \Delta \vx \rangle + \mathcal{O}\left(\Delta \vx\right)\norm{\Delta \vx},
        \end{equation*}
        \item We assume that the boundary value points are not conjugate points with respect to the Riemannian metric $g$. We assume a critical point of Eq.~\ref{eq:metric_definition} is a (local) minimum point.
    \end{itemize}
    }
\end{assumption}

\change{In this section, we generalize the proofs for convergence in \citep{georce} to include the regularization of the energy.}

\begin{proposition}[Global Convergence]
    \change{Under the assumptions in Assumptions~\ref{assum:quad_conv_assumptions}, then \textit{ProbGEORCE} in algorithm~\ref{al:prob_georce} has global convergence to a (local) minimum assuming $E(\vz)$ at the critical is locally strictly convex.}
\end{proposition}
\begin{proof}
    \change{\textit{ProbGEORCE} fulfills that
    \begin{equation} \label{eq:global_conv_cond}
        \begin{split}
            \nabla_{\vz_{s}}E(\vz,\vu)\left(\vz_{s}^{(k)}, \vu_{s}^{(k)}\right) &= \bm{\mu}_{s-1}-\bm{\mu}_{s}, \quad s=1,\dots,N_{\mathrm{grid}}-1, \\
            \nabla_{\vu_{s}}E(\vz,\vu)\left(\vz_{s}^{(k)}, \vu_{s}^{(k+1)}\right) &= -\bm{\mu}_{s}, \quad s=0,\dots,N_{\mathrm{grid}}-1,
        \end{split}
    \end{equation}
    where
    \begin{equation*}
        E(\vz,\vu) := \sum_{s=0}^{N_{\mathrm{grid}}-1}\left(\vu_{s}^{\top}\mG(\vz_{s})\vu_{s}+\lambda S(\vz_{s})\right),
    \end{equation*}
    and $\vz_{s}^{(k)}, \vu_{s}^{(k)}$ are the state and control variables in iteration $k$, respectively. These properties are identical to the properties for the \textit{GEORCE}-algorithm in \citep{georce}, and therefore the proof for global convergence for the \textit{GEORCE} algorithm also holds for \textit{ProbGEORCE}.}
\end{proof}

\change{In general, the local asymptotic quadratic convergence can not be extended to the regularized geodesic case unless $\lambda \mH_{S}$ is the zero matrix, where $\mH_{S}$ is the Hessian of $S$. This means in practice that the asymptotic local quadratic convergence of \textit{ProbGEORCE} will only hold approximately for low values of $\lambda$, or if the Hessian of $S$ is small. }

\change{In the following, we generalize the results in \citep{georce} to \textit{ProbGEORCE} to show that \textit{ProbGEORCE} has local asymptotic quadratic convergence when $\lambda \mH_{s}$ is close to the zero matrix.}
\begin{lemma} \label{lemma:reg_geodesics}
    \change{Let $z_{t}:=(\vx_{t},\vu_{t})$ and $\norm{\vu_{t}}^{2}_{\vx_{t}} = \vu_{t}^{\top}\mG(\vx_{t})\vu_{t} + \lambda S(\vx_{t})$, then the solution to the following optimal control problem is identical to a \textit{ProbGEORCE} step without line-search, i.e. $\alpha=1$.
    \begin{equation} \label{eq:georce_op_reg}
        \begin{split}
            \min_{\left(\vx_{t},\vu_{t}\right)} &\sum_{t=0}^{T-1}\norm{\vu_{t}^{(i)}}^{2}_{\vx_{t}^{(i)}} + \left(\nabla\norm{\vu_{t}^{(i)}}_{\vx_{t}^{(i)}}^{2}\right)^{\top}
            \begin{pmatrix}
                \vx_{t}-\vx_{t}^{(i)} \\
                \vu_{t}-\vu_{t}^{(i)}
            \end{pmatrix} \\
            \quad &+ \frac{1}{2}
            \begin{pmatrix}
                \vx_{t}-\vx_{t}^{(i)} & \vu_{t}-\vu_{t}^{(i)}
            \end{pmatrix}
            Q(\vx_{t}^{(i)},\vu_{t}^{(i)})
            \begin{pmatrix}
                \vx_{t}-\vx_{t}^{(i)} \\
                \vu_{t}-\vu_{t}^{(i)}
            \end{pmatrix} \\
            \text{s.t.} \quad &\vx_{t+1} = \vx_{t} + \vu_{t}, \quad t=0,\dots,T-1, \\
            &\vx_{0}=a, \, \vx_{T}=b, 
        \end{split}
    \end{equation}
    where
    \begin{equation*}
        Q(\vx_{t},\vu_{t}) = 
        \begin{pmatrix}
            \pmb{0} & \pmb{0} \\
            \pmb{0} & 2G(\vx_{t})
        \end{pmatrix}
    \end{equation*}}
\end{lemma}
\begin{proof}
    \change{The Hamiltonian function for Eq.~\ref{eq:georce_op_reg} is
    \begin{equation*}
        \begin{split}
            H_{t}(\vx_{t},\vu_{t},\mu_{t}) &= \norm{\vu_{t}^{(i)}}^{2}_{\vx_{t}^{(i)}} + \left(\nabla \norm{\vu_{t}^{(i)}}^{2}_{\vx_{t}^{(i)}}\right)^{\top}
            \begin{pmatrix}
                \vx_{t}-\vx_{t}^{(i)} \\
                \vu_{t}-\vu_{t}^{(i)}
            \end{pmatrix} \\
            &+ \frac{1}{2}
            \begin{pmatrix}
                \vx_{t}-\vx_{t}^{(i)} & \vu_{t}-\vu_{t}^{(i)}
            \end{pmatrix}
            Q(\vx_{t}^{(i)},\vu_{t}^{(i)})
            \begin{pmatrix}
                \vx_{t}-\vx_{t}^{(i)} \\
                \vu_{t}-\vu_{t}^{(i)}
            \end{pmatrix}
            + \mu_{t}^{\top}(\vx_{t}+\vu_{t})
        \end{split}
    \end{equation*}
    Applying Pontryagins' maximum principle to Eq.~\ref{eq:georce_op_reg} gives
    \begin{equation*}
        \begin{split}
            \min_{\vu_{t}} \quad &H_{t}(\vx_{t},\vu_{t},\mu_{t}) \\
            \text{s.t.} \quad &\vx_{t+1}=\vx_{t}+\vu_{t}, \quad t=0,\dots,T-1, \\
            &\nabla_{\vx_{t}}H_{t}(\vx_{t},\vu_{t},\mu_{t})=\mu_{t-1}, \quad t=1,\dots,T-1.
        \end{split}
    \end{equation*}
    As $H_{t}(\vx_{t},\vu_{t},\mu_{t})$ is strictly convex in $\vu_{t}$, then the optimal control $\vu_{t}^{*}$ is the stationary point to the Hamiltonian function, i.e.
    \begin{equation*}
        \begin{split}
            &\nabla_{\vu_{t}}H_{t}(\vx_{t},\vu_{t},\mu_{t}) = \nabla_{\vu_{t}}E(\vx_{t}^{(i)},\vu_{t}^{(i)}) + 2G(\vx_{t}^{(i)})(\vu_{t}-\vu_{t}^{(i)}) + \mu_{t} = 0 \\
            &\Leftrightarrow 2G(\vx_{t}^{(i)})u_{s}^{(i)} + \mu_{t} = 0, \quad t=0,\dots,T-1,
        \end{split}
    \end{equation*}
    since $\nabla_{\vu_{t}^{(i)}}\norm{\vu_{t}^{(i)}}^{2}_{2} = 2G(\vx_{t}^{(i)})\vu_{t}^{(i)}$. This is identical to the first equation in the equations for \textit{ProbGEORCE}. The state equation is unchanged, and the co-state equation for the problem becomes
    \begin{equation*}
        \nabla_{\vx_{t}}H_{t}(\vx_{t},\vu_{t},\mu_{t}) = \nabla_{\vx_{t}}\norm{\vu_{t}}^{2}_{\vx_{t}} + \mu_{t} = \mu_{t-1}, \quad t=1,\dots,T-1.
    \end{equation*}
    which is identical to the co-state equation of \textit{ProbGEORCE}.}
\end{proof}
\change{The optimality conditions can also be formulated as a linear system of equations similar to the original proof in \citep{georce}, but the result for local quadratic convergence in \textit{GEORCE} will generally not hold with regularization. This can be seen from the following error term between the Hessian for the regularized energy functional and $Q(\vx_{t},\vu_{t})$ in Lemma~\ref{lemma:reg_geodesics}
\begin{equation*}
    \nabla \norm{\vu_{t}^{*}}_{\vx_{t}^{*}}-Q(\vx_{t}^{*},\vu_{t}^{*}) = 
    \begin{pmatrix}
        \nabla^{2}_{\vx_{t},\vx_{t}}\norm{\vu_{t}^{*}}_{\vx_{t}^{*}} & \nabla^{2}_{\vx_{t},\vu_{t}}\norm{\vu_{t}^{*}}_{\vx_{t}^{*}} \\
        \nabla^{2}_{\vu_{t},\vx_{t}}\norm{\vu_{t}^{*}}_{\vx_{t}^{*}} & \pmb{0}
    \end{pmatrix}.
\end{equation*}
Consider the $\ell^{\infty}$ norm of the Hessian
\begin{equation*}
    \norm{\nabla^{2}_{\vx_{t},\vx_{t}}\norm{\vu_{t}^{*}}_{\vx_{t}^{*}}}_{\infty} = \norm{\nabla^{2}_{\vx_{t},\vx_{t}}\restr{\vu_{t}^{\top}\mG(\vx_{t})\vu_{t}+\lambda S(\vx_{t})}{(\vx_{t},\vu_{t})=(\vx_{t}^{*},\vu_{t}^{*})}}_{\infty}
\end{equation*}
The local asymptotic convergence now follows from the exactly same approach as in \citep{georce}, but only if the penalty term in the Hessian $\norm{\lambda \nabla^{2}_{\vx_{t},\vx_{t}}S(\vx_{t})}_{\infty}$ converges to $0$ when $T \rightarrow \infty$, since then $\norm{\nabla^{2}_{\vx_{t},\vx_{t}}S(\vx_{t})}_{\infty}$ will only approach zero in the case that $\norm{\lambda \nabla^{2}_{\vx_{t},\vx_{t}}S(\vx_{t})}_{\infty} = 0$. Thus, in practice the asymptotic local quadratic convergence will only be "approximately" observed if $S$ is approximately an affine function or if $\lambda$ is low.}

\change{Note that for a general function $S$, it cannot be ruled out that the solution to \textit{ProbGEORCE} has converged to a saddle point. However, the regularized energy will obtain a value lower or equal to the value of the starting point. If the regularized energy is locally strictly convex at the solution, \textit{ProbGEORCE} will converge to a local minimum point.}
\subsection{Adaptive update scheme for ProbGEORCE} \label{ap:adaptive_update}
\textit{ADAM} adaptively updates the step-size in gradient descent using higher-order variance of the gradient \citep{kingma2017adammethodstochasticoptimization}. This results in the following adaptive scheme in iteration $k$
\begin{equation*}
    \begin{split}
        \vg_{k} &\leftarrow \nabla_{\bm{\theta}} f_{k}\left(\bm{\theta}_{k-1}\right), \\
        \vm_{k} &\leftarrow \beta_{1} \vm_{k-1} + \left(1-\beta_{1}\right)\vg_{k}, \\
        \vv_{k} &\leftarrow \beta_{2}\vv_{k-1} + \left(1-\beta_{2}\right)\vg_{k}^{2}, \\
        \hat{\vm}_{k} &\leftarrow \frac{\vm_{k}}{1-\beta_{2}^{k}}, \\
        \hat{\vv}_{k} &\leftarrow \frac{\vv_{k}}{1-\beta_{2}^{k}}, \\
        \bm{\theta}_{k} &\leftarrow \bm{\theta}_{k-1} - \gamma \frac{\hat{\vm}_{k}}{\sqrt{\hat{\vv}_{k}}+\epsilon},
    \end{split}
\end{equation*}
where $f_{k}$ denotes the loss function in iteration $k$ and $\beta_{1},\beta_{2} > 0$ and $\gamma > 0$ are parameters, while $0 \leq \epsilon<<1$ is used for numerical stability. Now assume that we have a stochastic estimate of the energy in Eq.~\ref{eq:control_problem}
\begin{equation} \label{eq:partial_g}
    \tilde{\mG}(\vx) + \lambda \tilde{S}(\vx),
\end{equation}
such that we allow $\tilde{\mG}$ or $\tilde{S}$ to be stochastic samples of $\mG$ and $S$. For $\tilde{\bm{\nu}}_{s}^{(k)}\left(\vz_{s}^{(k)}\right) \leftarrow \restr{\nabla_{\vy}\left(\vu_{s}^{(k)}\tilde{\mG}\left(\vy\right)\vu_{s}^{(k)}+ \lambda S\left(\vy\right)\right)}{\vy=\vz_{s}^{(k)}}$ for $s=1,\dots,N_{\mathrm{grid}}-1$, we can interpret $\sum_{s=0}^{N_{\mathrm{grid}}-1}\norm{\tilde{\bm{\nu}}}^{2}$ as the variance of the estimator in \textit{ProbGEORCE}. By adaptively updating $\tilde{\bm{\nu}}$ and $\tilde{\mG}$, we can directly apply a similar update scheme as in \textit{ADAM} on the form:
\begin{equation*}
    \begin{split}
        \vz_{s}^{(k+1)} &\leftarrow \vz_{s}^{(k)}+\kappa \left(\vz_{s}^{(k+1)}-\vz_{s}^{(k)}\right), \\
        \vu_{s}^{(k+1)} &\leftarrow \vu_{s}^{(k)}+\kappa \left(\vu_{s}^{(k+1)}-\vu_{s}^{(k)}\right), \\
        \tilde{\mG}_{s}^{(k+1)} &\leftarrow (1-\beta_{1})\tilde{\mG}_{s}\left(\vz_{s}^{(k+1)}\right)+\beta_{1}\tilde{\mG}_{s}^{(k+1)}, \\
        \tilde{\bm{\nu}}_{s}^{(k+1)} &\leftarrow (1-\beta_{1})\tilde{\bm{\nu}}_{s}\left(\vz_{s}^{(k+1)}\right)+\beta_{1}\tilde{\bm{\nu}}_{s}^{(k+1)}, \\
        \tilde{\vg}^{(k+1)} &\leftarrow (1-\beta_{2})\left(\sum_{s=0}^{N_{\mathrm{grid}}-1}\norm{\tilde{\bm{\nu}}_{s}^{(k+1)}}^{2}\right)+\beta_{2}\tilde{\vg}^{(k+1)}, \\
        \hat{\mG}_{s}^{(k+1)} &= \frac{\tilde{\mG}_{s}^{(k+1)}}{1-\beta_{1}^{k+1}}, \\
        \hat{\bm{\nu}}_{s}^{(k+1)} &= \frac{\tilde{\bm{\nu}}_{s}^{(k+1)}}{1-\beta_{1}^{k+1}}, \\
        \hat{\vg}^{(k+1)} &= \frac{\tilde{\vg}^{(k+1)}}{1-\beta_{2}^{k+1}}, \\
        \kappa &= \min \left\{\frac{\gamma}{\sqrt{1+\tilde{\vg}^{(k+1)}}+\epsilon}, 1\right\}, \\ 
    \end{split}
\end{equation*}
where we cap the step size $\kappa$, since this can at most be one. We show the adaptive scheme in Algorithm~\ref{al:ada_prob_georce}, where $\mathrm{ProbGEORCE}$ denotes a step using Proposition~\ref{prop:update_scheme} or Corollary~\ref{cor:update_scheme_euclidean}. Note that if $\mG=I$, we do not have to update $I$ in the algorithm. Below we apply the update scheme of \textit{ADAM} due to its high-convergence in practice. 
\begin{algorithm}[!ht]
    \caption{Adaptive update scheme for ProbGEORCE}
    \label{al:ada_prob_georce}
    \begin{algorithmic}[1]
        \State \textbf{Input}: $\mathrm{tol}$, $N_{\mathrm{grid}}$, $\gamma$, $\beta_{1}$, $\beta_{2}$, $\epsilon$
        \State \textbf{Output}: Geodesic estimate $\vx_{0:N_{\mathrm{grid}}}$
        \State Set $\vz_{s}^{(0)} \leftarrow a+\frac{b-a}{N_{\mathrm{grid}}}s$ for $s=0.,\dots,N_{\mathrm{grid}}$, $\vu_{s}^{(0)} \leftarrow \frac{b-a}{N_{\mathrm{grid}}}$ for $s=0.,\dots,N_{\mathrm{grid}}-1$, $\kappa=\gamma$ and $k \leftarrow 0$
        \While{$\text{stop criteria} > \mathrm{tol}$}
        \State $\tilde{\mG}_{s}^{(k)} \leftarrow \mG_{s}\left(\vz_{s}^{(k)}\right)$ for $s=0,\dots,N_{\mathrm{grid}}-1$
        \State $\tilde{\bm{\nu}}_{s}^{(k)}\left(\vz_{s}^{(k)}\right) \leftarrow \restr{\nabla_{\vy}\left(\vu_{s}^{(k)}\tilde{\mG}\left(\vy\right)\vu_{s}^{(k)}+ \lambda S\left(\vy\right)\right)}{\vy=\vz_{s}^{(k)}}$ for $s=1,\dots,N_{\mathrm{grid}}-1$
        \If{$k=0$}
            \State 
            \begin{equation*}
                \begin{split}
                    &\left\{\vz_{s}^{(k+1)}\right\}_{s=0}^{N_{\mathrm{grid}}-1}, \left\{\vu_{s}^{(k+1)}\right\}_{s=0}^{N_{\mathrm{grid}}-1} \\
                    &= \mathrm{ProbGEORCE}\left(\left\{\tilde{\mG}_{s}^{(k)}\right\}_{s=0}^{N_{\mathrm{grid}}-1}, \left\{\left(\tilde{\mG}^{(k)}\right)^{-1}\right\}_{s=0}^{N_{\mathrm{grid}}-1}, \left\{\tilde{\bm{\nu}}_{s}^{(k)}\right\}_{s=0}^{N_{\mathrm{grid}}-1}, \left(\sum_{s=0}^{N_{\mathrm{grid}}-1}\left(\tilde{\mG}_{s}^{(k)}\right)^{-1}\right)^{-1}\right)
                \end{split}
            \end{equation*}
            
        \Else
            \State 
            \begin{equation*}
                \begin{split}
                    &\left\{\vz_{s}^{(k+1)}\right\}_{s=0}^{N_{\mathrm{grid}}-1}, \left\{\vu_{s}^{(k+1)}\right\}_{s=0}^{N_{\mathrm{grid}}-1} \\
                    &= \mathrm{ProbGEORCE}\left(\left\{\hat{\mG}_{s}^{(k)}\right\}_{s=0}^{N_{\mathrm{grid}}-1}, \left\{\left(\hat{\mG}^{(k)}\right)^{-1}\right\}_{s=0}^{N_{\mathrm{grid}}-1}, \left\{\hat{\vg}^{(k)}\right\}_{s=0}^{N_{\mathrm{grid}}-1}, \left(\sum_{s=0}^{N_{\mathrm{grid}}-1}\left(\hat{\mG}_{s}^{(k)}\right)^{-1}\right)^{-1}\right)
                \end{split}
            \end{equation*}
            
        \EndIf
        \State $\vz_{s}^{(k+1)} \leftarrow \vz_{s}^{(k)}+\kappa \left(\vz_{s}^{(k+1)}-\vz_{s}^{(k)}\right)$
        \State $\vu_{s}^{(k+1)} \leftarrow \vu_{s}^{(k)}+\kappa \left(\vu_{s}^{(k+1)}-\vu_{s}^{(k)}\right)$
        \State $\tilde{\mG}_{s}^{(k+1)} \leftarrow (1-\beta_{1})\tilde{\mG}_{s}\left(\vz_{s}^{(k+1)}\right)+\beta_{1}\tilde{\mG}_{s}^{(k+1)}$
        \State $\tilde{\bm{\nu}}_{s}^{(k+1)} \leftarrow (1-\beta_{1})\tilde{\bm{\nu}}_{s}\left(\vz_{s}^{(k+1)}\right)+\beta_{1}\tilde{\bm{\nu}}_{s}^{(k+1)}$
        \State $\tilde{\vg}^{(k+1)} \leftarrow (1-\beta_{2})\left(\sum_{s=0}^{N_{\mathrm{grid}}-1}\norm{\tilde{\bm{\nu}}_{s}^{(k+1)}}^{2}\right)+\beta_{2}\tilde{\vg}^{(k+1)}$
        \State $\hat{\mG}_{s}^{(k+1)} = \frac{\tilde{\mG}_{s}^{(k+1)}}{1-\beta_{1}^{k+1}}$, $\hat{\bm{\nu}}_{s}^{(k+1)} = \frac{\tilde{\bm{\nu}}_{s}^{(k+1)}}{1-\beta_{1}^{k+1}}$, $\hat{\vg}^{(k+1)} = \frac{\tilde{\vg}^{(k+1)}}{1-\beta_{2}^{k+1}}$
        \State $\kappa = \min \left\{\frac{\gamma}{\sqrt{1+\tilde{\vg}^{(k+1)}}+\epsilon}, 1\right\}$
        \State $k \leftarrow k+1$
        \EndWhile
        \State return $\vz_{s}$ for $s=0,\dots,N_{\mathrm{grid}}-1$
    \end{algorithmic}
\end{algorithm}
\subsection{Estimation of mean value} \label{ap:mean_value}
In this part, we show how to efficiently compute the Fr\'echet mean and minimize the connecting curves simultaneously by extending the results in \citet{georce_frechet} to take into account the regularization term. We will follow the same approach as in \citet{georce_frechet}, and consider the control formulation of Eq.~\ref{eq:frechet_energy}
\begin{equation} \label{eq:frechet_energy_control}
    \begin{split}
        \min_{(\vz_{s,i},\vu_{s,i})} E(\vx) &:= \min_{(\vz_{s,i},\vu_{s,i})}\left\{\sum_{i=1}^{N_{\mathrm{data}}}w_{i}\left(\sum_{s=0}^{N_{\mathrm{grid}}-1}\vu_{s,i}^{\top}\mG(\vz_{s,i})\vu_{s,i} + \lambda S(\vz_{s,i})\right)\right\} \\
        \vz_{s+1,i} &= \vz_{s,i}+\vu_{s,i}, \quad s=0,\dots,N_{\mathrm{grid}}-1, \, i=1,\dots,N_{\mathrm{data}}, \\
        \vz_{0,i}&=\va_{i},\vz_{N_{\mathrm{grid}},i}=\vy, \quad i=1,\dots,N_{\mathrm{data}}.
    \end{split}
\end{equation}
By a discrete-time version of Pontryagin's maximum principle, we arrive at the following
\begin{proposition}
    The necessary conditions for a minimum in Eq.~\ref{eq:frechet_energy_control} are
    \begin{equation} \label{eq:frechet_update_scheme}
        \begin{split}
            &2w_{i}\mG(\vz_{s,i})\vu_{s,i}+\bm{\mu}_{s,i}=0, \quad s=0,\dots, N_{\mathrm{grid}}-1, \, i=1,\dots,N_{\mathrm{data}}, \\
            &\vz_{s+1,i}=\vz_{s,i}+\vu_{s,i}, \quad s=0,\dots,N_{\mathrm{grid}}-1, \, i=1,\dots,N_{\mathrm{data}}, \\
            &\restr{\nabla_{\vz}\left[w_{i}\vu_{s,i}^{\top}\mG(\vz)\vu_{s,i} + \lambda S(\vz)\right]}{\vz=\vz_{s,i}}+\bm{\mu}_{s,i}=\bm{\mu}_{s-1,i}, \quad s=1,\dots,N_{\mathrm{grid}}-1, \, i=1,\dots,N_{\mathrm{grid}} ,\\
            &0 = \sum_{i=1}^{N}\bm{\mu}_{N_{\mathrm{grid}}-1,i}, \\
            &\vz_{0,i}=\va_{i}, \vz_{N_{\mathrm{grid}},i}=\vy, \quad i=1,\dots,N_{\mathrm{data}},
        \end{split}
    \end{equation}
    where $\bm{\mu}_{s,i} \in \mathbb{R}^{d}$ denotes the dual prices for the control problem in Eq.~\ref{eq:frechet_energy_control} for $s=0,\dots,N_{\mathrm{grid}}-1$ and $i=1,\dots,N_{\mathrm{data}}$.
\end{proposition}
\begin{proof}
    The Hamiltonian function of Eq.~\ref{eq:frechet_energy_control} is
    \begin{equation} 
        \begin{split}
            H_{s}\left(\vz_{s,i},\vu_{s,i},\bm{\mu}_{s,i}\right) &= \sum_{i=1}^{N_{\mathrm{grid}}-1}H_{s,i}\left(\vz_{s,i},\vu_{t,i},\bm{\mu}_{t,i}\right), \\
            H_{s,i}\left(\vz_{s,i},\vu_{s,i},\bm{\mu}_{s,i}\right) &= w_{i}\vu_{s,i}^{\top}\mG(\vz_{s,i})\vu_{s,i}+\bm{\mu}_{s,i}^{\top}\left(\vz_{s,i}+\vu_{s,i}\right).
        \end{split}
    \end{equation}
    Applying the time-discrete version of Pontryagin's maximum problem to Eq.~\ref{eq:frechet_update_scheme}, we get the following necessary conditions, where the endpoint $x_{N_{\mathrm{grid}},i}$ is replaced by the mean point $\vy$.
    \begin{align}
        \min_{\vu_{s,i}} \quad &\sum_{i=1}^{N_{\mathrm{data}}}\sum_{s=0}^{N_{\mathrm{grid}}-1} H_{s,i}(\vz_{s,i},\vu_{s,i},\bm{\mu}_{s,i}) \label{eq:energy_pontryagin_start}\\
        \text{s.t.} \quad &\vz_{s+1,i}=\vz_{s,i}+\vu_{s,i}, \quad s=0,\dots,N_{\mathrm{grid}}-1, \, i=1,\dots,N_{\mathrm{data}}, &&(\text{state equation}),  \\
        &\vy=\vz_{N_{\mathrm{grid}}-1,i}+\vu_{N_{\mathrm{grid}}-1,i}, \quad i=1,\dots,N_{\mathrm{data}}, &&(\text{state equation}),  \\
        &\nabla_{\vz_{s,i}}H_{s,i}(\vz_{s,i},\vu_{s,i},\bm{\mu}_{s,i}) =\bm{\mu}_{s-1,i}, \quad s=1,\dots,N_{\mathrm{grid}}-1, \, i=1,\dots,N_{\mathrm{data}} &&(\text{co-state equation}), \\
        &0 = \sum_{i=1}^{N_{\mathrm{data}}}\bm{\mu}_{N_{\mathrm{grid}}-1,i}, \quad i=1,\dots,N_{\mathrm{data}}, &&(\text{co-state equation}), \\
        &x_{0,i}=\va_{i}, \quad i=1,\dots,N_{\mathrm{data}}. \label{eq:energy_pontryagin_end}
    \end{align}
    We can decompose the minimization problem into the following sub-problems for each $i \in \{1,\dots,N_{\mathrm{data}}\}$
    \begin{equation*}
        \min_{\vu_{s,i}} H_{s,i}\left(\vz_{s,i},\vu_{s,i},\bm{\mu}_{s,i}\right), \quad s=0,\dots,N_{\mathrm{grid}}-1, \, i=1,\dots,N_{\mathrm{data}}.
    \end{equation*}
    Since $\mG(\vz_{s,i})$ is positive definite, $H_{t,i}\left(x_{t,i},\vu_{t,i},\bm{\mu}_{t,i}\right)$ is strictly convex in $\vu_{t,i}$. Thus, the stationary point in $\vu_{t,i}$ is also a global minimum minimum. We therefore get the following necessary conditions for a minimum.
    \begin{equation}
        \begin{split}
            &2w_{i}\mG(\vz_{s,i})\vu_{s,i}+\bm{\mu}_{s,i}=0, \quad s=0,\dots, N_{\mathrm{grid}}-1, \, i=1,\dots,N_{\mathrm{data}}, \\
            &\vz_{s+1,i}=\vz_{s,i}+\vu_{s,i}, \quad s=0,\dots,N_{\mathrm{grid}}-1, \, i=1,\dots,N_{\mathrm{data}}, \\
            &\restr{\nabla_{\vz}\left[w_{i}\vu_{s,i}^{\top}\mG(\vz)\vu_{s,i} + \lambda S(\vz)\right]}{\vz=\vz_{s,i}}+\bm{\mu}_{s,i}=\bm{\mu}_{s-1,i}, \quad s=1,\dots,N_{\mathrm{grid}}-1, \, i=1,\dots,N_{\mathrm{data}} ,\\
            &0 = \sum_{i=1}^{N}\bm{\mu}_{N_{\mathrm{grid}}-1,i}, \\
            &\vz_{0,i}=\va_{i}, \vz_{N_{\mathrm{grid}},i}=\vy, \quad i=1,\dots,N_{\mathrm{data}}.
        \end{split}
    \end{equation}
\end{proof}
We fix the following variables in iteration $k$ similar to Eq.~\ref{eq:fixed_variables} and \citet{georce_frechet}:
\begin{equation} 
    \begin{split}
        \bm{\nu}_{s,i} &:= \restr{\nabla_{\vz}\left(w_{i}\vu_{s,i}^{\top}\mG(\vz)\vu_{s,i} + \lambda S(\vz)\right)}{\vz=\vz_{s,i}^{(k)},\vu_{s,i}=\vu_{s,i}^{(k)}}, \quad s=1,\dots,N_{\mathrm{grid}}-1, \, i=1,\dots,N_{\mathrm{data}}, \\
        \mG_{s,i} &:= \mG\left(\vz_{s,i}^{(k)}\right), \quad s=0,\dots,N_{\mathrm{grid}}-1, \, i=1,\dots,N_{\mathrm{data}}. \\
    \end{split}
\end{equation}
In this way, the necessary conditions reduce to:
\begin{equation} \label{eq:frechet_simplified_eq}
    \begin{split}
        &2w_{i}\mG_{s,i}\vu_{s,i}+\bm{\mu}_{s,i}=0, \quad s=0,\dots, N_{\mathrm{grid}}-1, \, i=1,\dots,N_{\mathrm{data}}, \\
        &\vz_{s+1,i}=\vz_{s,i}+\vu_{s,i}, \quad s=0,\dots,N_{\mathrm{grid}}-1, \, i=1,\dots,N_{\mathrm{data}}, \\
        &\bm{\nu}_{s,i}+\bm{\mu}_{s,i}=\bm{\mu}_{s-1,i}, \quad s=1,\dots,N_{\mathrm{grid}}-1, \, i=1,\dots,N_{\mathrm{data}} ,\\
        &0 = \sum_{i=1}^{N}\bm{\mu}_{N_{\mathrm{grid}}-1,i}, \\
        &\vz_{0,i}=\va_{i}, \vz_{N_{\mathrm{grid}},i}=\vy, \quad i=1,\dots,N_{\mathrm{data}}.
    \end{split}
\end{equation}
Eq.~\ref{eq:frechet_simplified_eq} is completely similar to \citep{georce_frechet} with the modification that $\bm{\nu}_{s,i}$ depends on $S$, which has the closed-form solution.
\begin{proposition}
    The update scheme of $\vu_{s,i},\bm{\mu}_{s,i}$ and $\vz_{s,i}$ to minimize Eq.~\ref{eq:frechet_energy_control} is
    \begin{equation}
        \begin{split}
            &\vy = W^{-1}V, \\
            &\bm{\mu}_{N_{\mathrm{grid}}-1,i} = \left(\sum_{s=0}^{N_{\mathrm{grid}}-1}\mG_{s,i}^{-1}\right)^{-1}\left(2w_{i}(\va_{i}-\vy)-\sum_{s=0}^{N_{\mathrm{grid}}-1}\mG_{s,i}^{-1}\sum_{j>s}^{N_{\mathrm{grid}}-1}\bm{\nu}_{j,i}\right), \quad i=1,\dots,N_{\mathrm{data}}, \\
            &\vu_{s,i} = -\frac{1}{2w_{i}}\mG_{s,i}^{-1}\left(\bm{\mu}_{N_{\mathrm{data}}-1,i}+\sum_{j>s}^{N_{\mathrm{grid}}-1}\bm{\nu}_{j,i}\right), \quad s=0,\dots,N_{\mathrm{grid}}-1, \, i=1,\dots,N_{\mathrm{data}} \\
            &\vz_{s+1,i} = \vz_{s,i}+\vu_{s,i}, \quad s=0,\dots,N_{\mathrm{grid}}-1, \, i=1,\dots,N_{\mathrm{data}}, \\
            &\vz_{0,i}=\va_{i} \quad i=1,\dots,N_{\mathrm{data}},
        \end{split}
    \end{equation}
    where
    \begin{equation} \label{eq:w_v_def}
        \begin{split}
            W &= \sum_{i=1}^{N_{\mathrm{data}}}w_{i}\left(\sum_{s=0}^{N_{\mathrm{grid}}-1}\mG_{i,s}^{-1}\right)^{-1}, \\
            V &= \sum_{i=1}^{N_{\mathrm{data}}}w_{i}\left(\sum_{s=0}^{N_{\mathrm{grid}}-1}\mG_{s,i}^{-1}\right)^{-1}\va_{i}-\frac{1}{2}\sum_{i=1}^{N_{\mathrm{data}}}\left(\sum_{s=0}^{N_{\mathrm{grid}}-1}\mG_{s,i}^{-1}\right)^{-1}\sum_{s=0}^{N_{\mathrm{grid}}-1}\mG_{s,i}^{-1}\sum_{j>s}^{N_{\mathrm{grid}}-1}\bm{\nu}_{j,i}.
        \end{split}
    \end{equation}
\end{proposition}
If $\mG=I$ as in diffusion models, we get the following simplified update scheme.
\begin{corollary} \label{cor:euclidean_frechet_update_scheme}
    The update scheme of $\vu_{s,i},\bm{\mu}_{s,i}$ and $\vz_{s,i}$ to minimize Eq.~\ref{eq:frechet_energy} is
    \begin{equation} \label{eq:euclidean_frechet_update_scheme}
        \begin{split}
            &\vy = \frac{1}{\sum_{i=1}^{N_{\mathrm{data}}}w_{i}}\left(\sum_{i=1}^{N_{\mathrm{data}}}w_{i}\va_{i} - \frac{1}{2}\sum_{i=1}^{N_{\mathrm{data}}}\sum_{s=0}^{N_{\mathrm{grid}}-1}\sum_{j>s}\bm{\nu}_{j,i}\right), \\
            &\vu_{s,i} = \frac{\vy-\va_{i}}{N_{\mathrm{grid}}}+\frac{1}{2w_{i}}\left(\frac{1}{N_{\mathrm{grid}}}\sum_{k=0}^{N_{\mathrm{grid}}-1}\sum_{j>k}\bm{\nu}_{j,i}-\sum_{j>s}\bm{\nu}_{j,i}\right)\, \quad \quad s=0,\dots,N_{\mathrm{grid}}-1, \, i=1,\dots,N_{\mathrm{data}} \\
            &\vz_{s+1,i} = \vz_{s,i}+\vu_{s,i}, \quad s=0,\dots,N_{\mathrm{grid}}-1, \, i=1,\dots,N_{\mathrm{data}}, \\
            &\vz_{0,i}=\va_{i} \quad i=1,\dots,N_{\mathrm{data}},
        \end{split}
    \end{equation}
\end{corollary}
Note that by the same argument as in Appendix~\ref{ap:pgeorce_convergence}, the computation of the mean and minimizing curves will also have global convergence and local quadratic convergence. Similarly, to the adaptive update for \textit{ProbGEORCE} in Appendix~\ref{ap:adaptive_update}, we can update the solution for the mean computation by the adaptive scheme.
\begin{equation*}
    \begin{split}
        \vz_{s,i}^{(k+1)} &\leftarrow \vz_{s,i}^{(k)}+\kappa \left(\vz_{s,i}^{(k+1)}-\vz_{s,i}^{(k)}\right), \\
        \vu_{s,i}^{(k+1)} &\leftarrow \vu_{s,i}^{(k)}+\kappa \left(\vu_{s,i}^{(k+1)}-\vu_{s,i}^{(k)}\right), \\
        \tilde{\mG}_{s,i}^{(k+1)} &\leftarrow (1-\beta_{1})\tilde{\mG}_{s,i}\left(\vz_{s,i}^{(k+1)}\right)+\beta_{1}\tilde{\mG}_{s,i}^{(k+1)}, \\
        \tilde{\bm{\nu}}_{s}^{(k+1)} &\leftarrow (1-\beta_{1})\tilde{\bm{\nu}}_{s,i}\left(\vz_{s,i}^{(k+1)}\right)+\beta_{1}\tilde{\bm{\nu}}_{s,i}^{(k+1)}, \\
        \tilde{g}^{(k+1)} &\leftarrow (1-\beta_{2})\left(\sum_{i=1}^{N_{\mathrm{data}}}\sum_{s=0}^{N_{\mathrm{grid}}-1}\norm{\tilde{\bm{\nu}}_{s,i}^{(k+1)}}^{2}\right)+\beta_{2}\tilde{g}^{(k+1)}, \\
        \hat{\mG}_{s,i}^{(k+1)} &= \frac{\tilde{\mG}_{s,i}^{(k+1)}}{1-\beta_{1}^{k+1}}, \\
        \hat{\bm{\nu}}_{s,i}^{(k+1)} &= \frac{\tilde{\bm{\nu}}_{s}^{(k+1)}}{1-\beta_{1}^{k+1}}, \\
        \hat{g}^{(k+1)} &= \frac{\tilde{g}^{(k+1)}}{1-\beta_{2}^{k+1}}, \\
        \kappa &= \min \left\{\frac{\gamma}{\sqrt{1+\tilde{g}^{(k+1)}}+\epsilon}, 1\right\}, \\ 
    \end{split}
\end{equation*}
We show the update in pseudo-code in Algorithm~\ref{al:prob_georce_frechet}, where the line-search can be replaced by the update scheme above similar to Algorithm~\ref{al:ada_prob_georce}.
\begin{algorithm}[!ht]
    \caption{ProbGEORCE for means}
    \label{al:prob_georce_frechet}
    \begin{algorithmic}[1]
        \State \textbf{Input}: $\mathrm{tol}$, $\va_{1:N_{\mathrm{data}}}$, $N_{\mathrm{grid}}$
        \State \textbf{Output}: Geodesic estimate $\vz_{0:N_{\mathrm{grid}}}$
        \State Set $\vy^{(0)} \leftarrow \va_{0}$, $\vz_{s,i}^{(0)} \leftarrow \va_{i}+\frac{\vy^{(0)}-\va_{i}}{N_{\mathrm{grid}}}s$ for $s=0.,\dots,N_{\mathrm{grid}}$ and $\vu_{s,i}^{(0)} \leftarrow \frac{\vy^{(0)}-\va_{i}}{N_{\mathrm{grid}}}$ for $s=0.,\dots,N_{\mathrm{grid}}-1$ and $i=1,\dots,N_{\mathrm{data}}$.
        \While{$\frac{1}{N_{\mathrm{data}}}\norm{\restr{\nabla_{\vy}E(\vy)}{\vy=\vz_{s,i}^{(k)}}}_{2} > \mathrm{tol}$}
        \State $\mG_{s,i} \leftarrow \mG\left(\vz_{s,i}^{(k)}\right)$ for $s=0,\dots,N_{\mathrm{grid}}-1$ and $i=1,\dots,N_{\mathrm{data}}$.
        \State $\bm{\nu}_{s,i} \leftarrow \restr{\nabla_{\vz}\left(\vu_{s,i}^{(k)}\mG\left(\vz\right)\vu_{s,i}^{(k)} + \lambda S\left(\vz_{s,i}^{(k)}\right)\right)}{\vz=\vz_{s,i}^{(k)}}$ for $s=1,\dots,N_{\mathrm{grid}}-1$ and $i=1,\dots,N_{\mathrm{data}}$.
        \State $\vy \leftarrow W^{-1}V$ with $W,V$ given by eq.~\ref{eq:w_v_def}.
        \State $\bm{\mu}_{N_{\mathrm{grid}}-1,i} \leftarrow \left(\sum_{s=0}^{N_{\mathrm{grid}}-1}\mG_{s,i}^{-1}\right)^{-1}\left(2w_{i}(\va_{i}-\vy)-\sum_{s=0}^{N_{\mathrm{grid}}-1}\mG_{s,i}^{-1}\sum_{j>s}^{N_{\mathrm{grid}}-1}\bm{\nu}_{j,i}\right)$ for $i=1,\dots,N_{\mathrm{data}}$ and $s=1,\dots,N_{\mathrm{grid}}-1$.
        \State $\vu_{s,i} \leftarrow -\frac{1}{2w_{i}}\mG_{s,i}^{-1}\left(\bm{\mu}_{N_{\mathrm{grid}}-1,i}+\sum_{j>s}^{N_{\mathrm{grid}}-1}\bm{\nu}_{j,i}\right)$ for $t=0,\dots,N_{\mathrm{grid}}-1$ and  $i=1,\dots,N_{\mathrm{data}}$.
        \State $\vz_{s+1,i} \leftarrow \vz_{s,i}+\vu_{s,i}$ for $s=0,\dots,N_{\mathrm{grid}}-2$ and $i=1,\dots,N_{\mathrm{data}}$.
        \State Using line search find $\alpha^{*}$ for the following optimization problem for the discrete sum of energy $E$
        \begin{equation*}
            \begin{split}
                \min_{\alpha}\quad &E\left(\vz_{0:N_{\mathrm{grid}},1:N_{\mathrm{data}}}\right) \quad \text{(exact line search)} \\
                \text{s.t.} \quad &\vz_{s+1,i}=\vz_{s,i}+\alpha \tilde{\vu}_{s,i}+(1-\alpha)\vu_{s,i}^{(k)}, \quad s=0,\dots,N_{\mathrm{grid}}-1, \, i=1,\dots,N_{\mathrm{data}}. \\
                &\tilde{\vu}_{s,i} = \alpha \vu_{s,i}+(1-\alpha)\vu_{s,i}^{(k)}, \quad s=0,\dots,N_{\mathrm{grid}}-1, \, i=1,\dots,N_{\mathrm{data}}. \\
                &x_{0,i}=\va_{i}.
            \end{split}
        \end{equation*}
        \State Set $\vu_{s,i}^{(k+1)} \leftarrow \alpha^{*}\vu_{s,i}+(1-\alpha^{*})\vu_{s,i}^{(k)}$ for $s=0,\dots,N_{\mathrm{grid}}-1$ and $i=1,\dots,N_{\mathrm{data}}$.
        \State Set $\vz_{s+1,i}^{(k+1)} \leftarrow \vz_{s,i}^{(k+1)}+\vu_{s,i}^{(k+1)}$ for $s=0,\dots,N_{\mathrm{grid}}-1$ and $i=1,\dots,N_{\mathrm{data}}$.
        \EndWhile
        \State return $\vz_{s,i}$ for $s=0,\dots,N_{\mathrm{grid}}-1$ for $i=1,\dots,N_{\mathrm{data}}$.
    \end{algorithmic}
\end{algorithm}
\clearpage
\section{Algorithms} \label{ap:algorithms}
In Algorithm~\ref{al:prob_georce}, we state \textit{ProbGEORCE} in pseudo-code for solving the boundary value problem for curves on the Newtonian system on a Riemannian manifold.

\begin{algorithm}[hbt]
    \caption{ProbGEORCE: Probabilistic Geodesics}
    \label{al:prob_georce}
    \begin{algorithmic}[1]
        \State \textbf{Input}: $\mathrm{tol}$, $N_{\mathrm{grid}}$, $\rho$
        \State \textbf{Output}: Constrained geodesic estimate $\vz_{0:N_{\mathrm{grid}}}$
        \State Set $\vz_{s}^{(0)}\leftarrow \va+\frac{\vb-\va}{N_{\mathrm{grid}}}s$ for $s=0.,\dots,N_{\mathrm{grid}}$, $\vu_{s}^{(0)}\leftarrow \frac{\vb-\va}{N_{\mathrm{grid}}}$ for $s=0.,\dots,N_{\mathrm{grid}}-1$ and $k \leftarrow 0$
        \While{$\norm{\restr{\nabla_{\vy}E(\vy)}{\vy=\vz_{s}}^{(k)}}_{2} > \mathrm{tol}$}
        \State \change{$\mG_{s} \leftarrow \mG\left(\vz_{s}^{(k)}\right)$ for $s=0,\dots,N_{\mathrm{grid}}-1$}
        \State $\bm{\nu}_{s} \leftarrow \restr{\nabla_{\vy}\left(\vu_{s}^{(k)}\mG\left(\vy\right)\vu_{s}^{(k)} + S(\vy)\right)}{\vy=\vz_{s}^{(k)}}$ for $s=1,\dots,N_{\mathrm{grid}}-1$
        \State $\bm{\mu}_{N_{\mathrm{grid}}-1} \leftarrow \left(\sum_{s=0}^{N_{\mathrm{grid}}-1}\mG_{s}^{-1}\right)^{-1}\left(2(\va-\vb)-\sum_{s=0}^{N_{\mathrm{grid}}-1}\mG_{s}^{-1}\sum_{j>s}^{N_{\mathrm{grid}}-1}\bm{\nu}_{j}\right)$
        \State $\vu_{s} \leftarrow -\frac{1}{2}\mG_{s}^{-1}\left(\bm{\mu}_{N_{\mathrm{grid}}-1}+\sum_{j>s}^{N_{\mathrm{grid}}-1}\bm{\nu}_{j}\right)$ for $s=0,\dots,N_{\mathrm{grid}}-1$
        \State $\vz_{s+1} \leftarrow \vz_{s}+\vu_{s}$ for $s=0,\dots,N_{\mathrm{grid}}-1$
        \State $j \leftarrow 0$
        \While{$E\left(\vz_{0:N_{\mathrm{grid}}}\right) < E\left(\tilde{\vz}_{0:N_{\mathrm{grid}}}\right)$}
        \State $\tilde{\vz}_{s+1}=\tilde{\vz}_{s}+\rho^{j} \vu_{s}+(1-\rho^{j})\vu_{s}^{(k)}, \quad s=0,\dots,N_{\mathrm{grid}}-1, \quad \tilde{\vz}_{0}=\va$.
        \State $j \leftarrow j +1$
        \EndWhile
        \State Set $\vu_{s}^{(k+1)} \leftarrow \rho^{j-1}\vu_{s}+(1-\rho^{j-1})\vu_{s}^{(k)}$ for $s=0,\dots,N_{\mathrm{grid}}-1$
        \State Set $\vz_{s+1}^{(k+1)}\leftarrow \vz_{s}^{(k+1)}+\vu_{s}^{(k+1)}$ for $s=0,\dots,N_{\mathrm{grid}}-1$
        \State $k \leftarrow k+1$
        \EndWhile
        \State return $\vz_{s}$ for $s=0,\dots,N_{\mathrm{grid}}$.
    \end{algorithmic}
\end{algorithm}
\section{Existence and uniqueness of the Fr\'echet mean} \label{ap:frechet_mean_prop}
Let $(\mathcal{M},g)$ be a complete Riemannian manifold. Consider a probability space $\left(\Omega, \mathbb{B}, \mathbb{P}\right)$, where $\mathbb{B}$ denotes the Borel $\sigma$-algebra, then the Fr\'echet mean is defined as \citep{frechet1948, pennec2006statriemann}
\begin{equation} \label{eq:frechet_cont}
    \bm{\mu} = \argmin_{\vy \in \mathcal{M}} \int_{\mathcal{M}}\dist(\vz,\vy)^{2}p_{\vx}(\vz) \dif \mathcal{M}(\vz),
\end{equation}
where $\vx$ is a random variable on $\mathcal{M}$ with density $p_{\vx}$ and $\dif \mathcal{M}(\vz)$ is the Riemannian volume measure. The discrete version of Eq.~\ref{eq:frechet_cont} is \citep{pennec2006statriemann}
\begin{equation} \label{eq:frechet_disc}
    \bm{\mu} = \argmin_{\vy \in \mathcal{M}} \sum_{i=1}^{N_{\mathrm{data}}} \dist(\vz_{i},\vy)^{2}.
\end{equation}
For an in-depth analysis of the properties of Eq.~\ref{eq:frechet_disc} as an estimator, and the general properties of the Fr\'echet mean, we refer to \citep{Ziezold1977, bpc2003}. In general, the existence and uniqueness of the Fr\'echet mean is not guaranteed. Let $\mathcal{B}_{r}(p) = \left\{\vy \in \mathcal{M} \, | \, \dist(\vx,p) < r, \, \forall \vx \in \mathcal{M}\right\}$ be an open ball on $\mathcal{M}$. If there exists a unique minimizing geodesic from the center, $p \in \mathcal{M}$, to any other point in the open ball, then the open ball is said to be regular \citep{pennec2006statriemann}. If $\mathcal{M}$ is a complete Riemannian manifold with sectional curvature bounded by $\kappa$, and the support of the data distribution is within an open regular ball $\mathcal{B}_{r}(p)$ for a point $p \in \mathcal{M}$ with radius
\begin{equation*}
    r < \pi / (2\sqrt{\kappa}),
\end{equation*}
then the Fr\'echet mean exists and is unique \citep{Kendall1990ProbabilityCA}. Note that there also exist other results of existence and uniqueness that require a certain regularity of the variance \citep{karcher_1977}. 
\begin{figure}[h!]
    \vspace{-0.em}
    \centering
    \includegraphics[width=1.0\textwidth]{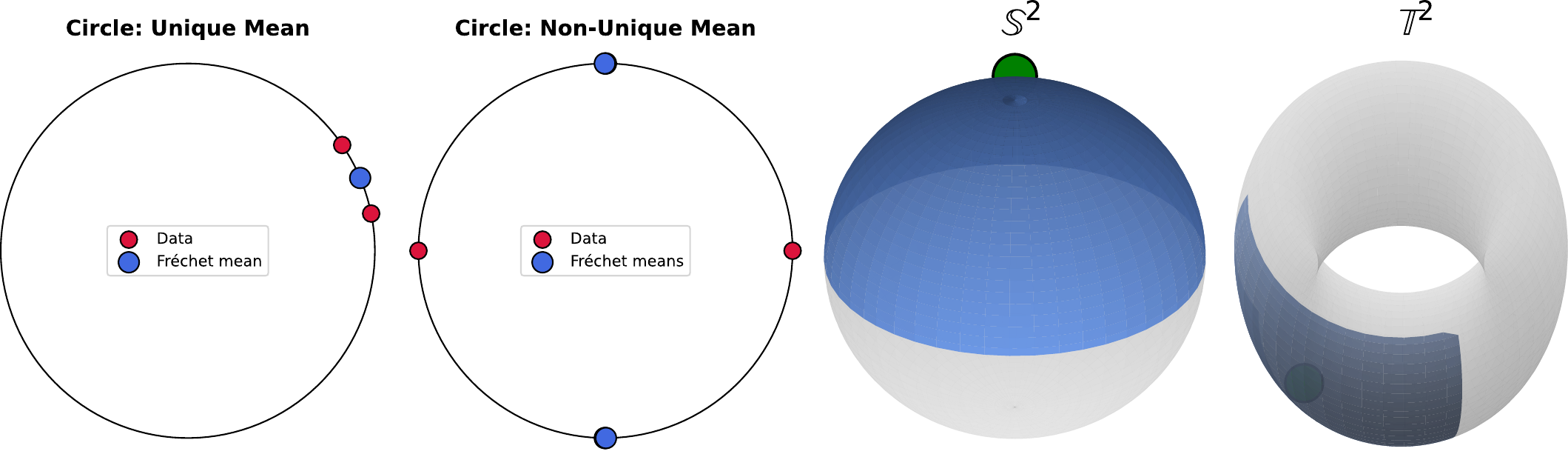}
    \caption{From left to right: The Fr\'echet mean (blue) for two data points (red) on the circle with uniqueness of the Fr\'echet mean, two points on the circle, where there are multiple Fr\'echet mean, the area (blue) for the data support around the north pole (green) of $\mathbb{S}^{2}$ for which the existence and uniqueness is guaranteed and similarly for the torus embedded in $\mathbb{R}^{3}$.}
    \label{fig:frechet_exist_and_uniq}
    \vspace{-0.em}
\end{figure}
To illustrate the existence and uniqueness, we consider the circle, unit-sphere and torus embedded in $\mathbb{R}^{3}$ in Fig.~\ref{fig:frechet_exist_and_uniq}. The two left-most figures illustrate that if the data points are sufficiently 'close', then the Fr\'echet mean is unique, while if, for example, there are antipodal data points on the circle, then there will be multiple Fr\'echet means. The two figures on the right illustrate an area in the unit sphere and torus embedded in $\mathbb{R}^{3}$, where the Fr\'echet mean is unique if the support of the data distribution is within these areas. For the unit sphere, the sectional curvature is constant and hence $\kappa=1$, and therefore there exists a unique Fr\'echet mean if the data support is within an open ball with $r < \pi / 2$. Thus, if the data distribution is within a hemisphere, then the Fr\'echet mean is unique, as illustrated in Fig.~\ref{fig:frechet_exist_and_uniq}. For the torus embedded in $\mathbb{R}^{3}$ parametrized by
\begin{equation*}
    \left(\left(R+r\cos \theta\right)\cos \phi, \left(R+r\cos\theta\right)\sin\phi, r\sin\theta\right), \quad \theta,\phi \in [0,2\pi),
\end{equation*}
with $R>r>0$, the sectional curvature is given by
\begin{equation*}
    \frac{\cos\theta}{r\left(R+r\cos\theta\right)}.
\end{equation*}
Thus, the curvature and the corresponding area with the uniqueness of the Fr\'echet mean guaranteed \citep{Kendall1990ProbabilityCA} depends on $\theta$.

In general, the conditions for uniqueness and existence of the Fr\'echet mean may not hold for the generalized version of the Fr\'echet mean in Eq.~\ref{eq:frechet_energy}. For the discretized version in Eq.~\ref{eq:frechet_energy_control}, define
\begin{equation*}
    \begin{split}
        f_{\mathrm{Frechet}} &= \sum_{i=1}^{N_{\mathrm{data}}}w_{i}\sum_{s=0}^{N_{\mathrm{grid}}-1}\left(\vx_{t+1,i}-\vx_{t,i}\right)^{\top}\mG(\vz_{s,i})\left(\vx_{t+1,i}-\vx_{t,i}\right), \\
        \lambda f_{S} &= \lambda \sum_{i=1}^{N_{\mathrm{data}}}w_{i}\sum_{s=0}^{N_{\mathrm{grid}}-1}\lambda S(\vz_{s,i}), \\
    \end{split}
\end{equation*}
Assume that the number of grid points $N_{\mathrm{grid}}$ is sufficiently high to approximate the continuous integrals. If all data points are within a certain set $A \subseteq \mathcal{M}$, such that the Fr\'echet mean in Eq.~\ref{eq:frechet_disc} is unique and exists, then it follows that if
\begin{equation*}
    \lambda\rho_{\mathrm{min}}\left(\mathrm{Hess}_{\vx}\left(f_{S}\right)\right) > - \rho_{\mathrm{min}}\left(\mathrm{Hess}_{\vx}f_{\mathrm{Frechet}}\right), \quad \forall \vx \in A
\end{equation*}
where $\rho_{\mathrm{min}}(\cdot)$ denotes the smallest eigenvalues of the Riemannian Hessian \citep{afsari}. We illustrate the negative loss landscape of the Fr\'echet mean in Fig.~\ref{fig:grid_ebm} for the energy-based model for checkerboard data used in Fig.~\ref{fig:multirow_plot_2d_gen}. In this case $\mG=I$, and therefore for $\lambda=0$ the Fr\'echet mean is unique and exists for any discrete data distribution. It can be seen that as $\lambda$ increases, two Fr\'echet means appear, and when $\lambda=100$ it seems that there are multiple means.
\begin{figure}[h!]
    \vspace{-0.em}
    \centering
    \includegraphics[width=1.0\textwidth]{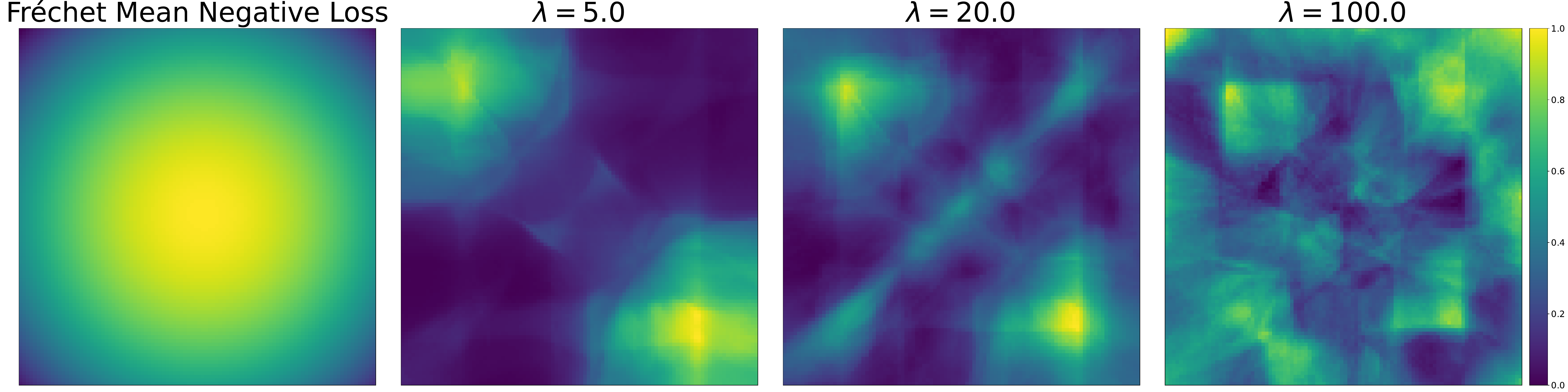}
    \caption{The negative loss of mean, where we for each grid point computes the energy from all other data points using Eq.~\ref{eq:metric_definition} for different values of $\lambda$. The right-most figure correspond to $\lambda=0$. Note that each loss is normalized for each plot to be between 0 and 1, and therefore the absolute scale between the plots can not be compared.}
    \label{fig:grid_ebm}
    \vspace{-0.em}
\end{figure}
\clearpage
\section{Application to diffusion models and image editing} \label{ap:app_diff_model}
Score-based diffusion models \citep{song2021scorebasedgenerativemodelingstochastic} approximate the data distribution by transforming the data using forward dynamics:
\begin{equation} \label{eq:forward_noise}
    \dif \vx_{t} = \mu(\vx_{t},t)\dif t + \sigma(t) \dif W_{t}, \quad \vx_{0} \sim p_{\mathrm{data}},
\end{equation}
such that $\mu: \mathbb{R}^{D} \times \mathbb{R}_{+} \rightarrow \mathbb{R}^{D}$ and $\sigma: \mathbb{R}_{+} \rightarrow \mathbb{R}^{D}$ are suitable functions for eq.~\ref{eq:forward_noise} to converge to a known limiting distribution $\pi$ for a sufficiently large time $T > 0$ \citep{song2021scorebasedgenerativemodelingstochastic}. Samples of the data distribution can then be generated using the time-reversal process $\vy_{t}:=\vx_{T-t}$ \citep{ANDERSON1982313}:
\begin{equation} \label{eq:score_based_treverse}
    dy_{t} = \left(\mu(\vy_{t},t)-\sigma(t)^{2}\nabla_{\vy} \log p_{t}(\vy_{t})\right)\dif t + \sigma(t) \dif \overline{W}_{t},
\end{equation}
where $\vy_{0} \sim \pi$, and $\nabla_{\vy} \log p_{t}(\cdot)$ denotes the \emph{score}, which can be learned using score matching \citep{score_matching, denoising_score_matching, song2019slicedscorematchingscalable}. Samples can also be generated deterministically using the probability flow \textsc{ode} \citep{chen2019neuralordinarydifferentialequations}
\begin{equation} \label{eq:score_based_ode}
    \dif \vy_{t} = \left(\mu(\vy_{t},t)-\frac{1}{2}\sigma(t)^{2}\nabla_{\vy_{t}}\log p_{t}(\vy_{t})\right)\dif t,
\end{equation}
Score-based diffusion models can be seen as the continuous version of denoising diffusion probabilistic models \textsc{ddpm} \citep{ho2020denoisingdiffusionprobabilisticmodels, song2021scorebasedgenerativemodelingstochastic} and have also been extended to data distributions on Riemannian manifolds \citep{huang2022riemanniandiffusionmodels, debortoli2022riemannianscorebasedgenerativemodelling, jo2024generativemodelingmanifoldsmixture}. We illustrate in Fig.~\ref{fig:ddpm_2d_data_noise_sampling} the difference between computing curves in noise space and data space as described in Section~\ref{sec:likely_geometry}.
\begin{figure}[h!]
    \vspace{-0.em}
    \centering
    \includegraphics[width=1.0\textwidth]{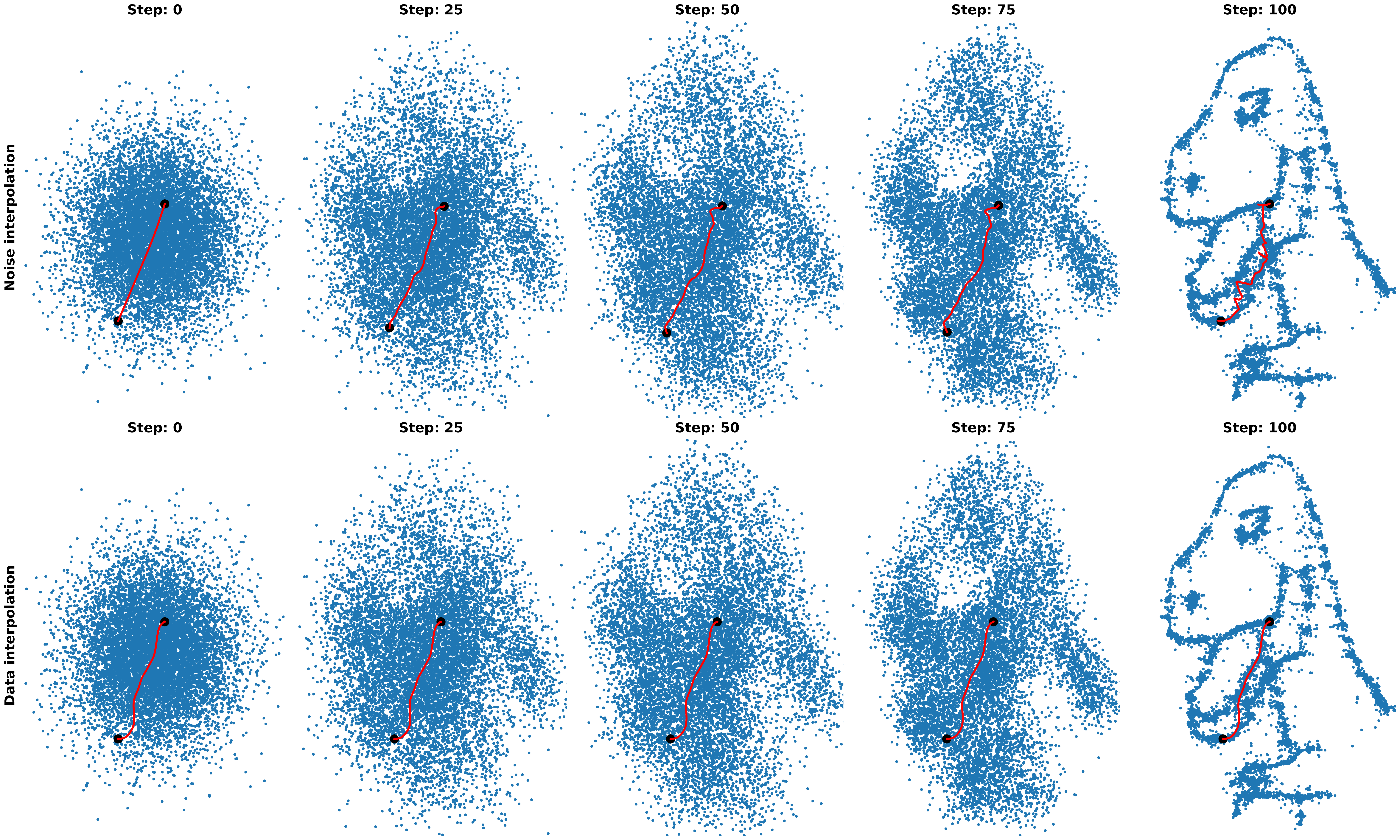}
    \caption{The first row shows interpolation in noise space (red) using the isotropic Gaussian limiting distribution and transports the curve to data space using a \textsc{ddpm}-model for a dino-dataset. The second row computes the curve directly in data space using the score of the \textsc{ddpm}, which, for illustrative purposes, is repeated in each step of the sampling process.}
    \label{fig:ddpm_2d_data_noise_sampling}
    \vspace{-0.em}
\end{figure}
\paragraph{Connection to image editing} Our method regularizes the geodesic interpolation and for $\lambda = 0$ our method will compute the connecting geodesic. In the Euclidean case, geodesic interpolation corresponds to linear interpolation, whereas spherical interpolation corresponds to a geodesic on a sphere. NoiseDiffusion \citep{zheng2024noisediffusioncorrectingnoiseimage} is designed for interpolation in diffusion models with image editing. Let $f$ denote the forward diffusion process that takes data to noise in eq.~\ref{eq:forward_noise} and $f^{-1}$ denote the reverse process that generates data samples using noise in eq.~\ref{eq:score_based_ode}, then NoiseDiffusion \citep{zheng2024noisediffusioncorrectingnoiseimage} proposes the following interpolation between two images $\vx,\vy$.
\begin{equation} \label{eq:noisediffusion}
    \begin{split}
        a &= \mathrm{clip}\left(f\left(\vx,t\right)\right), \\
        b &= \mathrm{clip}\left(f\left(\vy,t\right)\right), \\
        \vz_{s} &= \alpha(s) a + \beta(s) b + \left(\mu(s)-\alpha(s)\right)\vx + \left(\nu(s) - \beta(s)\right) \vy + \gamma(s) \epsilon, \\
        \vx_{s} &=  f^{-1}\left(\mathrm{clip}\left(\vz_{s},t\right)\right), \\
    \end{split}
\end{equation}
where $\mathrm{clip}$ denotes an element-wise operation that restricts a value to be between $\left[-\mathrm{boundary}, \mathrm{boundary}\right]$, while $\alpha(s), \beta(s)$ and $\gamma(s)$ are functions depending on $s$ with $\alpha^{2}(s)+\beta^{2}(s)+\gamma(s)^{2} = 1$ and $\epsilon \sim \mathcal{N_{\mathrm{grid}}}\left(0,I\right)$. The functions $\mu(s)$ and $\nu(s)$ serve as compensation for lost information. Note that
\begin{equation*}
    \vz_{s} = \alpha(s) a + \beta(s) b - \alpha(s)\vx - \beta(s) \vy + \mu(s) \vx + \nu(s) \vy + \gamma(s) \epsilon.
\end{equation*}
Inspired by eq.~\ref{eq:noisediffusion} we can easily extend our method to compute similar interpolants in noise space by
\begin{equation} \label{eq:prob_georce_noisediffusion}
    \vz_{s} = \mathrm{ProbGEORCE}_{1}(a, b) - \mathrm{ProbGEORCE}_{2}(\vx, \vy) + \mu(s) \vx + \nu(s) \vy + \gamma(s) \epsilon,
\end{equation}
where $\{\mathrm{ProbGEORCE}\}_{s=1}^{2}$ denotes the interpolations using Algorithm~\ref{al:prob_georce} using any metric or regularization function. Thus, we can obtain near-identical results for NoiseDiffusion with any original images to modify the interpolation curve as illustrated in Fig.~\ref{fig:image_edit_eagle}. Note that NoiseDiffusion assumes that the limiting distribution of the diffusion model is isotropic Gaussian, which our modification in eq.~\ref{eq:prob_georce_noisediffusion} does not assume. Note also that NoiseDiffusion \citep{zheng2024noisediffusioncorrectingnoiseimage} only defines interpolation and cannot be used to compute mean values or other statistics for diffusion models.
\begin{figure}[h!]
    \vspace{0.0em}
    \centering
    \includegraphics[width=1.0\textwidth]{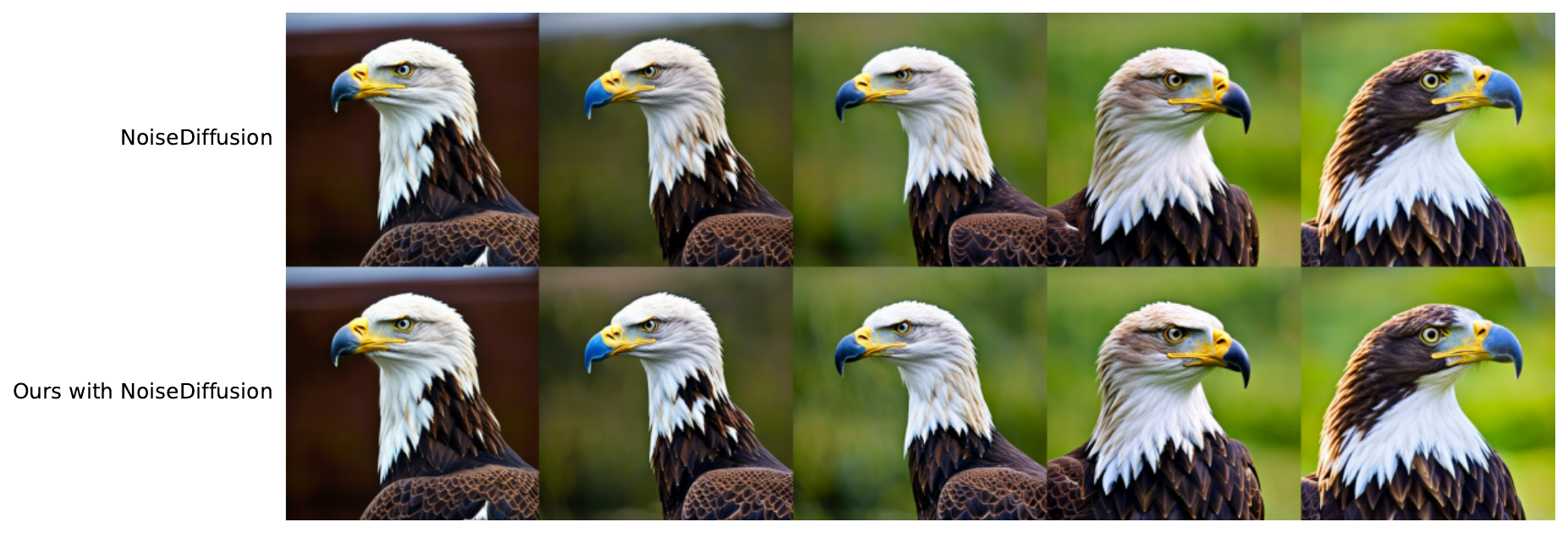}
    \caption{Computed interpolations for ControlNet \citep{zhang2023addingconditionalcontroltexttoimage} for eagle images of size $768 \times 768 \times 3$ similar to the experiment by \citet{zheng2024noisediffusioncorrectingnoiseimage}. We set the regularizer function in Eq.~\ref{eq:metric_definition} as the density of the $\chi^{2}$-distribution on the squared norm of the grid points $\vz_{0:N}$. We plot the using our method incorporated into \textit{NoiseDiffusion} using Eq.~\ref{eq:prob_georce_noisediffusion} and compare to NoiseDiffusion \citep{zheng2024noisediffusioncorrectingnoiseimage}.}
    \label{fig:image_edit_eagle}
    \vspace{0.0em}
\end{figure}
\clearpage
\section{Experimental details} \label{ap:exp_details}
The following contains experimental details and hyper-parameters. The code for reproducing the results can be found at \url{https://github.com/FrederikMR/likely_geometry}.

\paragraph{Normalizing $\lambda$} Since the energy and regularization function $S$ in Eq.~\ref{eq:metric_definition} might have very different scales, we will re-normalize $\lambda$ when computing interpolation and the mean value. We do it in the following sense: Let $E^{(0)}$ denote the energy $\int_{0}^{1}\dot{\gamma}(t)G(\gamma(t))\dot{\gamma}(t)\,\dif t$, where $\gamma$ is a simple straight line connecting any given start and end point. Let $S^{(0)}$ denote the corresponding value of $\int_{0}^{1}S(\gamma(t))\,\dif t$. We propose to use a normalized version, $\tilde{\lambda}$, for all computations
\begin{equation*}
    \tilde{\lambda} := \lambda\frac{E^{(0)}}{S^{(0)}}.
\end{equation*}
This is to ensure that the energy and regularization terms are approximately on the same scale before performing any computations. 

\paragraph{Benchmarks} In the following table, we describe the metrics and hyper-parameters used in Table~\ref{tab:gen2d_runtimes}. We consider the density $p$ as described in Section~\ref{sec:experiments}. For the methods that apply $p$ directly, we compute it as $\exp \log p$, although this may give a non-normalized density.

\begin{table}[ht]
\centering
\caption{Benchmark methods and hyper-parameters}
\label{tab:methods}
\begin{tabular}{l l l}
\hline
\textbf{Method} & \textbf{Metric} & \textbf{Hyper-parameters} \\
\hline
Linear & $G(\vz)=I$ & - \\
Spherical & Riemannian geometry of a sphere embedded in Euclidean space & - \\
Fisher-Rao & $G(\vz)=\nabla_{\vz}\log p (\vz) \nabla_{\vz} \log p(\vz)^{\top}$ & - \\
Jacobian & $G(\vz)=J_{\nabla_{\vz} \log p}^{\top}(\vz)J_{\nabla_{x} \log p}(\vz)$ & - \\
Inverse Density & $G(\vz)=\frac{1}{p(\vz)^{2}}I$ & - \\
Generative & $G(\vz)=\left(\frac{p(\vz)+\lambda}{p_{0}+\lambda}\right)^{2}I$ & $p_{0}=\lambda=1$ \\
Monge & $G(\vz)=I+\alpha^{2}\nabla_{\vz}\log p (\vz) \nabla_{\vz} \log p(\vz)^{\top}$ & $\alpha=1$ \\
\hline
\end{tabular}
\label{tab:becnhmarks}
\end{table}

When we write "Reg", we consider the metric described in Table~\ref{tab:becnhmarks} added with the identity matrix $\alpha I$, i.e., $\tilde{G}(x)=G(x)+\alpha I$ with $\alpha=1$. Note that we did not add this to the Fisher-Rao metric as this would correspond to the Monge metric for $\alpha=1$. For all methods, except linear and spherical, we apply \textit{ProbGEORCE} with adaptive update of the step-size in Algorithm~\ref{al:ada_prob_georce}. For Linear, we apply the closed-form expression for Euclidean geometry of the initial and boundary value problem as well as the closed-form expression of the mean. For spherical, we apply the closed-form expression for the initial value and boundary value problem. Note that \textsc{ivp} are started from the same point as the first point for the \textsc{bvp}. For the mean, we use the Logarithmic map and compute the mean using gradient descent \citep{pennec2006statriemann}. For the energy-based model on AFHQ, we give a brief description of the benchmarks in \citep{bethune2025energy} in Table~\ref{tab:becnhmarks_bethune}. Note that the constants $\alpha$ and $\beta$ are scaled to target high-density areas. We refer to \citep{bethune2025energy} for the approach. Note also that $1/p_{\theta}$ refers to using the inverse density, while $\log p_{\theta}$ refers to using the log-likelihood learned the energy based model. The $\mG_{\mathrm{LAND}}$ is described in \citep{bethune2025energy}.
\begin{table}[ht]
\centering
\caption{Kinetic energy for different metric types.}
\label{tab:becnhmarks_bethune}
\begin{tabular}{l l}
\hline
\textbf{Metric type} & \textbf{Kinetic Energy} \\
\hline
Conformal & \( E_{\text{kin}}(x, \dot{x}) = \left( \lambda + h(x) \right) \|\dot{x}\|^2 \), where $h$ is a scalar function. \\
Diagonal & \( E_{\text{kin}}(x, \dot{x}) = \left( \lambda + h(x) \right) \|\dot{x}\|^2 \), where $h$ is not a scalar function. \\
Full & \( E_{\text{kin}}(x, \dot{x}) = \lambda \|\dot{x}\|^2 + \left( \langle \dot{x}, \nabla h(x) \rangle \right)^2 \) \\
\hline
\end{tabular}
\end{table}

\paragraph{Optimizers} For the results in Table~\ref{tab:becnhmarks}, Fig.~\ref{fig:ddpm_2d_data_noise_sampling} and Fig.~\ref{fig:ddpm_2d_fmean_ode}, we apply \textit{ProbGEORCE} with adaptive update of the step-size in Algorithm~\ref{al:ada_prob_georce} with $\beta_{1}=\beta_{2}=0.5$, $\epsilon=10^{-8}$ and $\gamma=0.01$ with a maximum of $1000$ iterations and a tolerance of $10^{-4}$ to our method and all Riemannian methods with $\lambda=0$. For ControlNet, we use \textit{ProbGEORCE} with adaptive update of the step-size in Algorithm~\ref{al:ada_prob_georce} with $\beta_{1}=\beta_{2}=0.5$, $\epsilon=10^{-8}$ and $\gamma=0.001$ in data space, and in noise space we use \textit{ProbGEORCE} with line-search and $\rho=0.5$. For the runtime results for other methods, we set the learning rate to $0.01$.

\paragraph{FID and KID} We compute the FID and KID scores for the AFHQ dataset based on $10$ boundary value curves with $N_{\mathrm{grid}}=10$ grid points for ControlNet. We set the initial prompt in ControlNet to "A photo of a cat". For the energy-based model in \citep{bethune2025energy}, we use 6 boundary points.

\paragraph{Manifolds} For the manifolds used for the runtime results in Table~\ref{tab:runtime_manifolds}, we apply the same local chart and metric as in \cite{georce}, where we add a three isotropic Gaussian with random means in the local charts. For the conceptual figure in Fig.~\ref{fig:path_illustration}, we consider a chart on the form
\begin{equation*}
    \{\left(u,v,\exp\left(-\left(u^{2}+v^{2}\right)\right)+0.3\sin\left(3u\right)\cos\left(3v\right)\right) \, | \, \left(u,v\right) \in \mathbb{R}^{2}\}.
\end{equation*}
where the regularization corresponds to
\begin{equation*}
    S(u,v) = -\exp\left(-\frac{||u-v||^{2}}{0.15}\right).
\end{equation*}

\paragraph{Architecture} Table~\ref{tab:models_architecture} contains a list of all trained models and networks used in the paper. All other models are pre-trained.

\begin{table}[h!]
\centering
\caption{Summary of models, their architectures, and training details.}
\begin{tabular}{lp{0.5\textwidth}p{0.2\textwidth}}
\hline
\textbf{Model} & \textbf{Architecture} & \textbf{Training} \\
\hline
Energy-Based model & We use a MLP with one hidden layer with 128 neurons and ReLU activations using the package \citet{torchebm_library_2025}. & Trained for $1,000$ epochs with a batch size of $128$ for fixed $50,000$ samples. \\
\hline
Normalizing flow & We use 32 affine coupling blocks of MLPs with two hidden layers of 64 neurons using the package \citep{Stimper2023}. & Trained for $4,000$ epochs with a batch size of $2^{9}$. \\
\hline
AR & We use a MLP with two hidden layers of 64 neurons and TanH activations using the package. & Trained for $10,000$ epochs with a batch size of $512$ for fixed $50,000$ samples. \\
\hline
VAE & We use the same architecture for a noisy circle embedded in $\mathbb{R}^{3}$ as in \citet{shao2017riemanniangeometrydeepgenerative} with minor modifications. & Trained for $5,000$ epochs with a batch size of $512$ for fixed $50,000$ samples. \\
\hline
DDPM dinosaur & We use the architecture from \href{https://github.com/tanelp/tiny-diffusion}{github.com/tanelp/tiny-diffusion}. & Trained identical to \href{https://github.com/tanelp/tiny-diffusion}{github.com/tanelp/tiny-diffusion}. \\
\hline
\end{tabular}
\label{tab:models_architecture}
\end{table}

\paragraph{Hardware} The interpolation for \textsc{gmm} and \textsc{kde} as well as plots have been computed on a:
\textit{HP} computer with Intel Core i9-11950H 2.6 GHz 8C, 15.6'' FHD, 720P CAM, 32 GB (2$\times$16GB) DDR4 3200 So-Dimm, Nvidia Quadro TI2000 4GB Discrete Graphics, 1TB PCle NVMe SSD, 150W PSU, 8cell, W11Home 64 Advanced.

The interpolation for ControlNet, runtime tables, and benchmarks have been computed on a GPU for at most 24 hours with a maximum memory of $20$ GB. The $GPU$ consists of $4$ nodes on a \textit{Tesla V100}.

The interpolation of CIFAR10 and CelebAHQ with SGM \citep{song2021scorebasedgenerativemodelingstochastic} used in Appendix~\ref{ap:exp_additional} is conducted on a single \textit{Nvidia RTX A6000} GPU with 48G memory. 

The interpolation of OASIS3 with LDM used in Appendix~\ref{ap:exp_additional} is conducted on four \textit{Nvidia RTX A6000} GPUs.
\clearpage
\section{Additional experiments} \label{ap:exp_additional}
\paragraph{Energy-based model} In the following, we show the image transitions for the images used to compute \textsc{fid} and \textsc{kid} in Table~\ref{tab:energy_fid}. We show the result using our method and the benchmarks from \citep{bethune2025energy} in Fig.~\ref{fig:conf_ebm_invp}, Fig.~\ref{fig:conf_ebm_logp}, Fig.~\ref{fig:diag_ebm_invp}, Fig.~\ref{fig:diag_land_invp}, Fig.~\ref{fig:diag_land_invp}, Fig.~\ref{fig:newton_riemann_minvp} and Fig.~\ref{fig:newton_riemann_mlogp}.
\begin{figure}[h!]
    \vspace{0.0em}
    \centering
    \includegraphics[width=1.0\textwidth]{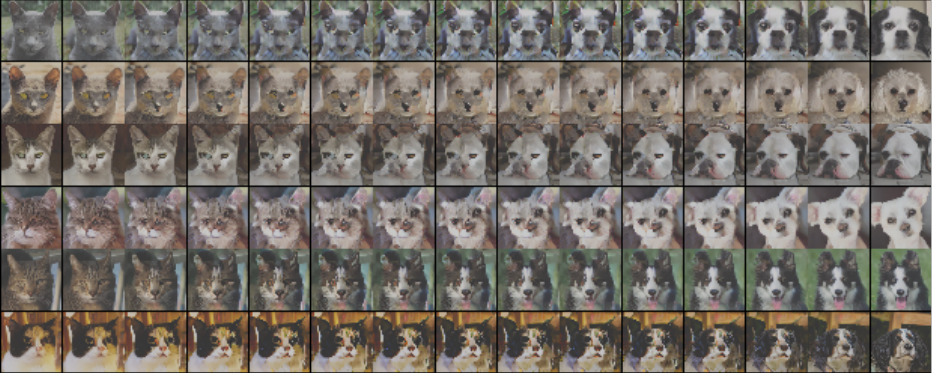}
    \caption{Image transitions on the AFHQ dataset \citep{choi2020starganv1} with $\mG_{1/p_{\theta}}$ (conformal).}
    \label{fig:conf_ebm_invp}
    \vspace{0.0em}
\end{figure}
\begin{figure}[h!]
    \vspace{0.0em}
    \centering
    \includegraphics[width=1.0\textwidth]{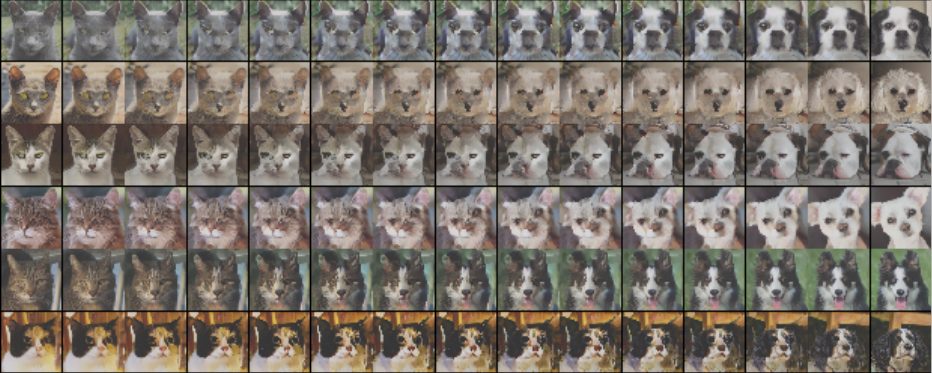}
    \caption{Image transitions on the AFHQ dataset \citep{choi2020starganv1} with $\mG_{\log p_{\theta}}$ (conformal).}
    \label{fig:conf_ebm_logp}
    \vspace{0.0em}
\end{figure}
\begin{figure}[h!]
    \vspace{0.0em}
    \centering
    \includegraphics[width=1.0\textwidth]{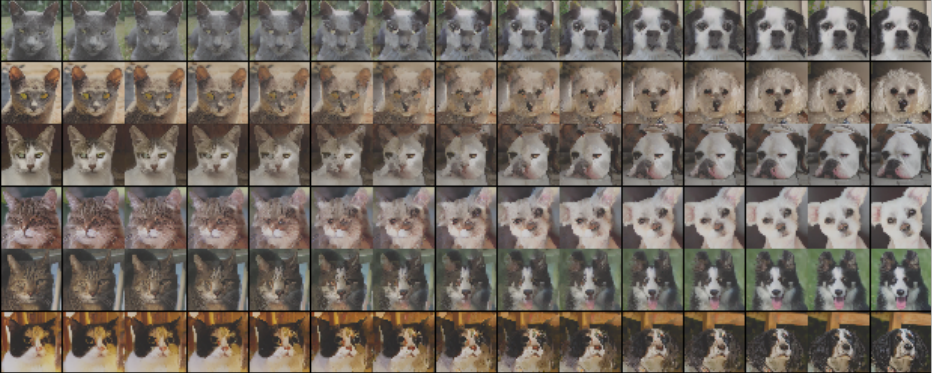}
    \caption{Image transitions on the AFHQ dataset \citep{choi2020starganv1} with $\mG_{1/p_{\theta}}$ (diagonal).}
    \label{fig:diag_ebm_invp}
    \vspace{0.0em}
\end{figure}
\begin{figure}[h!]
    \vspace{0.0em}
    \centering
    \includegraphics[width=1.0\textwidth]{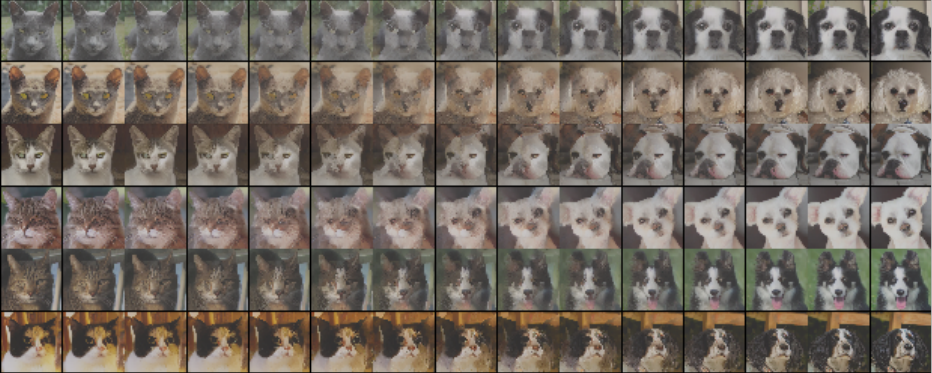}
    \caption{Image transitions on the AFHQ dataset \citep{choi2020starganv1} with $\mG_{\log p_{\theta}}$ (diagonal).}
    \label{fig:diag_ebm_logp}
    \vspace{0.0em}
\end{figure}
\begin{figure}[h!]
    \vspace{0.0em}
    \centering
    \includegraphics[width=1.0\textwidth]{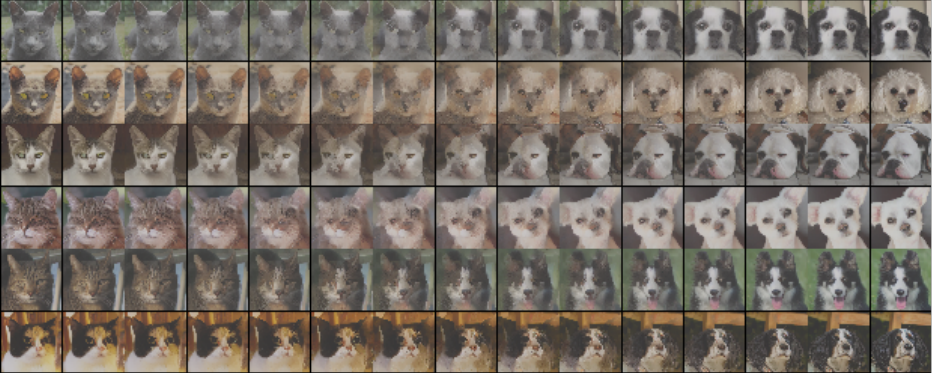}
    \caption{Image transitions on the AFHQ dataset \citep{choi2020starganv1} with $\mG_{\mathrm{LAND}}$.}
    \label{fig:diag_land_invp}
    \vspace{0.0em}
\end{figure}
\begin{figure}[h!]
    \vspace{0.0em}
    \centering
    \includegraphics[width=1.0\textwidth]{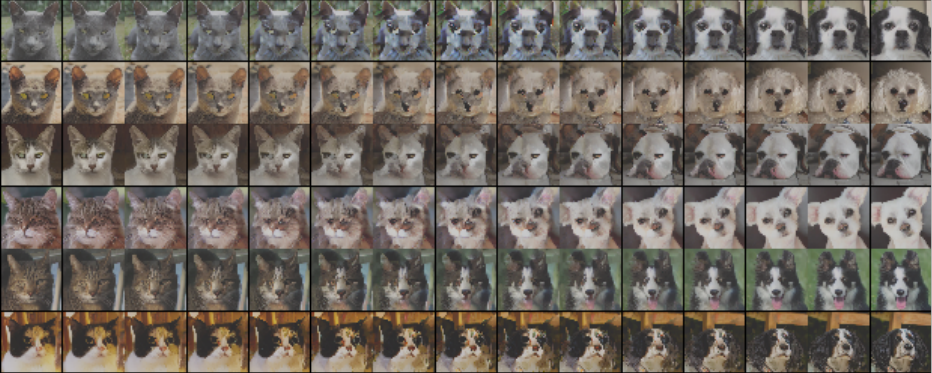}
    \caption{Image transitions on the AFHQ dataset \citep{choi2020starganv1} with $\mG_{1/p_{\theta}}$ (full).}
    \label{fig:full_ebm_invp}
    \vspace{0.0em}
\end{figure}
\begin{figure}[h!]
    \vspace{0.0em}
    \centering
    \includegraphics[width=1.0\textwidth]{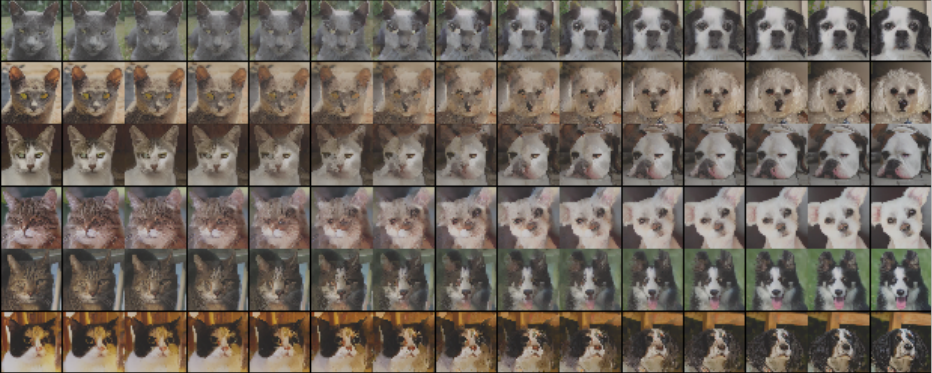}
    \caption{Image transitions on the AFHQ dataset \citep{choi2020starganv1} with $\mG_{\log p_{\theta}}$ (full).}
    \label{fig:full_ebm_logp}
    \vspace{0.0em}
\end{figure}
\begin{figure}[h!]
    \vspace{0.0em}
    \centering
    \includegraphics[width=1.0\textwidth]{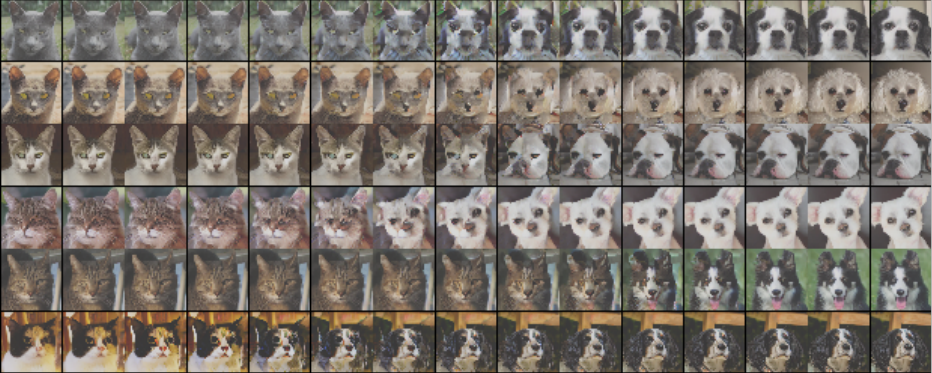}
    \caption{Image transitions on the AFHQ dataset \citep{choi2020starganv1} with our method ($S=-1/p_{\theta}$).}
    \label{fig:newton_riemann_minvp}
    \vspace{0.0em}
\end{figure}
\begin{figure}[h!]
    \vspace{0.0em}
    \centering
    \includegraphics[width=1.0\textwidth]{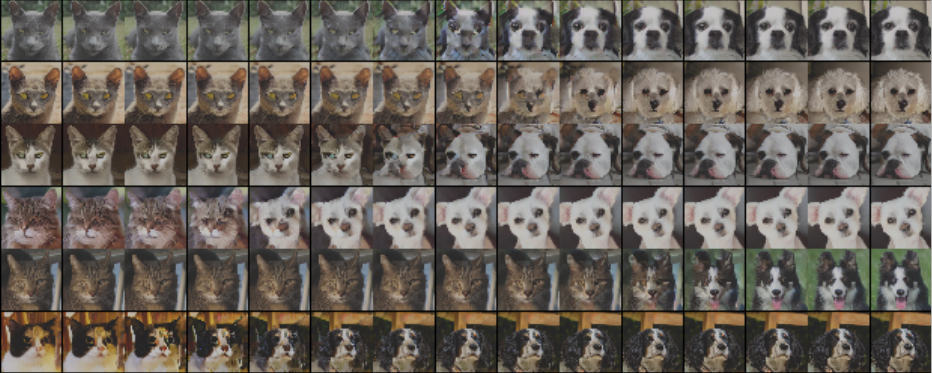}
    \caption{Image transitions on the AFHQ dataset \citep{choi2020starganv1} with our method ($S=-\log p_{\theta}$).}
    \label{fig:newton_riemann_mlogp}
    \vspace{0.0em}
\end{figure}

\paragraph{ControlNet} We provide additional qualitative examples for the ControlNet diffusion model \citep{zhang2023addingconditionalcontroltexttoimage} using the regularization function as described in Section~\ref{sec:experiments}. In Fig.~\ref{fig:image_row} and ~\ref{fig:house_image_row}, we show our interpolation method in noise space in the direction of a normalized standard normally distributed direction. We interpolate between an initial prompt and a target prompt linearly along the curve. We state the target prompt under the images.

\begin{figure}[h!]
    \vspace{0.0em}
    \centering
    \includegraphics[width=1.0\textwidth]{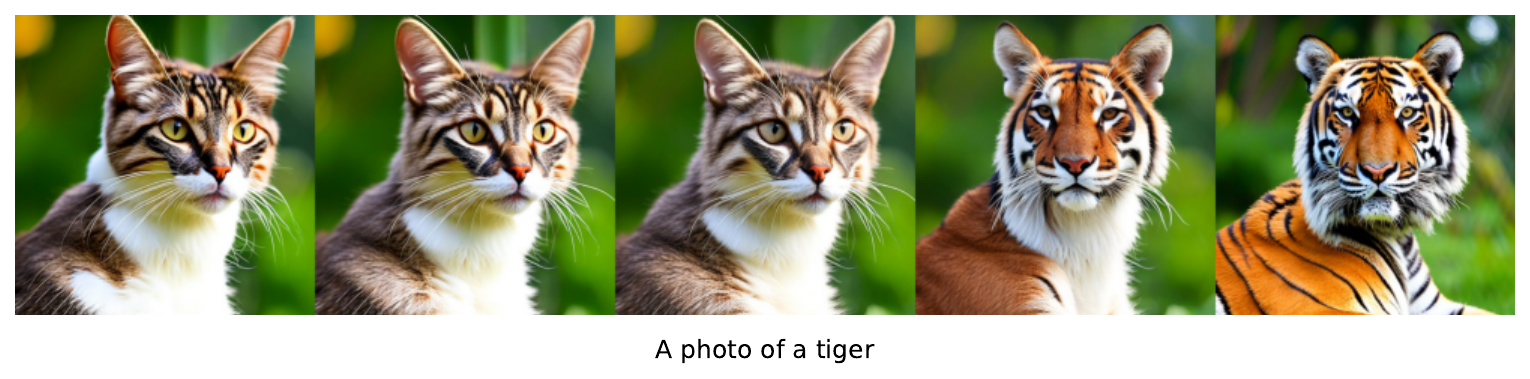}
    \caption{A transition corresponding to an initial value problem using our method in noise space for ControlNet, where we state the target prompt under the images. The initial prompt was "A photo of a cat".}
    \label{fig:image_row}
    \vspace{0.0em}
\end{figure}

\begin{figure}[h!]
    \vspace{0.0em}
    \centering
    \includegraphics[width=1.0\textwidth]{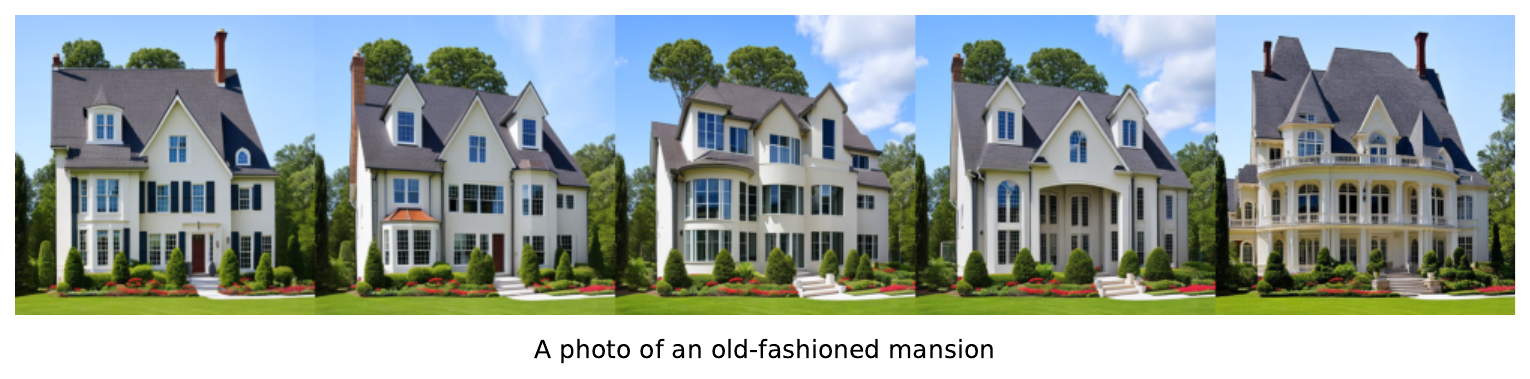}
    \caption{A transition corresponding to an initial value problem using our method in noise space for ControlNet, where we state the target prompt under the images. The initial prompt was "A photo of a house".}
    \label{fig:house_image_row}
    \vspace{0.0em}
\end{figure}

In Fig.~\ref{fig:mountainimage_grid},~\ref{fig:catimage_grid} and ~\ref{fig:aircraftimage_grid}, we show our interpolation method and other methods for curves connecting two images.

\begin{figure}[h!]
    \vspace{0.0em}
    \centering
    \includegraphics[width=1.0\textwidth]{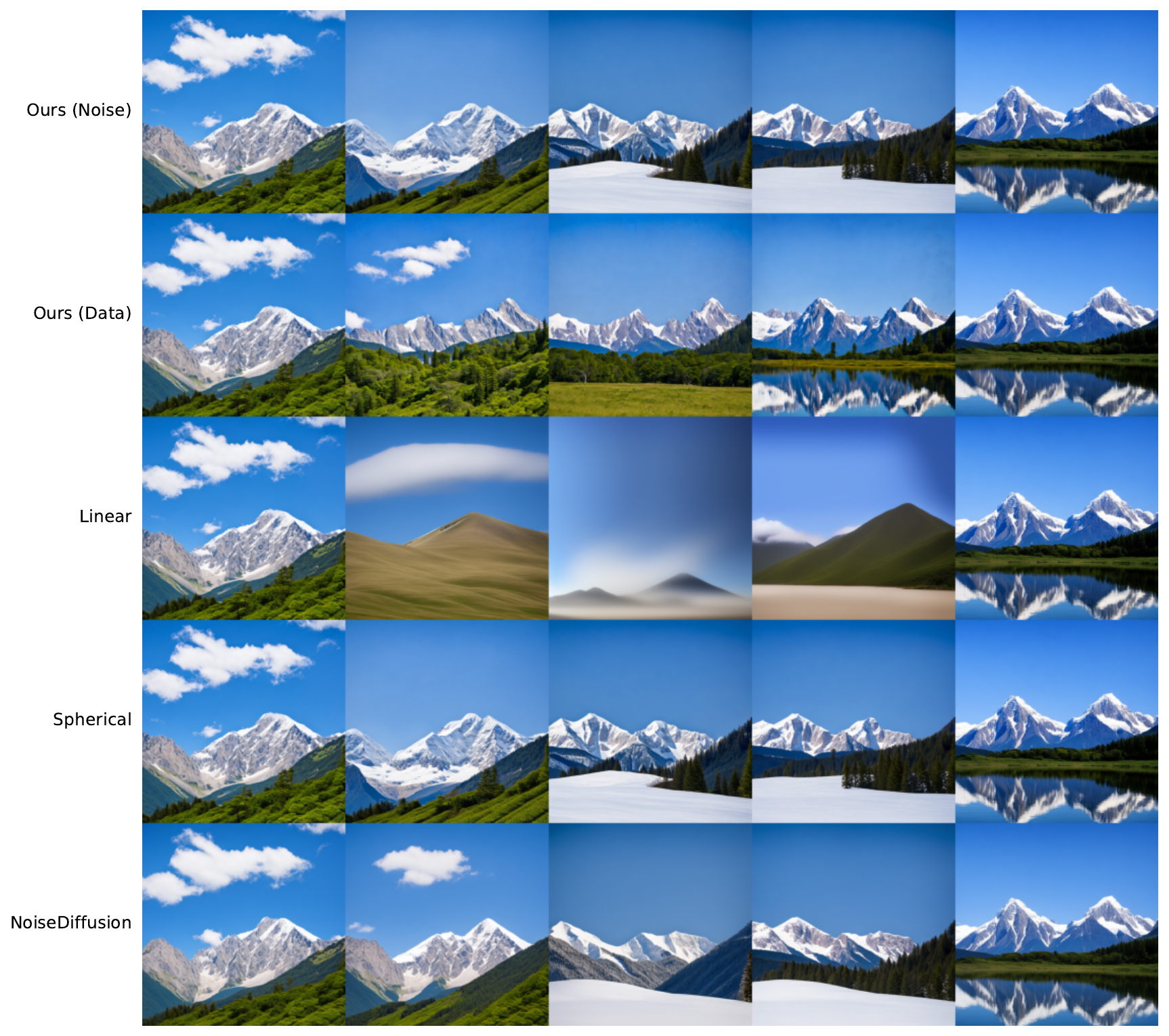}
    \caption{A display of different boundary value methods for ControlNet. The prompt was "A photo of a mountain".}
    \label{fig:mountainimage_grid}
    \vspace{0.0em}
\end{figure}

\begin{figure}[h!]
    \vspace{0.0em}
    \centering
    \includegraphics[width=1.0\textwidth]{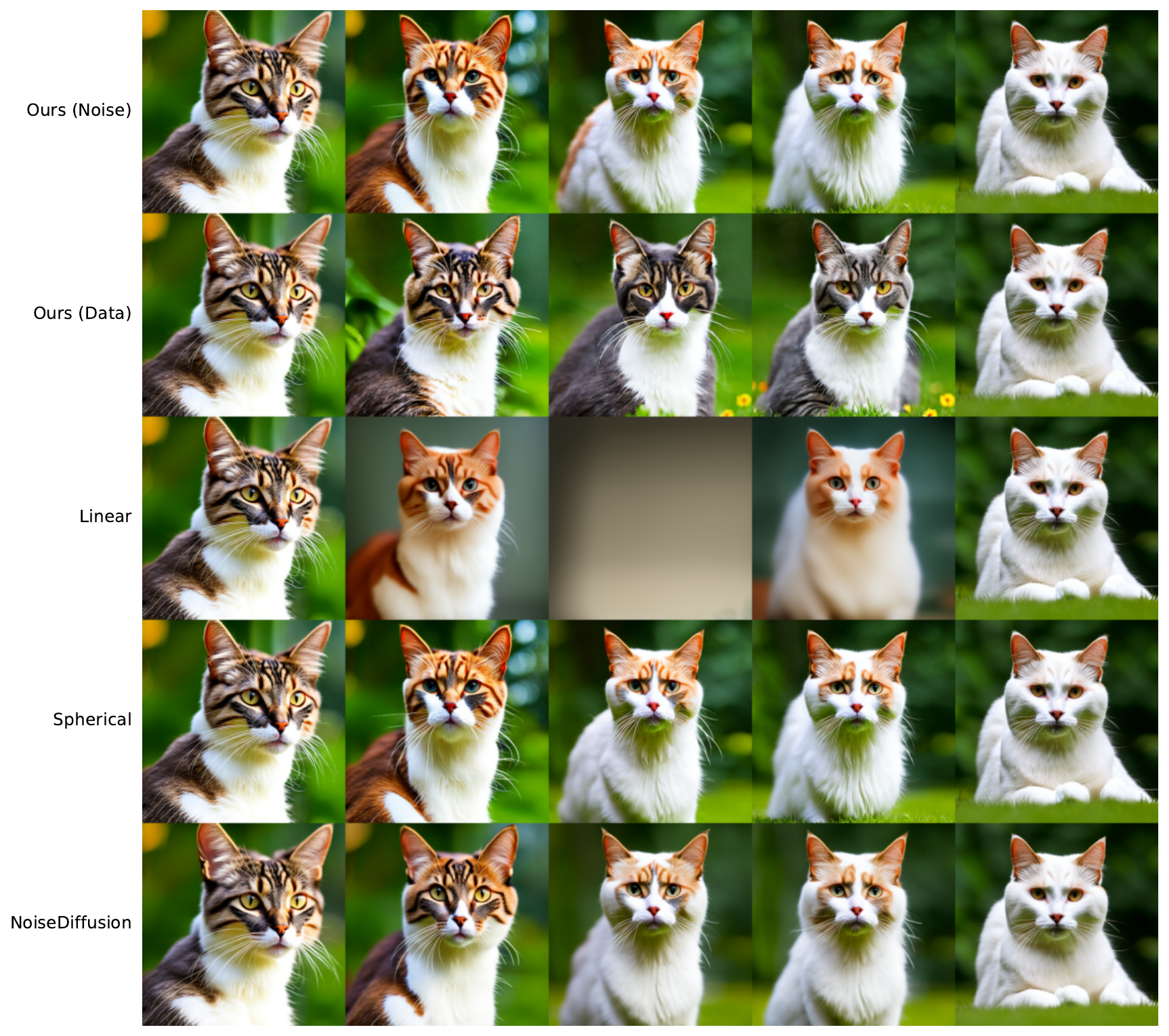}
    \caption{A display of different boundary value methods for ControlNet. The prompt was "A photo of a cat".}
    \label{fig:catimage_grid}
    \vspace{0.0em}
\end{figure}

\begin{figure}[h!]
    \vspace{0.0em}
    \centering
    \includegraphics[width=1.0\textwidth]{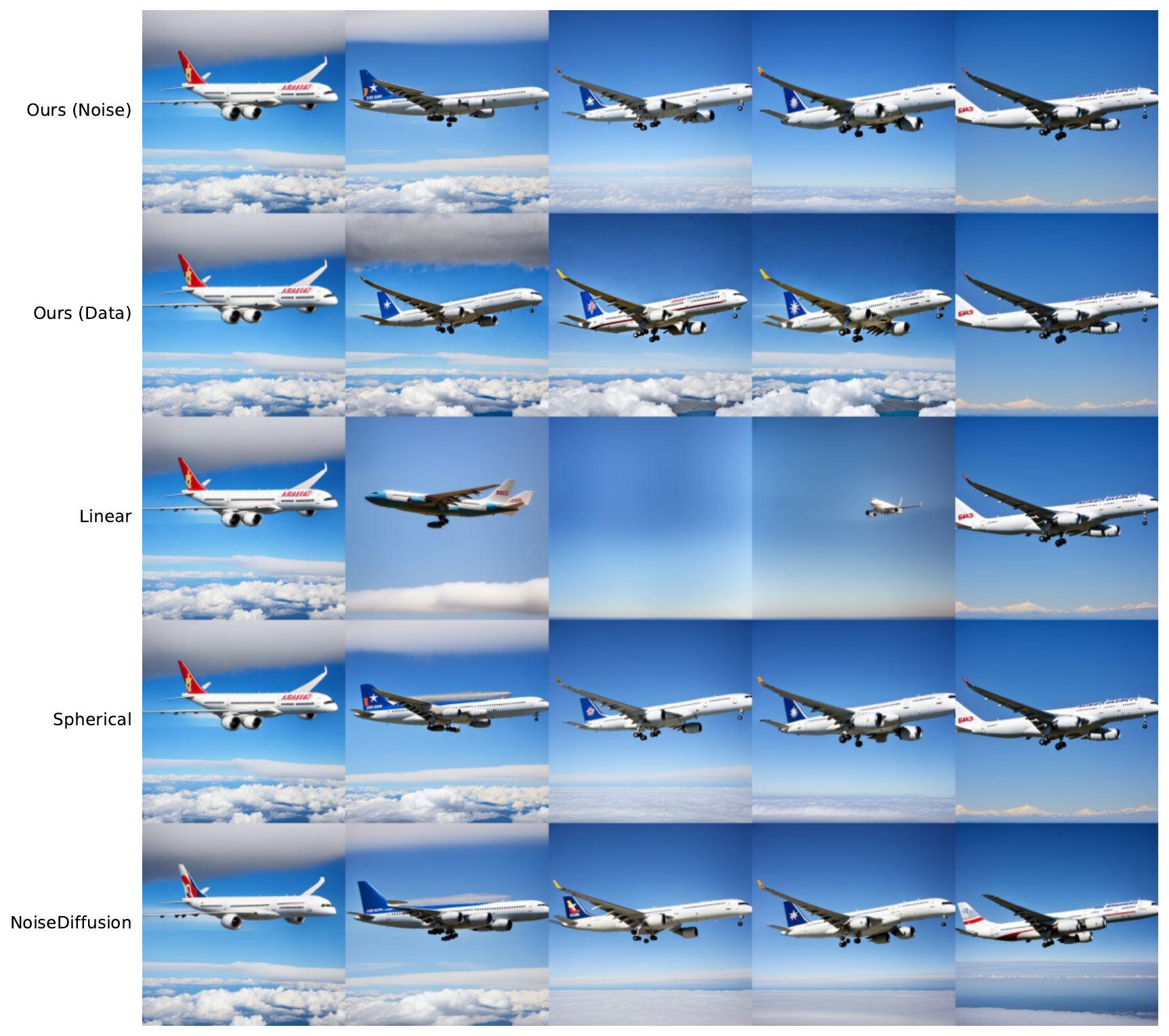}
    \caption{A display of different boundary value methods for ControlNet. The prompt was "A photo of an aircraft".}
    \label{fig:aircraftimage_grid}
    \vspace{0.0em}
\end{figure}

In Fig.~\ref{fig:spherical_mean_grid_left}, ~\ref{fig:linear_mean_grid_left} and ~\ref{fig:pgeorce_noise_mean_grid_left} we show the estimated mean and their transition for different methods.

\begin{figure}[h!]
    \vspace{0.0em}
    \centering
    \includegraphics[width=1.0\textwidth]{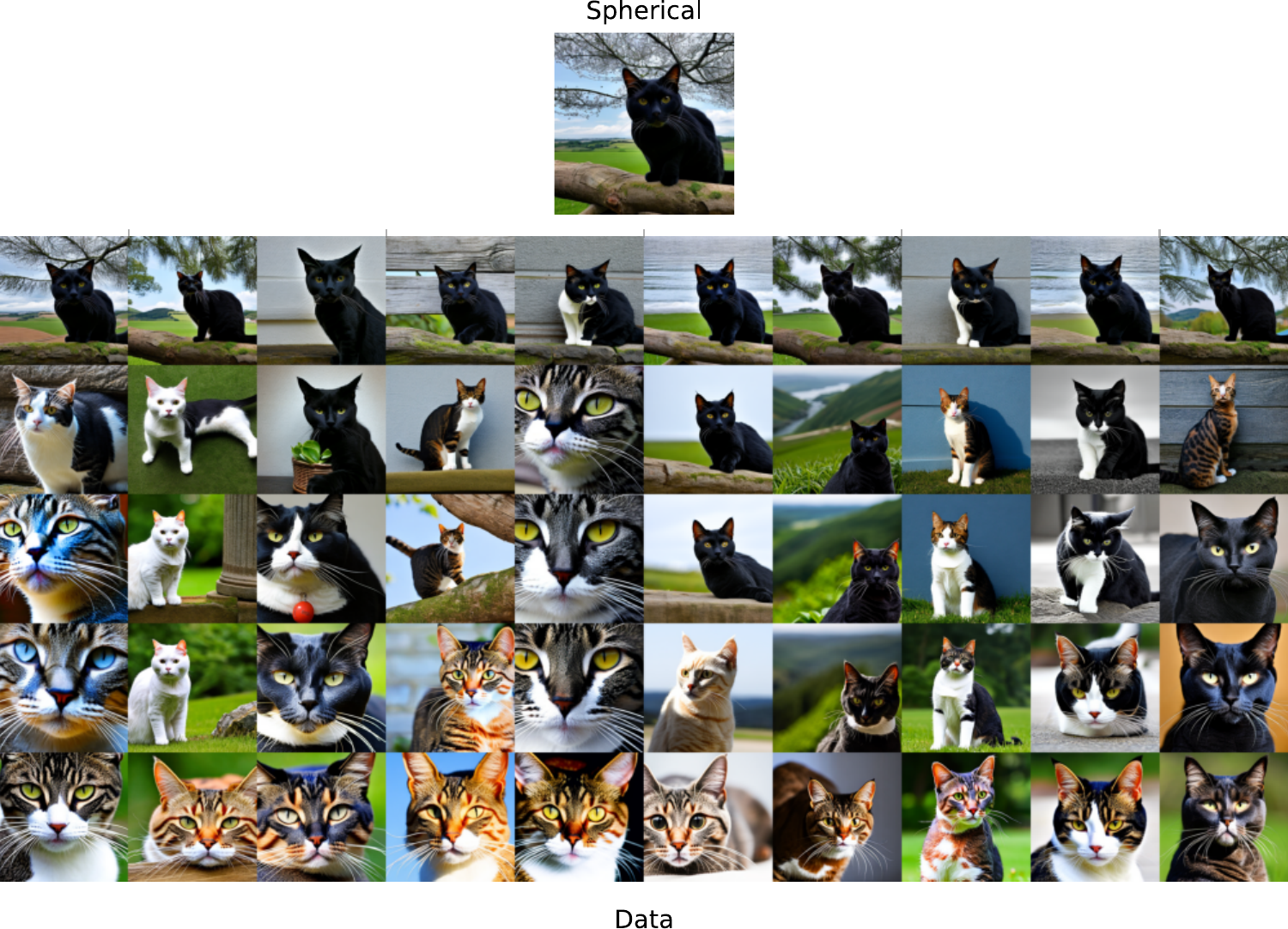}
    \caption{The estimated mean using spherical geometry for the data in Fig.~\ref{fig:mean_grid_left}. The prompt was "A photo of a cat".}
    \label{fig:spherical_mean_grid_left}
    \vspace{0.0em}
\end{figure}

\begin{figure}[h!]
    \vspace{0.0em}
    \centering
    \includegraphics[width=1.0\textwidth]{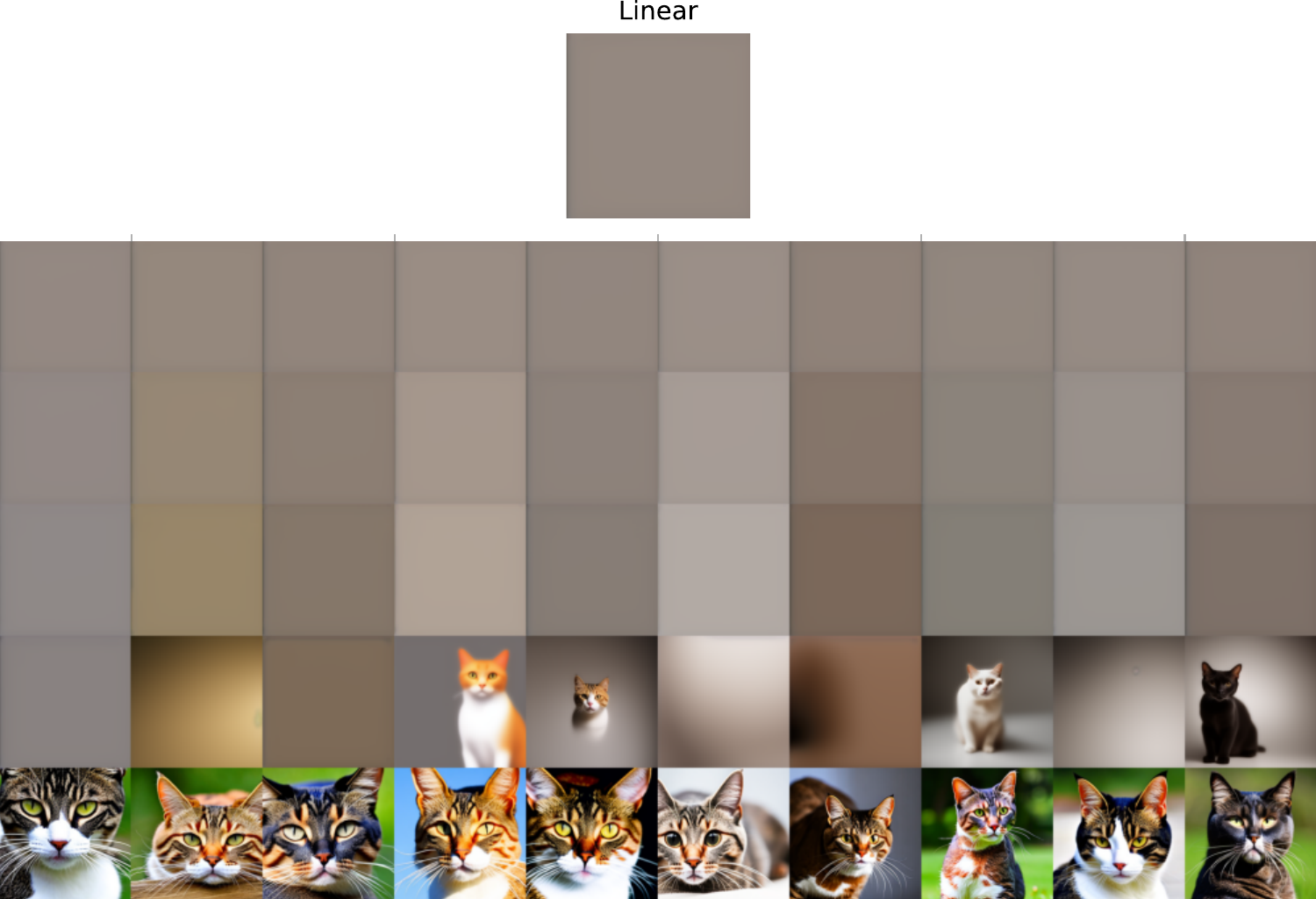}
    \caption{The estimated mean using linear geometry for the data in Fig.~\ref{fig:mean_grid_left}. The prompt was "A photo of a cat".}
    \label{fig:linear_mean_grid_left}
    \vspace{0.0em}
\end{figure}

\begin{figure}[h!]
    \vspace{0.0em}
    \centering
    \includegraphics[width=1.0\textwidth]{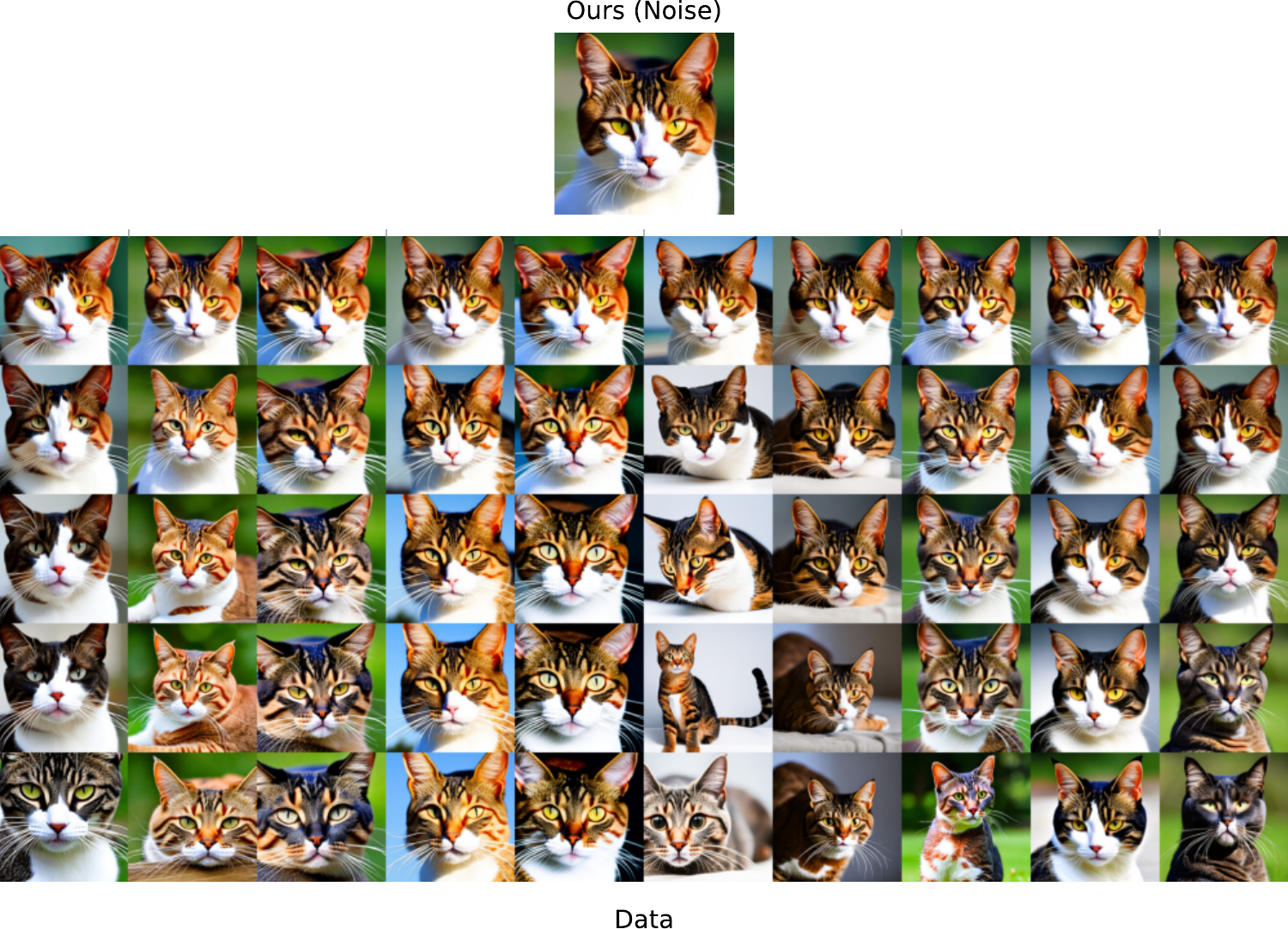}
    \caption{The estimated mean using our method in noise space for the data in Fig.~\ref{fig:mean_grid_left}. The prompt was "A photo of a cat".}
    \label{fig:pgeorce_noise_mean_grid_left}
    \vspace{0.0em}
\end{figure}

\clearpage
\paragraph{Other diffusion models} We consider the score-based generative model (SGM. \citep{song2021scorebasedgenerativemodelingstochastic}, where we use the variance exploding SDE (VESDE) for both CIFAR10 and CelebAHQ datasets with SGM. For CIFAR10 interpolation, the minimum $\sigma$ and maximum $\sigma$ are set to 0.01 and 50 respectively, with the number of scales being 1000. For CelebAHQ interpolation, we use 0.01 and 348 as the minimum and maximum $\sigma$. The sampling $\epsilon$ is set to $1e-5$ for both cases. We also consider the latent score-based generative model \citep{vahdat2021scorebasedgenerativemodelinglatent} for the mentioned datasets as well as OASIS3 \citep{LaMontagne2019.12.13.19014902} with LDM \citep{rombach2022highresolutionimagesynthesislatent}, where the linear $\beta$ time schedule with 1000 time steps is adopted with the initial $\beta$ value being 0.0015 and the last $\beta$ value being 0.0205. For the score-based models, we sample using the probability flow \textsc{ode}, while for the OASIS3 we use DDIM sampling from noise space \citep{song2022denoisingdiffusionimplicitmodels}. For our method, we set $\lambda=1.0$ for both datasets and consider the regularization function.
\begin{equation*}
    S(x) = -\log p\left(\norm{x}^{2}\right) + 0.1\left(\norm{x}^{2}-r^{2}\right),
\end{equation*}
where $r$ denotes the data dimension. Note that we apply this slightly different regularization function compared to ControlNet to account for the dimension of the noise space being lower. We compute the likelihood of the interpolation curves and \textsc{fid} for our method compared to linear and spherical interpolation in Table~\ref{tab:diffusion_log_likelihood}. In Fig.~\ref{fig:brain_interpolation}, ~\ref{fig:samples_celebahq} and ~\ref{fig:samples_cifar10}, we provide interpolation paths.

\begin{table}[h!]
    \scriptsize
    \centering
    \caption{Mean log-likelihood / \textsc{FID} for interpolation in different diffusion models and datasets. SGM is score-based generative models \citep{kingma2017adammethodstochasticoptimization}, while LSGM is latent score-based generative model \citep{vahdat2021scorebasedgenerativemodelinglatent}. OASIS3 \citep{LaMontagne2019.12.13.19014902} was trained using latent diffusion model \citep{rombach2022highresolutionimagesynthesislatent} used by \citet{pinaya2022brainimaginggenerationlatent}.}

    \begin{tabular*}{\textwidth}{@{\extracolsep{\fill}} l c c c c c}
        \toprule
        \textbf{Interpolation} &
        \textbf{OASIS3} &
        \textbf{CIFAR10 (SGM)} &
        \textbf{CelebAHQ (SGM)} &
        \textbf{CIFAR10 (LSGM)} &
        \textbf{CelebAHQ (LSGM)} \\
        \midrule
        Linear &
        27.83 / 85.94 &
        0.33 / 391.64 &
        -1.61 / 285.32 &
        26.49 / 254.39 &
        1.51 / 149.53 \\
        Spherical &
        \textbf{28.46} / 68.97 &
        \textbf{4.30} / 237.64 &
        -1.56 / 142.14 &
        \textbf{30.00} / \textbf{217.10} &
        \textbf{1.70} / 145.02 \\
        ProbGEORCE &
        28.41 / \textbf{66.54} &
        4.12 / \textbf{233.84} &
        \textbf{-0.46} / \textbf{138.69} &
        29.99 / 217.52 &
        \textbf{1.70} / \textbf{143.79} \\
        \bottomrule
    \end{tabular*}
    \label{tab:diffusion_log_likelihood}
    \vspace{-2em}
\end{table}

\begin{figure}[h!]
    \vspace{0.0em}
    \centering
    \includegraphics[width=1.0\textwidth]{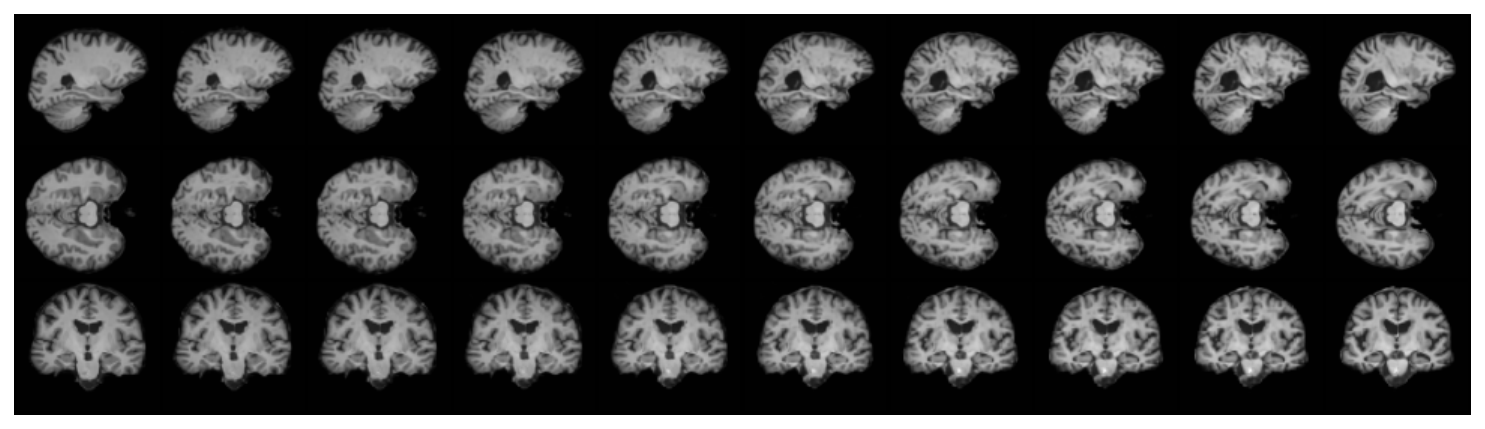}
    \caption{Interpolation using a latent diffusion model similar to \citep{pinaya2022brainimaginggenerationlatent} with our proposed method. Each row represents different slices of the brain. The interpolation shows the generative transition between a healthy brain (left) and a brain with Alzheimer disease (right).}
    \label{fig:brain_interpolation}
    \vspace{0.0em}
\end{figure}

\begin{figure}[h!]
    \vspace{0.0em}
    \centering
    \includegraphics[width=1.0\textwidth]{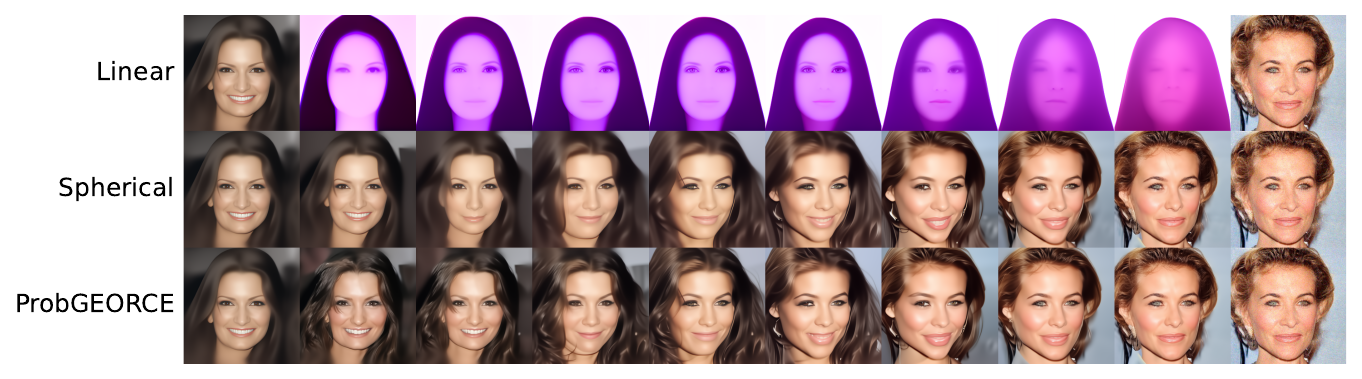}
    \caption{Comparisons of interpolation curves for CelebAHQ with SGM.}
    \label{fig:samples_celebahq}
    \vspace{0.0em}
\end{figure}

\begin{figure}[h!]
    \vspace{0.0em}
    \centering
    \includegraphics[width=1.0\textwidth]{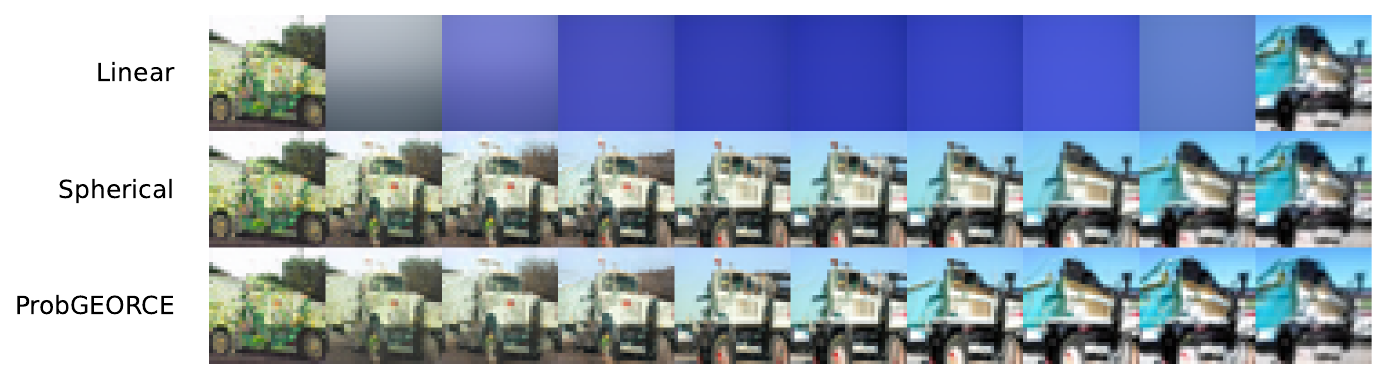}
    \caption{Comparisons of interpolation curves for CIFAR10 with SGM.}
    \label{fig:samples_cifar10}
    \vspace{0.0em}
\end{figure}

\clearpage
\section{Additional runtime estimates} \label{ap:add_runtime}
In this section, we provide additional runtime estimates using our method compared to standard optimizers. We consider the \textsc{bvp} for Eq.~\ref{eq:metric_definition} using our method \textit{ProbGEORCE} with line-search and with adaptive update scheme. In Fig.~\ref{fig:gen2d_runtime}, we show the estimated regularized energy and runtime for different optimizers for different values of $\lambda$.
\begin{figure*}[h]
    \centering
    \includegraphics[
      width=1.0\textwidth,
      height=1.0\textheight, 
      keepaspectratio
    ]{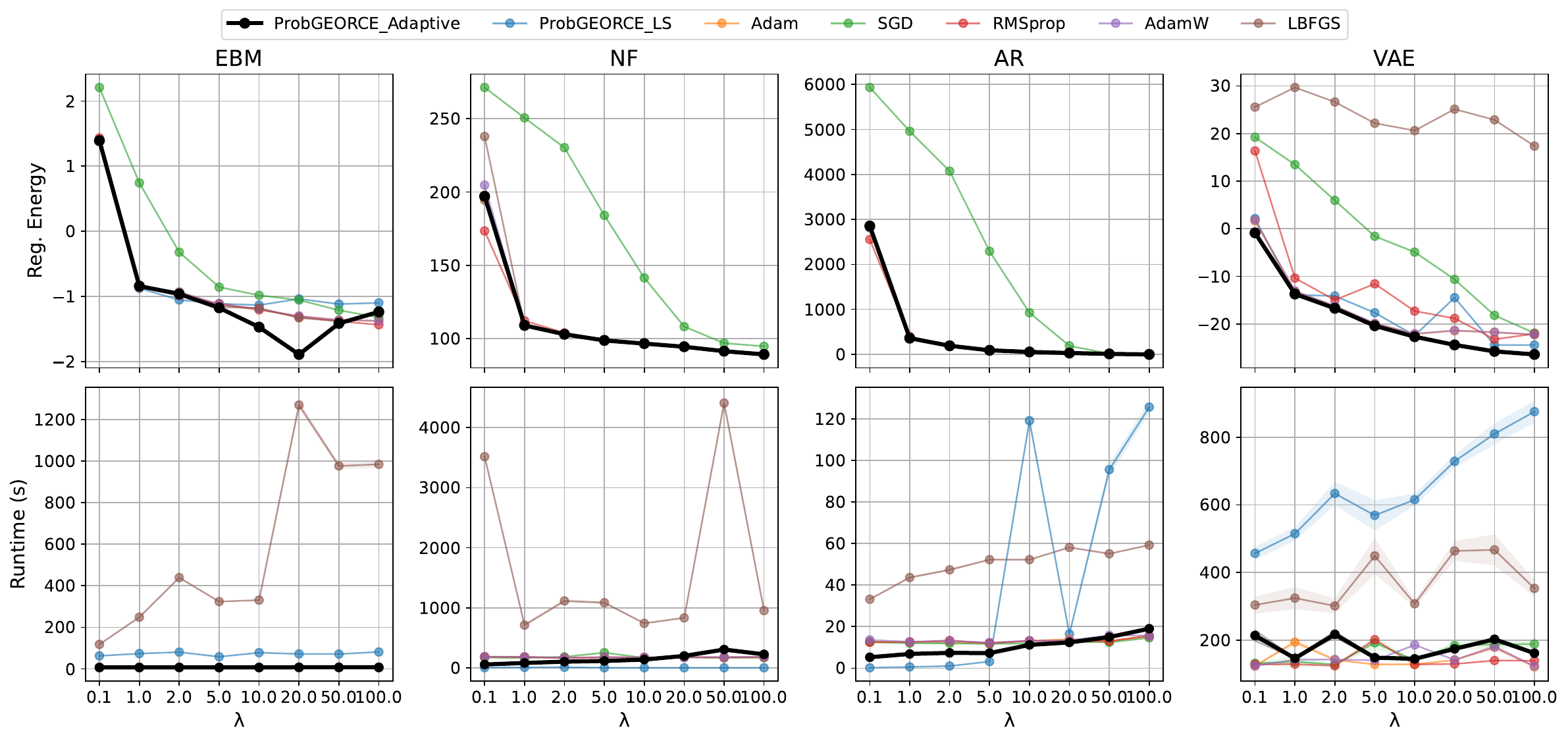}
    \caption{The regularized energy and runtime for different optimizers for the regularized energy in Eq.~\ref{eq:metric_definition} solving the \textsc{bvp} for the models in Fig.~\ref{fig:multirow_plot_2d_gen}.}
    \label{fig:gen2d_runtime}
    \vspace{-1.0em}
\end{figure*}
In general, we see that our method is faster for lower values of $\lambda$ and as $\lambda$ increases, the methods converge to each other. This is expected as when $\lambda$ increases, the metric-term becomes negligible, and the regularization term dominates the expression. In this case, the minimization problem is approximately the same as minimizing the $S$ function, where \textit{ProbGEORCE} does not exploit any specific structures of $S$ that can speed-up the computation.

We also see that the adaptive version is usually faster than line-search, especially if the regularized energy is expensive to evaluate as with the \textsc{vae}, which is expected as this avoids repeated evaluations of the regularized energy functional. In Table~\ref{tab:gen2d_runtimes}, we show the corresponding runtimes for Fig.~\ref{fig:gen2d_runtime}. Note that we terminate the methods if the gradient of the reguralized energy is less than $10^{-4}$. Given the runtimes, it seems that no algorithms terminate beforehand, which is mostly likely for numerical reasons due to the neural network.
\begin{table*}[h!]
\caption{Runtime and regularized energy for different methods for different values of $\lambda$ for the methods in Fig.~\ref{fig:multirow_plot_2d_gen}. For ProbGEORCE (LS) for \textsc{nf}, the neural networks estimate of the likelihood returned nans during the iterations, and therefore we set $-$.}
\centering
\resizebox{\textwidth}{!}{%
\begin{tabular}{lcccccccc}
\toprule
  & \multicolumn{2}{c}{\bfseries\large EBM} & \multicolumn{2}{c}{\bfseries\large NF} & \multicolumn{2}{c}{\bfseries\large AR} & \multicolumn{2}{c}{\bfseries\large VAE} \\
\cmidrule(lr){2-3} \cmidrule(lr){4-5} \cmidrule(lr){6-7} \cmidrule(lr){8-9} 
Optimizer & Reg. Energy & Runtime & Reg. Energy & Runtime & Reg. Energy & Runtime & Reg. Energy & Runtime \\
\midrule
\multicolumn{9}{c}{\bfseries\large $\lambda = 1.0$} \\
\bottomrule
\addlinespace[0.3em]
ProbGEORCE (Adaptive) & 0.124 & 7.45 $\pm$ 0.01 & \textbf{0.087} & 81.03 $\pm$ 0.39 & \textbf{0.115} & 6.83 $\pm$ 0.01 & 0.043 & 145.86 $\pm$ 0.53 \\
ProbGEORCE (LS) & \textbf{0.124} & 73.62 $\pm$ 0.29 & \textbf{0.087} & \textbf{6.42 $\pm$ 0.01} & 0.115 & \textbf{0.58 $\pm$ 0.00} & \textbf{0.037} & 514.56 $\pm$ 10.30 \\
Adam & 0.124 & 6.60 $\pm$ 0.00 & 0.087 & 181.43 $\pm$ 3.04 & 0.115 & 12.65 $\pm$ 0.03 & 0.045 & 194.05 $\pm$ 5.60 \\
SGD & 0.206 & \textbf{5.70 $\pm$ 0.01} & 0.111 & 166.12 $\pm$ 2.02 & 0.175 & 11.99 $\pm$ 0.02 & 0.150 & 135.86 $\pm$ 9.14 \\
RMSprop & 0.146 & 6.02 $\pm$ 0.04 & 0.107 & 178.22 $\pm$ 3.90 & 0.134 & 12.58 $\pm$ 0.02 & 0.068 & \textbf{128.54 $\pm$ 0.20} \\
AdamW & 0.124 & 6.68 $\pm$ 0.03 & 0.087 & 182.80 $\pm$ 1.65 & 0.115 & 12.67 $\pm$ 0.01 & 0.045 & 141.78 $\pm$ 0.72 \\
LBFGS & 0.124 & 248.88 $\pm$ 2.07 & 0.087 & 711.39 $\pm$ 10.12 & 0.115 & 43.66 $\pm$ 0.23 & 0.214 & 324.14 $\pm$ 16.69 \\
\addlinespace[0.5em]
\midrule
\multicolumn{9}{c}{\bfseries\large $\lambda = 5.0$} \\
\bottomrule
\addlinespace[0.3em]
ProbGEORCE (Adaptive) & \textbf{-0.126} & 7.11 $\pm$ 0.24 & \textbf{0.174} & 114.12 $\pm$ 1.32 & \textbf{0.125} & 7.23 $\pm$ 0.01 & \textbf{-0.289} & 147.95 $\pm$ 0.42 \\
ProbGEORCE (LS) & -0.107 & 58.47 $\pm$ 0.74 & - & \textbf{0.66 $\pm$ 0.00} & 0.125 & \textbf{3.16 $\pm$ 0.01} & -0.234 & 568.85 $\pm$ 23.03 \\
Adam & -0.101 & 6.28 $\pm$ 0.03 & 0.174 & 172.02 $\pm$ 2.80 & 0.125 & 12.02 $\pm$ 0.02 & -0.279 & \textbf{127.81 $\pm$ 0.18} \\
SGD & -0.010 & \textbf{5.61 $\pm$ 0.04} & 0.255 & 252.45 $\pm$ 0.02 & 0.283 & 11.68 $\pm$ 0.04 & 0.112 & 191.93 $\pm$ 9.31 \\
RMSprop & -0.085 & 6.28 $\pm$ 0.05 & 0.194 & 173.34 $\pm$ 2.73 & 0.146 & 11.96 $\pm$ 0.02 & -0.060 & 201.13 $\pm$ 1.66 \\
AdamW & -0.100 & 6.52 $\pm$ 0.03 & 0.174 & 179.02 $\pm$ 0.36 & 0.125 & 12.38 $\pm$ 0.07 & -0.279 & 139.85 $\pm$ 0.68 \\
LBFGS & -0.119 & 323.73 $\pm$ 0.31 & 0.174 & $1.08 \times 10^{3}$ $\pm$ 11.59 & 0.125 & 52.26 $\pm$ 0.05 & 0.582 & 448.98 $\pm$ 27.14 \\
\addlinespace[0.5em]
\midrule
\multicolumn{9}{c}{\bfseries\large $\lambda = 10.0$} \\
\bottomrule
\addlinespace[0.3em]
ProbGEORCE (Adaptive) & \textbf{-0.577} & 7.25 $\pm$ 0.05 & \textbf{0.278} & 136.70 $\pm$ 0.31 & 0.130 & \textbf{11.18 $\pm$ 0.07} & -0.789 & 144.33 $\pm$ 0.54 \\
ProbGEORCE (LS) & -0.368 & 78.16 $\pm$ 0.33 & - & \textbf{0.63 $\pm$ 0.01} & \textbf{0.130} & 119.25 $\pm$ 0.33 & \textbf{-0.794} & 615.24 $\pm$ 9.14 \\
Adam & -0.420 & 6.46 $\pm$ 0.05 & 0.278 & 172.64 $\pm$ 0.26 & 0.130 & 13.19 $\pm$ 0.02 & -0.763 & 128.05 $\pm$ 0.13 \\
SGD & -0.305 & \textbf{5.77 $\pm$ 0.02} & 0.363 & 165.85 $\pm$ 2.17 & 0.255 & 11.94 $\pm$ 0.33 & -0.028 & 140.18 $\pm$ 0.71 \\
RMSprop & -0.397 & 6.23 $\pm$ 0.01 & 0.298 & 171.90 $\pm$ 0.60 & 0.150 & 13.13 $\pm$ 0.05 & -0.480 & \textbf{127.77 $\pm$ 0.16} \\
AdamW & -0.409 & 6.51 $\pm$ 0.04 & 0.278 & 170.25 $\pm$ 0.36 & 0.130 & 13.15 $\pm$ 0.03 & -0.763 & 185.49 $\pm$ 1.18 \\
LBFGS & -0.409 & 330.91 $\pm$ 0.33 & 0.278 & 741.10 $\pm$ 2.46 & \textbf{0.130} & 52.27 $\pm$ 0.06 & 1.015 & 307.28 $\pm$ 5.25 \\
\addlinespace[0.5em]
\midrule
\multicolumn{9}{c}{\bfseries\large $\lambda = 20.0$} \\
\bottomrule
\addlinespace[0.3em]
ProbGEORCE (Adaptive) & \textbf{-1.791} & 7.83 $\pm$ 0.02 & \textbf{0.481} & 198.07 $\pm$ 1.98 & \textbf{0.136} & \textbf{12.38 $\pm$ 0.01} & \textbf{-1.875} & 173.77 $\pm$ 0.72 \\
ProbGEORCE (LS) & -0.903 & 72.13 $\pm$ 0.34 & - & \textbf{0.64 $\pm$ 0.00} & \textbf{0.136} & 16.41 $\pm$ 0.06 & -0.821 & 728.79 $\pm$ 9.63 \\
Adam & -1.157 & 6.35 $\pm$ 0.02 & 0.481 & 172.26 $\pm$ 1.23 & 0.136 & 13.85 $\pm$ 0.05 & -1.633 & 140.63 $\pm$ 0.78 \\
SGD & -0.930 & \textbf{5.68 $\pm$ 0.01} & 0.528 & 173.78 $\pm$ 0.46 & 0.181 & 12.71 $\pm$ 0.05 & -0.777 & 183.66 $\pm$ 5.63 \\
RMSprop & -1.149 & 7.30 $\pm$ 0.04 & 0.501 & 179.19 $\pm$ 0.30 & 0.156 & 13.08 $\pm$ 0.74 & -1.255 & \textbf{129.16 $\pm$ 0.21} \\
AdamW & -1.154 & 6.39 $\pm$ 0.02 & 0.481 & 173.33 $\pm$ 0.18 & 0.136 & 13.34 $\pm$ 0.02 & -1.630 & 141.13 $\pm$ 0.52 \\
LBFGS & -1.181 & $1.27 \times 10^{3}$ $\pm$ 10.42 & 0.481 & 832.08 $\pm$ 6.46 & \textbf{0.136} & 58.16 $\pm$ 0.05 & 2.360 & 463.50 $\pm$ 14.78 \\
\addlinespace[0.5em]
\midrule
\multicolumn{9}{c}{\bfseries\large $\lambda = 50.0$} \\
\bottomrule
\addlinespace[0.3em]
ProbGEORCE (Adaptive) & \textbf{-3.596} & 7.80 $\pm$ 0.03 & \textbf{1.074} & 304.15 $\pm$ 0.50 & 0.144 & 14.94 $\pm$ 0.27 & \textbf{-5.309} & 202.25 $\pm$ 0.22 \\
ProbGEORCE (LS) & -2.611 & 71.58 $\pm$ 0.19 & - & \textbf{0.63 $\pm$ 0.01} & 0.144 & 95.65 $\pm$ 1.56 & -5.018 & 809.95 $\pm$ 15.94 \\
Adam & -3.490 & 6.39 $\pm$ 0.04 & 1.074 & 171.42 $\pm$ 0.36 & 0.144 & 13.04 $\pm$ 0.25 & -4.231 & 178.80 $\pm$ 4.90 \\
SGD & -3.134 & \textbf{5.69 $\pm$ 0.02} & 1.110 & 168.95 $\pm$ 2.57 & 0.149 & \textbf{12.41 $\pm$ 0.13} & -3.666 & 186.83 $\pm$ 6.66 \\
RMSprop & -3.513 & 6.29 $\pm$ 0.02 & 1.094 & 171.08 $\pm$ 2.67 & 0.165 & 12.99 $\pm$ 0.22 & -4.596 & \textbf{139.37 $\pm$ 1.44} \\
AdamW & -3.487 & 6.37 $\pm$ 0.02 & 1.074 & 178.96 $\pm$ 0.58 & 0.144 & 15.84 $\pm$ 0.13 & -4.230 & 179.64 $\pm$ 4.45 \\
LBFGS & -3.318 & 976.39 $\pm$ 7.88 & - & $4.41 \times 10^{3}$ $\pm$ 19.60 & \textbf{0.144} & 55.11 $\pm$ 0.04 & 5.274 & 466.87 $\pm$ 23.32 \\
\addlinespace[0.5em]
\midrule
\multicolumn{9}{c}{\bfseries\large $\lambda = 100.0$} \\
\bottomrule
\addlinespace[0.3em]
ProbGEORCE (Adaptive) & -5.980 & 7.78 $\pm$ 0.04 & 2.034 & 224.24 $\pm$ 2.25 & 0.146 & 19.00 $\pm$ 0.31 & \textbf{-11.225} & 161.09 $\pm$ 0.76 \\
ProbGEORCE (LS) & -5.787 & 81.56 $\pm$ 0.56 & - & \textbf{0.66 $\pm$ 0.00} & 0.148 & 125.78 $\pm$ 1.56 & -10.326 & 875.75 $\pm$ 16.40 \\
Adam & -7.405 & 6.26 $\pm$ 0.03 & 2.034 & 173.19 $\pm$ 1.58 & 0.146 & 15.60 $\pm$ 0.15 & -9.053 & \textbf{120.66 $\pm$ 0.22} \\
SGD & -7.192 & \textbf{5.64 $\pm$ 0.04} & 2.108 & 169.45 $\pm$ 2.83 & 0.151 & \textbf{14.60 $\pm$ 0.13} & -9.018 & 188.26 $\pm$ 5.28 \\
RMSprop & \textbf{-7.763} & 6.17 $\pm$ 0.05 & 2.052 & 179.11 $\pm$ 0.22 & 0.168 & 15.50 $\pm$ 0.02 & -8.921 & 138.88 $\pm$ 0.60 \\
AdamW & -7.400 & 6.32 $\pm$ 0.01 & 2.034 & 182.77 $\pm$ 0.33 & 0.146 & 15.64 $\pm$ 0.12 & -9.033 & 122.84 $\pm$ 0.23 \\
LBFGS & -4.619 & 984.25 $\pm$ 8.40 & \textbf{2.034} & 952.73 $\pm$ 12.46 & \textbf{0.146} & 59.30 $\pm$ 0.07 & 7.969 & 352.84 $\pm$ 12.44 \\
\addlinespace[0.5em]
\bottomrule
\end{tabular}}
\label{tab:gen2d_runtimes}
\end{table*}

\begin{figure}
    \vspace{0.0em}
    \centering
    \includegraphics[width=0.6\textwidth]{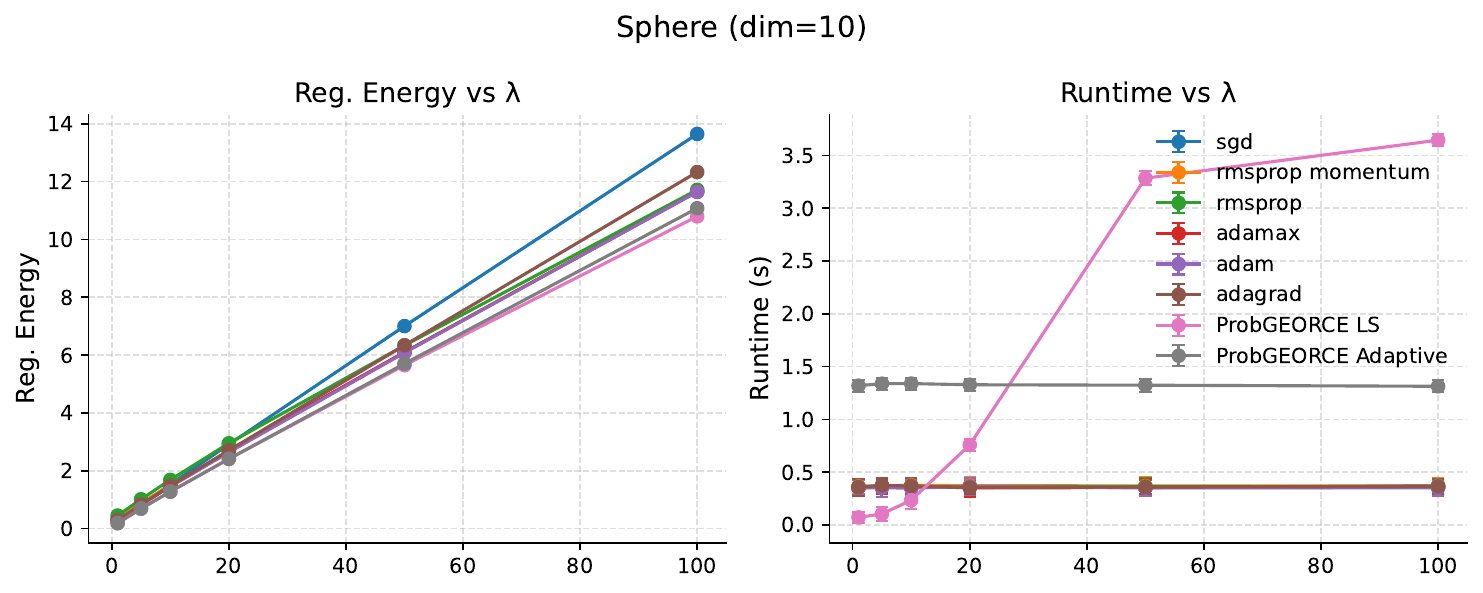}
    \caption{The regularized energy and runtimes for different optimizers minimizing Eq.~\ref{eq:metric_definition}. We consider the $n$-sphere with $n=10$ for different values of $\lambda$ and $S=-\log p$, where $p$ is the density of a mixture of three randomly weighted Gaussians with random means and the identity as covariance matrix. We report the regularized energy and runtime (mean runtime $\pm$ standard deviation runtime over five runs repeated five times).}
    \label{fig:runtime_lam}
    \vspace{0.0em}
\end{figure}
In Table~\ref{tab:runtime_manifolds}, we consider the regularized energy and runtimes for different optimizers minimizing Eq.~\ref{eq:metric_definition}. We consider different manifolds and dimensions with $\lambda=1.0$ and $S=-\log p$, where $p$ is the density of a mixture of three randomly weighted Gaussians with random means and the identity as covariance matrix. In this case, we see that \textit{ProbGEORCE} with line-search is significantly faster than alternative methods and compared to the adaptive update scheme. This is expected since the regularization function is inexpensive to evaluate. In Fig.~\ref{fig:runtime_lam}, we consider the $n$-sphere with $n=10$ for increasing values of $\lambda$. In this case, we see that ProbGEORCE with line-search and the adaptive update scheme computes the smallest regularized energy, but also that the runtime increases using line-search as $\lambda$ increases, while the adaptive scheme is more stable in runtime.

\begin{sidewaystable*}[h!]
\caption{The regularized energy and runtimes for different optimizers minimizing Eq.~\ref{eq:metric_definition}. We consider different manifolds and dimensions with $\lambda=1.0$ and $S=-\log p$, where $p$ is the density of a mixture of three randomly weighted Gaussians with random means and the identity as covariance matrix. We report the regularized energy and runtime (mean runtime $\pm$ standard deviation runtime over five runs repeated five times).}
\centering
\resizebox{\textwidth}{!}{%
\begin{tabular}{lcccccccccccccccc}
\toprule
  & \multicolumn{2}{c}{\bfseries\large sgd} & \multicolumn{2}{c}{\bfseries\large rmsprop momentum} & \multicolumn{2}{c}{\bfseries\large rmsprop} & \multicolumn{2}{c}{\bfseries\large adamax} & \multicolumn{2}{c}{\bfseries\large adam} & \multicolumn{2}{c}{\bfseries\large adagrad} & \multicolumn{2}{c}{\bfseries\large ProbGEORCE LS} & \multicolumn{2}{c}{\bfseries\large ProbGEORCE Adaptive} \\
\cmidrule(lr){2-3} \cmidrule(lr){4-5} \cmidrule(lr){6-7} \cmidrule(lr){8-9} \cmidrule(lr){10-11} \cmidrule(lr){12-13} \cmidrule(lr){14-15} \cmidrule(lr){16-17}
Manifold & Reg. Energy & Runtime & Reg. Energy & Runtime & Reg. Energy & Runtime & Reg. Energy & Runtime & Reg. Energy & Runtime & Reg. Energy & Runtime & Reg. Energy & Runtime & Reg. Energy & Runtime \\
\midrule
Cauchy (dim=2) & 0.0853 & 0.344 $\pm$ 0.037 & 539.4545 & 0.350 $\pm$ 0.033 & 0.1059 & 0.358 $\pm$ 0.033 & 0.0839 & 0.333 $\pm$ 0.033 & 0.0843 & 0.344 $\pm$ 0.035 & 0.0839 & 0.340 $\pm$ 0.033 & \textbf{0.0689} & \textbf{0.037 $\pm$ 0.029} & 0.0689 & 0.614 $\pm$ 0.034 \\
\midrule
Ellipsoid (dim=2) & 0.0854 & 0.330 $\pm$ 0.035 & 0.1108 & 0.344 $\pm$ 0.036 & 0.1282 & 0.361 $\pm$ 0.045 & 0.0887 & 0.352 $\pm$ 0.033 & 0.0894 & 0.334 $\pm$ 0.033 & 0.0892 & 0.333 $\pm$ 0.041 & \textbf{0.0833} & \textbf{0.026 $\pm$ 0.027} & 0.0833 & 0.487 $\pm$ 0.026 \\
Ellipsoid (dim=3) & 0.0967 & 0.328 $\pm$ 0.036 & 0.1190 & 0.338 $\pm$ 0.030 & 0.1540 & 0.319 $\pm$ 0.036 & 0.1010 & 0.365 $\pm$ 0.027 & 0.1019 & 0.365 $\pm$ 0.027 & 0.1014 & 0.323 $\pm$ 0.036 & \textbf{0.0929} & \textbf{0.034 $\pm$ 0.028} & 0.0929 & 0.665 $\pm$ 0.031 \\
Ellipsoid (dim=5) & 0.1151 & 0.353 $\pm$ 0.030 & 0.1669 & 0.346 $\pm$ 0.027 & 0.1954 & 0.350 $\pm$ 0.039 & 0.1223 & 0.382 $\pm$ 0.033 & 0.1238 & 0.354 $\pm$ 0.036 & 0.1227 & 0.342 $\pm$ 0.033 & \textbf{0.0969} & \textbf{0.090 $\pm$ 0.031} & 0.0974 & 1.125 $\pm$ 0.032 \\
Ellipsoid (dim=10) & 0.1441 & 0.369 $\pm$ 0.036 & 0.2233 & 0.381 $\pm$ 0.032 & 0.2655 & 0.371 $\pm$ 0.036 & 0.1583 & 0.374 $\pm$ 0.039 & 0.1618 & 0.376 $\pm$ 0.036 & 0.1588 & 0.382 $\pm$ 0.034 & \textbf{0.1038} & \textbf{0.147 $\pm$ 0.034} & 0.1054 & 1.325 $\pm$ 0.030 \\
Ellipsoid (dim=20) & 0.1682 & 0.463 $\pm$ 0.030 & 0.3182 & 0.460 $\pm$ 0.028 & 0.3240 & 0.467 $\pm$ 0.035 & 0.1870 & 0.427 $\pm$ 0.032 & 0.1951 & 0.457 $\pm$ 0.038 & 0.1889 & 0.459 $\pm$ 0.028 & \textbf{0.1024} & \textbf{0.092 $\pm$ 0.031} & 0.1028 & 1.655 $\pm$ 0.029 \\
Ellipsoid (dim=50) & 0.1685 & 0.767 $\pm$ 0.030 & 0.4460 & 0.757 $\pm$ 0.031 & 0.3461 & 0.699 $\pm$ 0.029 & 0.1797 & 0.707 $\pm$ 0.033 & 0.1922 & 0.714 $\pm$ 0.031 & 0.1842 & 0.756 $\pm$ 0.032 & \textbf{0.0903} & \textbf{0.160 $\pm$ 0.035} & 0.0905 & 2.463 $\pm$ 0.032 \\
Ellipsoid (dim=100) & 0.1382 & 1.597 $\pm$ 0.030 & 0.4146 & 1.615 $\pm$ 0.031 & 0.3022 & 1.614 $\pm$ 0.031 & 0.1334 & 1.605 $\pm$ 0.031 & 0.1378 & 1.622 $\pm$ 0.030 & 0.1351 & 1.623 $\pm$ 0.033 & \textbf{0.0664} & \textbf{0.590 $\pm$ 0.031} & 0.0669 & 2.519 $\pm$ 0.030 \\
\midrule
Frechet (dim=2) & 0.0291 & 0.349 $\pm$ 0.036 & 2421.0039 & 0.327 $\pm$ 0.029 & 0.0917 & 0.327 $\pm$ 0.039 & 0.0286 & 0.328 $\pm$ 0.030 & 0.0291 & 0.363 $\pm$ 0.027 & 0.0286 & 0.328 $\pm$ 0.030 & \textbf{0.0283} & \textbf{0.021 $\pm$ 0.027} & 0.0283 & 0.363 $\pm$ 0.031 \\
\midrule
Gaussian (dim=2) & - & - & 1222.1246 & 0.341 $\pm$ 0.034 & 0.2471 & 0.361 $\pm$ 0.031 & 0.1752 & 0.335 $\pm$ 0.028 & 0.1758 & 0.345 $\pm$ 0.032 & 0.1754 & 0.356 $\pm$ 0.030 & \textbf{0.1541} & \textbf{0.029 $\pm$ 0.028} & 0.1541 & 0.700 $\pm$ 0.029 \\
\midrule
Pareto (dim=2) & 0.0150 & 0.342 $\pm$ 0.031 & 2151.8325 & 0.353 $\pm$ 0.033 & 0.0465 & 0.356 $\pm$ 0.036 & 0.0146 & 0.331 $\pm$ 0.031 & 0.0149 & 0.359 $\pm$ 0.030 & 0.0146 & 0.331 $\pm$ 0.029 & \textbf{0.0145} & \textbf{0.021 $\pm$ 0.027} & 0.0145 & 0.338 $\pm$ 0.030 \\
\midrule
SPDN (dim=2) & 0.0275 & 0.485 $\pm$ 0.029 & 0.0659 & 0.450 $\pm$ 0.032 & 0.1574 & 0.456 $\pm$ 0.030 & 0.0265 & 0.442 $\pm$ 0.031 & 0.0259 & 0.464 $\pm$ 0.033 & 0.0265 & 0.448 $\pm$ 0.030 & \textbf{0.0251} & \textbf{0.042 $\pm$ 0.029} & 0.0251 & 0.667 $\pm$ 0.030 \\
SPDN (dim=3) & 0.1197 & 0.504 $\pm$ 0.030 & 0.2602 & 0.494 $\pm$ 0.030 & 0.5051 & 0.514 $\pm$ 0.029 & 0.1185 & 0.500 $\pm$ 0.030 & 0.1196 & 0.522 $\pm$ 0.029 & 0.1184 & 0.517 $\pm$ 0.030 & \textbf{0.1126} & \textbf{0.072 $\pm$ 0.031} & 0.1126 & 0.899 $\pm$ 0.031 \\
\midrule
Sphere (dim=2) & 0.1208 & 0.364 $\pm$ 0.036 & 0.1433 & 0.324 $\pm$ 0.035 & 0.1739 & 0.320 $\pm$ 0.036 & 0.1224 & 0.358 $\pm$ 0.035 & 0.1228 & 0.331 $\pm$ 0.037 & 0.1226 & 0.330 $\pm$ 0.040 & \textbf{0.1160} & \textbf{0.026 $\pm$ 0.029} & 0.1160 & 0.701 $\pm$ 0.031 \\
Sphere (dim=3) & 0.1546 & 0.314 $\pm$ 0.041 & 0.1897 & 0.313 $\pm$ 0.035 & 0.2310 & 0.335 $\pm$ 0.034 & 0.1584 & 0.339 $\pm$ 0.030 & 0.1591 & 0.327 $\pm$ 0.032 & 0.1587 & 0.341 $\pm$ 0.038 & \textbf{0.1462} & \textbf{0.042 $\pm$ 0.030} & 0.1462 & 1.031 $\pm$ 0.029 \\
Sphere (dim=5) & 0.2048 & 0.340 $\pm$ 0.040 & 0.2585 & 0.319 $\pm$ 0.038 & 0.3151 & 0.343 $\pm$ 0.036 & 0.2075 & 0.377 $\pm$ 0.031 & 0.2080 & 0.352 $\pm$ 0.030 & 0.2076 & 0.319 $\pm$ 0.048 & \textbf{0.1653} & \textbf{0.050 $\pm$ 0.026} & 0.1655 & 1.049 $\pm$ 0.032 \\
Sphere (dim=10) & 0.2749 & 0.362 $\pm$ 0.038 & 0.3606 & 0.364 $\pm$ 0.038 & 0.4526 & 0.355 $\pm$ 0.036 & 0.2820 & 0.353 $\pm$ 0.042 & 0.2837 & 0.353 $\pm$ 0.036 & 0.2822 & 0.358 $\pm$ 0.041 & \textbf{0.1873} & \textbf{0.071 $\pm$ 0.028} & 0.1879 & 1.319 $\pm$ 0.029 \\
Sphere (dim=20) & 0.3242 & 0.423 $\pm$ 0.029 & 2.8070 & 0.462 $\pm$ 0.030 & 0.5753 & 0.421 $\pm$ 0.030 & 0.3325 & 0.440 $\pm$ 0.029 & 0.3339 & 0.435 $\pm$ 0.030 & 0.3310 & 0.461 $\pm$ 0.031 & \textbf{0.1893} & \textbf{0.122 $\pm$ 0.037} & 0.1894 & 1.642 $\pm$ 0.031 \\
Sphere (dim=50) & 0.3152 & 0.700 $\pm$ 0.033 & 1.3125 & 0.687 $\pm$ 0.031 & 0.6369 & 0.706 $\pm$ 0.031 & 0.3247 & 0.691 $\pm$ 0.031 & 0.3292 & 0.712 $\pm$ 0.029 & 0.3221 & 0.685 $\pm$ 0.031 & \textbf{0.1654} & \textbf{0.243 $\pm$ 0.033} & 0.1654 & 3.050 $\pm$ 0.029 \\
Sphere (dim=100) & 0.2504 & 1.475 $\pm$ 0.032 & 2.1177 & 1.653 $\pm$ 0.028 & 0.5268 & 1.530 $\pm$ 0.030 & 0.2441 & 1.592 $\pm$ 0.032 & 0.2409 & 1.639 $\pm$ 0.029 & 0.2386 & 1.612 $\pm$ 0.030 & \textbf{0.1177} & \textbf{0.401 $\pm$ 0.030} & 0.1179 & 4.443 $\pm$ 0.030 \\
\midrule
T2 (dim=2) & 2.7905 & 0.349 $\pm$ 0.036 & 2.7950 & 0.348 $\pm$ 0.038 & 2.8467 & 0.343 $\pm$ 0.038 & 2.7551 & 0.361 $\pm$ 0.027 & 2.7582 & 0.347 $\pm$ 0.035 & 2.6305 & 0.353 $\pm$ 0.027 & \textbf{2.4578} & \textbf{0.078 $\pm$ 0.029} & 2.4578 & 0.925 $\pm$ 0.029 \\
\bottomrule
\end{tabular}}
\label{tab:runtime_manifolds}
\end{sidewaystable*}

\end{document}